\documentclass[twoside,11pt]{article}

\usepackage{amsmath}
\usepackage{pgfplots}
\pgfplotsset{compat=1.14}
\usepackage{filecontents}
\usepackage{amsfonts}
\usepackage{amsthm}
\usepackage{mathabx}
\usepackage{bbold}
\usepackage{hyperref}
\usepackage{mdframed}
\usepackage{ulem}
\usepackage{caption,subcaption}
\usepackage{pdflscape}
\usepackage{lscape}
\usepackage{txfonts}
\usepackage{blkarray}
\usepackage{tikz}
\usetikzlibrary{positioning,shapes,calc,patterns,graphs,arrows.meta,decorations}
\usepackage{booktabs}
\usepackage{boldline}
\usepackage{colortbl}
\usepackage{float}
\usepackage{colortbl}
\usepackage{algorithm}
\usepackage{algpseudocode}
\usepackage{ulem}
\usepackage{amsfonts}
\usepackage{upgreek}
\usepackage{multirow}
\usepackage{ifthen}
\usepackage{enumitem}
\usepackage{slashbox}

\newtheorem{theorem}{Theorem}[section]
\newtheorem{corollary}{Corollary}[theorem]

\newtheorem*{remark}{Remark}

\newcommand{\blockmatrix}[9]{
  \draw[draw=#4,fill=#5] (0,0) rectangle( #1,#2);
  \ifthenelse{\equal{#6}{true}}
  {
    \draw[draw=#7,fill=#8] (0,#2) -- (#9,#2) -- ( #1,#9) -- ( #1,0) -- ( #1 - #9,0) -- (0,#2 -#9) -- cycle;
  }
  {}
  \draw ( #1/2, #2/2) node { #3};
  \draw[pattern=north west lines] (0,0.5) rectangle (0.5,1.5);
}

\newcommand{\blockmatrixXaTrain}[9]{
  \draw[draw=#4,fill=#5] (0,0) rectangle( #1,#2);
  \ifthenelse{\equal{#6}{true}}
  {
    \draw[draw=#7,fill=#8] (0,#2) -- (#9,#2) -- ( #1,#9) -- ( #1,0) -- ( #1 - #9,0) -- (0,#2 -#9) -- cycle;
  }
  {}
  \draw ( #1/2, #2/2) node { #3};
  \draw[pattern=north west lines] (0,0) rectangle (#1,0.5);
  \draw[pattern=north west lines] (0,1.5) rectangle (#1,1.8);
}

\newcommand{\blockmatrixXaTest}[9]{
  \draw[draw=#4,fill=#5] (0,0) rectangle( #1,#2);
  \ifthenelse{\equal{#6}{true}}
  {
    \draw[draw=#7,fill=#8] (0,#2) -- (#9,#2) -- ( #1,#9) -- ( #1,0) -- ( #1 - #9,0) -- (0,#2 -#9) -- cycle;
  }
  {}
  \draw ( #1/2, #2/2) node { #3};
  \draw[pattern=north west lines] (0,0.1) rectangle (#1,0.3);
}
\newcommand{\blockmatrixXbTrain}[9]{
  \draw[draw=#4,fill=#5] (0,0) rectangle( #1,#2);
  \ifthenelse{\equal{#6}{true}}
  {
    \draw[draw=#7,fill=#8] (0,#2) -- (#9,#2) -- ( #1,#9) -- ( #1,0) -- ( #1 - #9,0) -- (0,#2 -#9) -- cycle;
  }
  {}
  \draw ( #1/2, #2/2) node { #3};
  \draw[pattern=north west lines] (0,0.5) rectangle (#1,1.2);
}

\newcommand{\blockmatrixXbTest}[9]{
  \draw[draw=#4,fill=#5] (0,0) rectangle( #1,#2);
  \ifthenelse{\equal{#6}{true}}
  {
    \draw[draw=#7,fill=#8] (0,#2) -- (#9,#2) -- ( #1,#9) -- ( #1,0) -- ( #1 - #9,0) -- (0,#2 -#9) -- cycle;
  }
  {}
  \draw ( #1/2, #2/2) node { #3};
  \draw[pattern=north west lines] (0,0.2) rectangle (#1,0.3);
   \draw[pattern=north west lines] (0,0.8) rectangle (#1,1);
}

\newcommand{\blockmatrixY}[9]{
  \draw[draw=#4,fill=#5] (0,0) rectangle( #1,#2);
  \ifthenelse{\equal{#6}{true}}
  {
    \draw[draw=#7,fill=#8] (0,#2) -- (#9,#2) -- ( #1,#9) -- ( #1,0) -- ( #1 - #9,0) -- (0,#2 -#9) -- cycle;
  }
  {}
  \draw ( #1/2, #2/2) node { #3};
}

\newcommand{\mblockmatrix}[4][none]{
  \begin{tikzpicture} 
  \ifthenelse{\equal{#1}{none}}
  {
    \blockmatrix{#2}{#3}{#4}{none}{none}{false}{none}{none}{0.0}
  }
  {
    \definecolor{fillcolor}{rgb}{#1}
    \blockmatrix{#2}{#3}{#4}{none}{fillcolor}{false}{none}{none}{0.0}
  }
  \end{tikzpicture}%this comment is necessary
}

\newcommand{\fblockmatrix}[4][none]{
  \begin{tikzpicture} 
  \ifthenelse{\equal{#1}{none}}
  {
    \blockmatrix{#2}{#3}{#4}{black}{none}{false}{none}{none}{0.0}
  }
  {
    \definecolor{fillcolor}{rgb}{#1}
    \blockmatrix{#2}{#3}{#4}{black}{fillcolor}{false}{none}{none}{0.0}
  }
  \end{tikzpicture}%this comment is necessary
}
\newcommand{\fblockmatrixXaTrain}[4][none]{
  \begin{tikzpicture} 
  \ifthenelse{\equal{#1}{none}}
  {
    \blockmatrixXaTrain{#2}{#3}{#4}{black}{none}{false}{none}{none}{0.0}
  }
  {
    \definecolor{fillcolor}{rgb}{#1}
    \blockmatrixXaTrain{#2}{#3}{#4}{black}{fillcolor}{false}{none}{none}{0.0}
  }
  \end{tikzpicture}%this comment is necessary
}
\newcommand{\fblockmatrixXaTest}[4][none]{
  \begin{tikzpicture} 
  \ifthenelse{\equal{#1}{none}}
  {
    \blockmatrixXaTest{#2}{#3}{#4}{black}{none}{false}{none}{none}{0.0}
  }
  {
    \definecolor{fillcolor}{rgb}{#1}
    \blockmatrixXaTest{#2}{#3}{#4}{black}{fillcolor}{false}{none}{none}{0.0}
  }
  \end{tikzpicture}%this comment is necessary
}

\newcommand{\fblockmatrixXbTrain}[4][none]{
  \begin{tikzpicture} 
  \ifthenelse{\equal{#1}{none}}
  {
    \blockmatrixXbTrain{#2}{#3}{#4}{black}{none}{false}{none}{none}{0.0}
  }
  {
    \definecolor{fillcolor}{rgb}{#1}
    \blockmatrixXbTrain{#2}{#3}{#4}{black}{fillcolor}{false}{none}{none}{0.0}
  }
  \end{tikzpicture}%this comment is necessary
}
\newcommand{\fblockmatrixXbTest}[4][none]{
  \begin{tikzpicture} 
  \ifthenelse{\equal{#1}{none}}
  {
    \blockmatrixXbTest{#2}{#3}{#4}{black}{none}{false}{none}{none}{0.0}
  }
  {
    \definecolor{fillcolor}{rgb}{#1}
    \blockmatrixXbTest{#2}{#3}{#4}{black}{fillcolor}{false}{none}{none}{0.0}
  }
  \end{tikzpicture}%this comment is necessary
}

\newcommand{\fblockmatrixY}[4][none]{
  \begin{tikzpicture} 
  \ifthenelse{\equal{#1}{none}}
  {
    \blockmatrixY{#2}{#3}{#4}{black}{none}{false}{none}{none}{0.0}
  }
  {
    \definecolor{fillcolor}{rgb}{#1}
    \blockmatrixY{#2}{#3}{#4}{black}{fillcolor}{false}{none}{none}{0.0}
  }
  \end{tikzpicture}%this comment is necessary
}

\newcommand{\valignbox}[1]{
  \vtop{\null\hbox{#1}}% necessary comment
}

\newenvironment{blockmatrixtabular}
{% necessary comment
  \begin{tabular}{
  @{}l@{}l@{}l@{}l@{}l@{}l@{}l@{}l@{}l@{}l@{}l@{}l@{}l@{}l@{}l@{}l@{}l@{}l@{}l
  @{}l@{}l@{}l@{}l@{}l@{}l@{}l@{}l@{}l@{}l@{}l@{}l@{}l@{}l@{}l@{}l@{}l@{}l@{}l
  @{}l@{}l@{}l@{}l@{}l@{}l@{}l@{}l@{}l@{}l@{}l@{}l@{}l@{}l@{}l@{}l@{}l@{}l@{}l
  @{}
  }
}
{
  \end{tabular}%necessary comment
}

\algnewcommand{\Inputs}[1]{%
  \State \textbf{Inputs:}
  \Statex \hspace*{\algorithmicindent}\parbox[t]{.8\linewidth}{\raggedright #1}
}
\algnewcommand{\Initialize}[1]{%
  \State \textbf{Initialize:}
  \Statex \hspace*{\algorithmicindent}\parbox[t]{.8\linewidth}{\raggedright #1}
}

\newcolumntype{L}[1]{>{\raggedright\let\newline\\\arraybackslash\hspace{0pt}}m{#1}}
\newcolumntype{C}[1]{>{\centering\let\newline\\\arraybackslash\hspace{0pt}}m{#1}}
\newcolumntype{R}[1]{>{\raggedleft\let\newline\\\arraybackslash\hspace{0pt}}m{#1}}

\newmdenv[
  topline=false,
  bottomline=false,
  skipabove=\topsep,
  skipbelow=\topsep,
  leftmargin=-10pt,
  rightmargin=-10pt,
  innertopmargin=0pt,
  innerbottommargin=0pt
]{siderules}
\usepackage{natbib}

\title{}

\pgfdeclarelayer{background}
\pgfsetlayers{background,main}
\definecolor{darkgreen}{rgb}{0.05, 0.5, 0.06}
\definecolor{darkblue}{rgb}{0.0, 0.0, 0.55}
\definecolor{LightCyan}{rgb}{0.88,1,1}

\definecolor{paired1}{HTML}{A6CEE3} 
\definecolor{paired2}{HTML}{1F78B4}
\definecolor{paired3}{HTML}{B2DF8A}
\definecolor{paired4}{HTML}{33A02C}
\definecolor{paired5}{HTML}{FB9A99}
\definecolor{paired6}{HTML}{E31A1C}
\definecolor{paired7}{HTML}{FDBF6F}
\definecolor{paired8}{HTML}{FF7F00}
\definecolor{paired9}{HTML}{CAB2D6}
\definecolor{paired10}{HTML}{6A3D9A}

\begin{document}

% \begin{frontmatter}

% \title{Supervised learning for multi-block incomplete data and application to an Ebola vaccine trial data set}

% \begin{abstract}
% Missing data may hide information and this can be solved using imputation techniques. In the context of high dimensional setting, corresponding to modern vaccine trials for example, it is important to reduce the number of variables but also of sub-space estimations. That sub-space selection is often managed with supervised tools. We propose a PLS inspired method, called ``dd-sPLS'' for data-driven-sparse PLS, and its multiblock version ``mdd-sPLS'' for multi-data-driven-sparse PLS, permitting jointly sub-space and variable selections and missing data imputation. This method reveals very easy to use capacities and also interpretable results on classic data sets. We challenge it against well-known baseline methods through simulations. Finally we apply the method to the rVSV-ZEBOV Ebola vaccine trial data set and show its interest in that case.
% \end{abstract}

% \begin{keyword}
% Variable selection \sep Supervised learning \sep Dimension reduction \sep PLS regression \sep Missing data \sep Heterogeneous data \sep Multi-block data \sep Omics \sep Medicine
% \end{keyword}

% \end{frontmatter}

%\tableofcontents

%\linenumbers

\title{Supervised Learning for Multi-Block Incomplete Data}

\author{ Hadrien Lorenzo \\ hadrien.lorenzo@u-bordeaux.fr\\
SISTM INRIA BSO, Inserm-U1219 BPH\\146 rue Léo Saignat
33076 Bordeaux cedex
\and
 J\'{e}r\^{o}me Saracco \\ jerome.saracco@inria.fr
\\ CQFD INRIA Bordeaux Sud-Ouest-France, CNRS (UMR5251)\\ 200 Avenue de la Vieille Tour, 33405 Talence
\and
 Rodolphe Thi\'{e}baut \\ rodolphe.thiebaut@u-bordeaux.fr
\\ SISTM INRIA BSO, Inserm-U1219 BPH, Vaccine Research Institute\\8 rue du Général Sarrail
94010 Créteil cedex }

%\editor{xx}

\maketitle

\begin{abstract}%
In the supervised high dimensional settings with a large number of variables and a low number of individuals, one objective is to select the relevant variables and thus to reduce the dimension. That subspace selection is often managed with supervised tools. However, some data can be missing, compromising the validity of the sub-space selection. We propose a Partial Least Square (PLS) based method, called Multi-block Data-Driven sparse PLS ``mdd-sPLS'', allowing jointly variable selection and subspace estimation while training and testing missing data imputation through a new algorithm called Koh-Lanta. This method was challenged through simulations against existing methods such as mean imputation, nipals, softImpute and imputeMFA. In the context of supervised analysis of high dimensional data, the proposed method shows the lowest prediction error of the response variables. So far this is the only method combining data imputation and response variable prediction. The superiority of the supervised multi-block mdd-sPLS method increases with the intra-block and inter-block correlations. The application to a real data-set from a rVSV-ZEBOV Ebola vaccine trial revealed interesting and biologically relevant results. The method is implemented in a \textbf{R}-package available on the \textbf{CRAN} and a \textbf{Python}-package available on \textbf{pypi}. 
\end{abstract}

\section{Introduction}
\label{S:1}

Missing data may also happen in high-dimensional settings where the number of variables is very large but not fully completed. Among the many different methodological solutions, of which a description can be found in~\citet{JMLR:v18:17-073}, this paper focuses on singular value decomposition (\textbf{SVD})-based methods in the context of simple imputation. The \textbf{SVD}-based imputation methods assume that the eigen vectors are not too much influenced by the missing values and also can be denoised by alternating \textbf{SVD}-decompositions and imputations through linear combinations of the current eigen vectors until convergence is reached~\citep[][]{troyanskaya2001missing,hastie2015matrix}. The \textbf{mean} imputation is used as a reference method and leads to good results in many situations.~\citet{hastie2015matrix} propose a particularly fast ridge-regression and \textbf{SVD} soft-thresholding alternated method, designed for the mono-block context.~\citet{husson2013handling} develop a multi-block method called \textbf{imputeMFA} that uses multiple axis decomposition with self-tunable ridge penalization. The \textbf{nipals} algorithm, for \textbf{N}onlinear \textbf{I}terative \textbf{Pa}rtial \textbf{L}east \textbf{S}quares firstly detailed by~\citet{wold1966estimation}, allows dealing with missing observations~\citep[see for example][]{nelson1996missing} in the context of unsupervised problems with principal component analysis (\textbf{PCA}) but also in supervised problems with the Partial Least Squares (\textbf{PLS}). Other methods not considered in this paper are either not based on \textbf{SVD} such as \textbf{missForest} developed by~\citet{stekhoven2011missforest}  or are multiple imputation methods such as \texttt{mice} detailed by~\citet{buuren2010mice}.

All the methods existing so far and presented above are two-steps approaches where the data are first imputed and then the imputed data set is analyzed using any unrelated statistical tool. Our objective was to develop a method able to impute missing covariates and predict the response in the same time, for a multi-block supervised data set with potential missing data in the covariate part.

% This situation is found in The Netflix competition dealing with very large matrix imputation, hundreds of thousands of columns and tens of thousands of rows. 
% The multi-block data context is faced with the problem of sensory analysis and judge olfactory saturation in wine or perfume notation context (\citep{husson2013handling}). Here $n$ is the number of judges,  $n\approx p$, which implies that the number of subspaces is not huge. Finally another context, such as phase 1 clinical trials with omics data, $n$ is very small and $p$ tends to be huge.
% For example~\citep{hastie2015matrix} propose an alternated Ridge regression and \textbf{SVD} soft-thresholding to impute more than $99\%$ of a matrix which is supposed to be very ``sparse'' with  a low number of hidden dimensions. 
% [A RETIRER OU BIEN DONNER UNE DESCRIPTION PLUS EXHAUSTIVE DES METHODES EMPLOYEES DANS LE CONTEXTE]
In the following, matrices are written with bold capital letters and vectors with bold lower-case letters. For any matrix ${\bf M}$, $\mathbf{m}^{(i)}$ denotes its $i^{th}$-column.
Let $\mathbf{X}\in\mathbb{R}^{n\times p}$ and $\mathbf{Y}\in\mathbb{R}^{n\times q}$, be respectively the covariate matrix and the  response matrix, describing $n$ individuals through $p$ (resp. $q$) variables. Unless otherwise stated, data matrices $\mathbf{X}$ and $\mathbf{Y}$ are supposed to be standardized (zero-mean and unit-variance). \\
 The \textbf{PLS} problem maximizes the covariance between both of the projections of $\mathbf{X}$ and $\mathbf{Y}$ on their proper weights denoted by $\mathbf{u}\in \mathbb{R}^{p} $ and $\mathbf{v}\in \mathbb{R}^{q}$ respectively.  The underlying optimization problem can be written as
\begin{equation}
\begin{aligned}
& \underset{(\mathbf{u},\mathbf{v})}{\max}
& & \mathbf{u}^T\mathbf{X}^T\mathbf{Y}\mathbf{v} \\
& \text{s.t.}
& & \mathbf{u}^T\mathbf{u}=\mathbf{v}^T\mathbf{v}=1,
\end{aligned}
\label{equ:equPLS0_intro}
\end{equation}
for the current axis decomposition. The \textbf{nipals} permits to solve that problem. If further axes are needed, deflations are successively performed to remove the information carried by previous axes before solving the problem~\eqref{equ:equPLS0_intro} on the corresponding residual matrices. It has been shown that~\eqref{equ:equPLS0_intro} is equivalent to find eigen-vector linked to the largest eigen-value of $\mathbf{X}^T\mathbf{Y}\mathbf{Y}^T\mathbf{X}$ such as 
reformulated by \citet{hoskuldsson1988pls}.

%as follows
%\begin{equation}
%\begin{aligned}
%& \underset{\mathbf{u}}{\max}
%& & \mathbf{u}^T\mathbf{X}^T\mathbf{Y}\mathbf{Y}^T\mathbf{X}\mathbf{u} \\
%& \text{s.t.}
%& & \mathbf{u}^T\mathbf{u}=1,
%\end{aligned}
%\label{equ:equPLS1_intro}
%\end{equation}
%where only the vector of weights $\mathbf{u}$ is computed. The weights $\mathbf{v}$ are then obtained from $\mathbf{u}$ as follows: $\mathbf{v}=\mathbf{Y}^T\mathbf{X}\mathbf{u}/||\mathbf{Y}^T\mathbf{X}\mathbf{u}||_2$. Problem~\eqref{equ:equPLS1_intro} clearly consists in finding the largest eigenvalue of the positive semi-definite matrix $\mathbf{X}^T\mathbf{Y}\mathbf{Y}^T\mathbf{X}$.\\
%A third formulation of the problem~\eqref{equ:equPLS0_intro} can be retrieved using Eckart-Young's theorem, as recalled in~\citep{liquet2015group}
%\begin{equation*}
%\begin{aligned}
%& \underset{(\mathbf{u},\mathbf{v})}{\min}
%& & ||\mathbf{Y}^T\mathbf{X}-\mathbf{v}\mathbf{u}^T||_F^2 \\
%& \text{s.t.}
%& & \mathbf{u}^T\mathbf{u}=1,
%\end{aligned}
%\end{equation*}
%where the problem now minimizes the square Frobenius distance between $\mathbf{Y}^T\mathbf{X}$ and its estimation  through the computed weights. Here $\mathbf{u}$ is the right  singular vector associated to the largest singular value of $\mathbf{Y}^T\mathbf{X}$, and $\mathbf{v}$ is the left singular vector multiplied by the largest singular value.
To deal with the variable selection problem in {\bf PLS}, two existing methods are presented hereafter. They are based on the $\mathcal{L}$asso formulation which shrinks the $\mathcal{L}_1$-norm of the weights~\citep[see][]{tibshirani1996regression}. Based on the SCoTLASS solution to {\bf sparse PCA}, proposed by~\citet{jolliffe2003modified},~\citet{le2008sparse} considered a $\mathcal{L}_1$-norm penalization of the $\bf X$ and $\bf Y$ weights, introducing two $\mathcal{L}$agrangian coefficients which are fixed by the user,
%\begin{equation}
%\begin{aligned}
%& \underset{(u,v)}{\min}
%& & ||\mathbf{Y}^T\mathbf{X}-\mathbf{v}\mathbf{u}^T||_F^2+\lambda_x||\mathbf{u}||_1+\lambda_y||\mathbf{v}||_1,\\
%\end{aligned}
%\label{equ:equPLS3_intro}
%\end{equation}
%where two $\mathcal{L}$agrangian-like coefficients $\lambda_x$ and $\lambda_y$, 
 denoted as $\mathcal{L}$asso parameters. The associated approach has been implemented in the \textbf{R}-package \textbf{mixOmics}, see~\citet{le2009integromics}, and is denoted as \textbf{classic-sPLS} in the following. As pointed out by~\citet{zou2006sparse} and recalled by~\citet{chun2010sparse}, problem from~\citet{le2008sparse} is not convex and solutions are not in practice sufficiently sparse. \citet{chun2010sparse} proposed an alternative formulation to mitigate those drawbacks.
%\begin{equation}
%\begin{aligned}
%& \underset{(\mathbf{w},\mathbf{u})}{\min}
%& & -\kappa \mathbf{w}^T\mathbf{M}\mathbf{w}+(1-\kappa)(\mathbf{u}-\mathbf{w})^T\mathbf{M}(\mathbf{u}-\mathbf{w})+\lambda_1||\mathbf{u}||_1+\lambda_2||\mathbf{u}||_2\\
%& \text{s.t.}
%& & \mathbf{w}^T\mathbf{w}=1,
%\end{aligned}
%\label{equ:equPLS4_intro}
%\end{equation}
%where $\mathbf{M}=\mathbf{X}^T\mathbf{Y}\mathbf{Y}^T\mathbf{X}$. Here the authors looked for a balance between the classical eigen-problem represented by $ \mathbf{w}^T\mathbf{M}\mathbf{w}$ and the proximity between $\mathbf{w}$ and $\mathbf{u}$ through the $\mathbf{M}$ matrix with $(\mathbf{u}-\mathbf{w})^T\mathbf{M}(\mathbf{u}-\mathbf{w})$. The coefficient $\kappa$ permits to deal with that balance and is fixed by the user: a small value of
Introducing a parameter $\kappa$, they can reduce the concave part of the original problem and so reduces its impact on the global optimization problem. The authors performed simulations showing that a low $\kappa$ indeed provides ``a numerically easier optimization problem'' but no general result was given.
%Note that $\mathcal{L}_1$ and $\mathcal{L}_2$ penalizations on $\mathbf{u}$ imply two other parameters, $\lambda_1$ and $\lambda_2$, to tune. The power of this method is to optimize on $\mathbf{w}$ the eigen-problem and to soft-threshold on another vector $\mathbf{u}$, close to $\mathbf{w}$. The second term permits to provide a vector $\mathbf{u}$ as close as possible to $\mathbf{w}$ while the $\mathcal{L}_1$-penalization on $\mathbf{u}$ gives it the selection property.
The main drawback is the computational cost due to the number of parameters. Furthermore, their problem allows to select only on the $\bf X$ data set while ~\citet{le2008sparse} allow to select variables on both data sets $\bf X$ and $\bf Y$. This is clearly a limit if the number $q$ of variables in $\mathbf{Y}$ is large and so tends to draw uncorrelated subspaces.

% The first solution permits indeed to present an elegant solution but the comments of the successive authors shall prevent the users from giving it too much confidence. While the second one, was built on a strong desire to answer the problems so defined but, as elegant as it is, with no considerations to the structure of the PLS-covariance itself and its behavior to the $\mathcal{L}$asso penalization.\\
Other variable selection methods have been recently studied, such as the \textbf{SVD} decomposition of thresholded variance matrices, developed to tackle the \textbf{sparse PCA} question. \citet{d2005direct} have developed an elegant convex relaxed optimization problem and detailed an algorithm to solving it. Subsequently,~\citet{amini2008high} have compared that \textbf{SDP} solution to the Diagonal-Thresholding (\textbf{DT}) method, developed by~\citet{johnstone2004sparse} and~\citet{johnstone2009consistency}. This method showed comparable results with higher computational efficiency. Different types of threshold operators are considered and studied for example by~\citet{rothman2009generalized} and~\citet{johnstone2009consistency} and more recently by~\citet{cai2011adaptive}.~\citet{JMLR:v17:15-160} detailed an algorithm in which the \textbf{SVD} is performed on the soft-thresholded variance-covariance matrix based on results of~\citet{krauthgamer2015semidefinite} for which element-wise hard-thresholding was considered.  Strong theoretical results about selectivity and consistency exist for those approaches and~\citet{JMLR:v17:15-160} have proved the consistency of the soft-thresholding covariance matrix as an estimator of the covariance matrix in the high-dimensional context, $n<<p$, and using a spiked model hypothesis on the covariable \textbf{X} for which the authors seeks a sparse estimation of the variance-covariance matrix. In the present work, the matrix $\mathbf{Y^TX}$, which is at the core of the \textbf{PLS} problem through its \textbf{SVD}, has been modified under soft-thresholding manipulations which is justified by some of the works cited above and detailed below.

In regards of the context of multi-block high-dimensional data, several supervised approaches inspired by the PLS have been proposed.
The covariate part is defined through $T$ different matrices $\mathbf{X}_1,...,\mathbf{X}_T$ describing the same $n$ individuals. The adaptation of the {\bf PLS} method to the multi-block structure has been initially proposed through the ``SW-Harald-HW multi-block algorithm'' by~\citet{wold1984three}. A few years later ~\citet{wangen1989multiblock} detailed this approach today known as \textbf{MBPLS} (for Multi-Block Partial Least Square) through the optimization problem
\begin{equation*}
\begin{aligned}
& \underset{(\mathbf{u},\boldsymbol{\beta},\mathbf{v})}{\max}
& &\sum_{t=1}^T\beta_t\mathbf{v}^T\mathbf{Y}^T\mathbf{X}_t\mathbf{u}_t\\
& \text{s.t.}
& & \mathbf{u}_t^T\mathbf{u}_t=\mathbf{v}^T\mathbf{v}=\boldsymbol{\beta}^T\boldsymbol{\beta}=1,
\end{aligned}
\end{equation*}
 where the ${\beta}_t's$ gathering the information from the $T$ blocks $\mathbf{X}_t$ via their weight $\mathbf{u}_t$, and they make it possible to build the super-component  $\sum_{t=1}^T\beta_t\mathbf{X}_t\mathbf{u}_t$.
\textbf{MBPLS} is a \textbf{nipals} flavored method, using deflation procedures to obtain further axes. Initially the deflation of each block was made on its proper component, but~\citet{westerhuis1997multivariate} have shown the interest, in terms of prediction, to deflate each block on the super-component. Many authors have decided to challenge that question of deflation~\citep[see for example][]{westerhuis2001deflation}. Supposing that each block has been divided by its square root number of variables,~\citet{qin2001unifying} have demonstrated  the similarities of the \textbf{MBPLS} problem with a classical \textbf{PLS} problem.~\citet{westerhuis2001deflation} have rewritten the \textbf{MBPLS} problem by re-weighting a standard \textbf{PLS} model built on the concatenated matrix of the $T$ blocks.~\citet{bougeard2011multiblock} implemented the \textbf{MBPLS} algorithm in the \textbf{R}-package \textbf{ade4} with the super-component deflation version of~\citet{westerhuis2001deflation}. ~\citet{bougeard2011multiblock} bind the \textbf{MBPLS} problem to the \textbf{RA} problem (Redundancy Analysis) defined through the same criterion as \textbf{PLS}, covariance maximization of the covariates projection and response projections, but here weights $(\mathbf{u}_t)_{t\in\llbracket1,T\rrbracket}$ are not directly constrained while components $(\mathbf{X}_t\mathbf{u}_t)_{t\in\llbracket1,T\rrbracket}$ are constrained to $\mathcal{L}_2$-norm equal to 1. Their solution uses a regularization by convexly balance the power of the variance-covariance matrix of each block towards the identity matrix, and permitting to solve the $\textbf{MBRA}$ (Multi-Block Redundancy Analysis) problem in the context of badly conditioned matrices. That solution has been generalized to the canonical correlation analysis by~\citet{tenenhaus2011regularized} and the corresponding method is called $\textbf{RGCCA}$. A sparse version of that method has been developed and detailed by~\citet{tenenhaus2014variable} using $\mathcal{L}_1$-norm regularization of the weights. Those  methods use \textbf{nipals} typed algorithms and the authors demonstrated their monotonically convergences.

Here, we propose a method to deal with variable selection on multi-block data structured with a specific case of missing data (some entire block rows are missing). The proposed approach is called \textbf{mdd-sPLS} for multi-block data-driven sparse PLS, the term data-driven has been chosen because of the highly interpretable nature of the penalization parameter, it corresponds to the minimum correlation accepted between $\bf X$ and $\bf Y$ variables to put it in the model. It shows good theoretical performances on variable selection and regularization capacities. Simulations and applications to real data sets exhibit its practical, interpretable and numerical interests in comparison to four baseline methods in the presence of missing samples.

The paper is organized as follows. Section~\ref{S:ddspls} describes the covariance-thresholding sparse \textbf{PLS}, called \textbf{ct-sPLS} when there are no missing values. Useful theoretical results are provided in this context. Section~\ref{S:3} details the proposed multi-block approach, (\textbf{mdd-sPLS}), and the chosen algorithm to deal with missing data, denoted \textbf{Koh-Lanta}.  Section~\ref{S:4} studies the numerical behavior of the \textbf{mdd-sPLS} approach through simulations. Section~\ref{S:5} provides results on a real data set from an Ebola rVSV phase 1 vaccine trial. Concluding remarks are given in Section~\ref{S:6}. The method is implemented in the \textbf{R}-package \url{https://cran.r-project.org/package=ddsPLS} and in the \textbf{Python}-package \url{https://pypi.org/project/py_ddspls/}.

%==================================================================
%==================================================================
\section{Covariance-Thresholding Sparse PLS (ct-sPLS)}
\label{S:ddspls}
%==================================================================
%==================================================================

In this section, we focus on the mono-block context associated with the two data matrices ${\bf X}\in {\mathbb{R}}^{n\times p}$ and ${\bf Y}\in \mathbb{R}^{n\times q}$ without missing data. 
First, let us define the soft-thresholding operator, applied term to term to a matrix, as 
$$S_\lambda:\quad x \rightarrow \mbox{sign}(x)(|x|-\lambda)_+,$$
where $\lambda\in[0,1]$ is a regularization parameter  and $(.)_+=\max(0,.)$.
Thus, the matrix $S_\lambda(\mathbf{Y}^T\mathbf{X}/(n-1))$ is the soft-thresholded version of the empirical correlation matrix between $\mathbf{X}$ and $\mathbf{Y}$ (since these matrices are assumed to be standardized) with respect to the threshold $\lambda$.

The proposed \textbf{ct-sPLS} problem, written for a $R$-dimensional decomposition, is defined as 
\begin{equation}
\begin{array}{cl}
\displaystyle \max_{\mathbf{U}\in\mathbb{R}^{p\times R}} & \sum_{r=1}^R ||S_\lambda\Big(\frac{\mathbf{Y}^T\mathbf{X}}{n-1}\Big)\mathbf{u}^{(r)}||_2^2\\
 \text{s.t.} &  \mathbf{U}^T\mathbf{U}=\mathbb{I}_R.
\end{array}
\label{equ:ddspls_problem}
\end{equation}
Solution of problem~\eqref{equ:ddspls_problem} is
$$\mathbf{U}=\mathbf{SVD}_R\Big(S_\lambda\left(\frac{\mathbf{Y}^T\mathbf{X}}{n-1}\right)\Big)=:\textbf{ct-sPLS}(\mathbf{X},\mathbf{Y},\lambda,R)\in \mathbb{R}^{p\times R},$$
where  $\mathbf{SVD}_R({\bf M})$ gives the $R$ first right-singular-vectors of $\bf M$.
The $R$ components rely on the same regularization parameter $\lambda$, which can be obtained thanks to cross-validation. Note that there is no deflation step for the construction of  the $R$ different axes constructions. Similarly, let $\mathbf{V}\in \mathbb{R}^{q\times R}$ be the $R$ first left singular vectors of $S_\lambda(\mathbf{Y}^T\mathbf{X}/(n-1))$. The associated components of $\bf X$ and $\bf Y$ are respectively denoted by $\mathbf{T=XU}\in \mathbb{R}^{n\times R}$ and $\mathbf{S=YV}\in \mathbb{R}^{n\times R}$.

It is relevant to have a regression model in the context of supervised learning in order to explain $\bf Y$ by $\bf X$. In the following a linear approximation, $\mathbf{Y\approx XB}$, is shown to be reasonable, where $\bf B$ is constructed using the \textbf{ct-sPLS} components $\bf T$ and $\bf S$. 

Some theoretical results are formulated hereafter.
%Section~\ref{sub:regul} shows the regularization property of  \textbf{ct-sPLS} method as $\lambda$ increases.
Section~\ref{sub:sparsity} proves that \textbf{ct-sPLS} method is a sparse method.
Section~\ref{sub:consistency} provides consistency results. This implies the considered method indeed performs regularization over the data and also allows building the associated regression model based on the \textbf{ct-sPLS} components.

%==========================================
%\subsection{A regularization method}
%\label{sub:regul}
%==========================================

%The following theorem shows that the information extracted from the data sets $\bf X$ and $\bf Y$ shrinks as the regularization parameter $\lambda$ increases. This parameter $\lambda$ is interpreted as the threshold value above which correlation interactions between variables of $\bf X$ and $\bf Y$ are taken into account.

%\begin{theorem}
%\label{th:contin}
%Let $(\mathbf{X},\mathbf{Y})\in \mathbb{R}^{n\times p}\times\mathbb{R}^{n\times q}$ assumed to be standardized.  The function $\beta$ defined as
%\[
%\begin{split}
%\beta:\lambda & \rightarrow \left\{\begin{aligned}
%& \underset{\mathbf{U}}{\max}
%& &\sum_{r=1}^R ||S_\lambda\Big(\frac{\mathbf{Y}^T\mathbf{X}}{n-1}\Big)\mathbf{u}^{(r)}||_2^2\\
%& \text{s.t.}
%& & \mathbf{U}^T\mathbf{U}=\mathbb{I}_R,
%\end{aligned}\right.
%\end{split}
%\]
%is decreasing and continuous on $[0,1]$, $\forall R\in[0,\min(\mbox{\normalfont rank}(\mathbf{X}),\mbox{\normalfont rank}(\mathbf{Y}))]$
%\end{theorem}
%The proof of this theorem is given in Appendix~\ref{app:proof_th1}.
%Note that only the range $[0,||\mathbf{Y}^T\mathbf{X}/(n-1)||_\infty[$ for $\lambda$ is really interesting. This range defines the positive values for $\lambda$ in which the proposed optimization problem (\ref{equ:ddspls_problem}) regularizes the data without giving the trivial solution, corresponding to the trivial problem of maximizing the null constant. 

%==========================================
\subsection{Sparsity}
\label{sub:sparsity}
%==========================================
The following theorem illustrates the ability of the \textbf{ct-sPLS} method to provide a sparse solution $\bf U$.
%For sake of simplicity, we consider the case where $R=1$, but the result can easily be extended to the multiple components case ($R> 1$).

\begin{theorem}\label{th:3}
Let $(\mathbf{X},\mathbf{Y})\in \mathbb{R}^{n\times p}\times\mathbb{R}^{n\times q}$ assumed to be standardized. Let $\lambda\in[0,1]$ such as $\mathbf{S}=S_\lambda\left(\frac{\mathbf{Y}^T\mathbf{X}}{n-1}\right)$ is not null and 
%$\mathbf{u}= \argmax\{\mathbf{u}^T\mathbf{S}^T\mathbf{S}\mathbf{u}|\mathbf{u}^T\mathbf{u}=1\}\in\mathbb{R}^{p}$ and 
%$\mathbf{v}= \argmax\{\mathbf{v}^T\mathbf{S}\mathbf{S}^T\mathbf{v}|\mathbf{v}^T\mathbf{v}=1\}\in\mathbb{R}^{q}$
$\mathbf{u}\in\mathbb{R}^{p}$ and $\mathbf{v}\in\mathbb{R}^{q}$ respectively the right and left singular vectors associated to the same, non null, singular value denoted  $\sqrt{\theta}$. Then, we have \[
\begin{split}
\forall i= \llbracket1,p\rrbracket: \quad <\mathbf{s}^{(i)},\mathbf{v}>=0 \iff u_i=0
\end{split}
\]
where ${u}_i$ stands for the $i^{th}$-element of $\mathbf{u}$ and
${\bf s}^{(i)}$ stands for the $i^{th}$-column of $\mathbf{S}$.
\end{theorem}
\begin{proof}
According to the definition of $\mathbf{u}$, $\mathbf{v}$ and $\sqrt{\theta}$, we get   $u_i=0\iff(\mathbf{S}^T\mathbf{S}\mathbf{u})_i=\theta u_i=0$. Moreover, $(\mathbf{S}^T\mathbf{S}u)_i={\mathbf{s}^{(i)}}^T\mathbf{S}\mathbf{u}=<{\mathbf{s}^{(i)}},\mathbf{S}\mathbf{u}>$. Since $\mathbf{S}\mathbf{u}=\sqrt{\theta}\mathbf{v}$, by definition of the left and right singular vectors, we get $<{\mathbf{s}^{(i)}},\mathbf{S}\mathbf{u}>=<{\mathbf{s}^{(i)}},\sqrt{\theta}\mathbf{v}>=\sqrt{\theta}<{\mathbf{s}^{(i)}},\mathbf{v}>=0\iff<{\mathbf{s}^{(i)}},\mathbf{v}>=0$.
\end{proof}

In Theorem~\ref{th:3}, the nullity of an element of weights $\bf u$ associated with $\mathbf{X}$ is the only one case considered. Equivalent results can straightforwardly obtained for the  weights $\bf v$ associated with $\mathbf{Y}$ considering $\mathbf{S}^T$ instead of $\mathbf{S}$.
This implies that the \textbf{ct-sPLS} method simultaneously selects variables in the $\mathbf{X}$ part and in the $\mathbf{Y}$ part. The condition obtained in the previous theorem is computationally unacceptable for the user and in practice \textbf{ct-sPLS} user would appreciate an upper-bound to the cardinality of $\mathbf{u}$ and $\mathbf{v}$, the number of variables selected for the $\bf X$ part and for the $\bf Y$ part respectively. The next corollary indicates that for all $\lambda$ in $[0,1]$, there exists an easy to compute upper bound to the number of variable selected for $\bf X$ and for $\bf Y$. Let $\mathbf{Card}({\bf u})=\#\{\mbox{Non null elements in }{\bf u}\}$.

\begin{corollary}
\label{coro:1}
Under assumptions of Theorem~\ref{th:3}, we have
\[
\begin{split}
\mathbf{Card}({\bf u})&\leq p-\#\{\textit{Null columns of $S_\lambda\left(\frac{\mathbf{Y}^T\mathbf{X}}{n-1}\right)$}\},\\
\mathbf{Card}({\bf v})&\leq q-\#\{\textit{Null rows of $S_\lambda\left(\frac{\mathbf{Y}^T\mathbf{X}}{n-1}\right)$}\}.
\end{split}
\]
\end{corollary}
\begin{proof}
Applying Theorem~\ref{th:3} and counting cases where the $i^{th}$ column (respectively the $j^{th}$ row) of $S_\lambda\left(\frac{\mathbf{Y}^T\mathbf{X}}{n-1}\right)$ are filled with null elements only, then the corresponding ${\bf u}$ (respectively ${\bf v}$) coefficients are equal to $0$, which demonstrates the corollary.
\end{proof}

The previous corollary is applied, for a given $\lambda$, to any singular-space of $S_\lambda(\frac{\mathbf{Y}^T\mathbf{X}}{n-1})$. This is only data-dependent, under the $\lambda$ user appreciation.
Note that the cardinality of any axis of $\bf U$ and $\bf V$ is not monotonic. In Appendix~\ref{app:1}, a simulation study exhibits a particular case, for $R=1$ considering the singular-space associated with the largest singular-value, in which the cardinality of $\bf u$ is not monotonic.

%==========================================
\subsection{Consistency}
\label{sub:consistency}
%==========================================

Under additional (theoretical) assumptions, we will show that it is relevant to consider a regression model in order to explain $\bf Y$ by $\bf X$ using the \textbf{ct-sPLS} components $\bf T=YV$ and $\bf S=XU$. 

As described by~\citet{johnstone2004sparse}, let us consider that $\bf X$ and $\bf Y$ follow two spiked covariance models: 
\begin{equation}
\left\{\begin{aligned}
{\bf X}&= {\bf L}{\bf \Omega}_x^{1/2}{\bf U}^T_{mod}+{\bf E}_x\\
{\bf Y}&= {\bf L}{\bf \Omega}_y^{1/2}{\bf V}^T_{mod}+{\bf E}_y\\
\end{aligned}
\right.,
\label{equ:johnstone_2004}
\end{equation}
where ${\bf \Omega}_x$ and ${\bf \Omega}_y$ are $R$-dimensional diagonal matrices with strictly positive diagonal elements, ${\bf U}_{mod}\in\mathbb{R}^{p\times R}$ and ${\bf V}_{mod}\in\mathbb{R}^{q\times R}$ are two matrices with orthonormal columns, ${\bf L}\in\mathbb{R}^{n\times R}$ is a matrix where elements are i.i.d. standard Gaussian random effects, ${\bf E}_x\in\mathbb{R}^{n\times p}$ (resp. ${\bf E}_y\in\mathbb{R}^{n\times q}$) is a matrix such that each row follows the standard multivariate normal distribution $N_p(0,\mathbb{I}_p)$ (resp. $N_q(0,\mathbb{I}_q)$) and the $n$ rows are independent and mutually independent noise vectors. Note that ${\bf L}$ introduces a common structure to $\bf X$ and $\bf Y$ models. The intuition developed by~\citet{JMLR:v17:15-160} is applied to ${\bf Y^TX}/(n-1)$ (instead of ${\bf X^TX}/n-\mathbb{I}_p$ in the original paper) and permits to use their Theorem 1 inferring the consistency of $S_\lambda({\bf Y^TX}/(n-1))$ to the population covariance ${\bf V}_{mod}{\bf \Omega}_y^{1/2}{\bf \Omega}_x^{1/2}{\bf U}^T_{mod}$. That result is of prior importance because it implies that the solution of the optimization problem~\eqref{equ:ddspls_problem} is such that
\begin{itemize}[topsep=0pt,itemsep=-1ex]
    \item $S_\lambda({\bf Y^TX}/(n-1)){\bf U}$ and ${\bf Y^TX}{\bf U}$ span approximately the same $R$-dimensional subspace of $\mathbb{R}^q$,
    \item $S_\lambda({\bf X^TY}/(n-1)){\bf V}$ and ${\bf X^TY}{\bf V}$ span approximately the same $R$-dimensional subspace of $\mathbb{R}^p$.
\end{itemize}
Therefore ${\bf U}$ and ${\bf V}$ permit to find a cross-subspace of $\mathbb{R}^n$ such as ${\bf S=XU}$ and ${\bf T=YV}$ span approximately the same $R$-dimensional subspace of $\mathbb{R}^n$.
From this result, it is then possible to construct a regression model to explain $\bf Y$ from $\bf X$ such as $\mathbf{Y\approx XB}$. The matrix  $\bf B$, obtained from the \textbf{ct-sPLS} components $\bf T$ and $\bf S$, is detailed in the multi-block framework, see Section~\ref{subsub:reg_multi}.
\section{Multi-Block Data-Driven sparse PLS (mdd-sPLS)}
\label{S:3}
%==================================================================
%==================================================================

In this section, let us focus on the multi-block context associated with the $T$ data matrices ${\bf X}_t\in \mathbb{R}^{n\times p_t}, t=1\dots, T$ and a single response matrix $\mathbf{Y}\in \mathbb{R}^{n\times q}$.
 Here missing samples might be in any of the $T$ blocks  ${\bf X}_t$, which means that any row of any block of covariates can be missing, but not in the response matrix $\bf Y$. 
 
 The  idea of the multi-block data-driven sPLS (\textbf{mdd-sPLS}) is to use \textbf{ct-sPLS} properties to find structure between the ${\bf X}_t$'s and the response matrix $\bf Y$ including variable selection in the context of missing samples.
Contrary to some existing multi-block approaches, the \textbf{mdd-sPLS} does not consider correlations between covariate blocks covariances but only covariate block ${\bf X}_t$  and response block $\bf Y$ covariances. Two sequential optimization problems are then defined, in terms of intra/inter-blocks weights, particularly stable and efficient in terms of prediction. Note that the proposed method is non ``component-wise'' in the sense that all the axes are built together and no deflation is used as to answer variance constraint problems discussed in the introduction and inherent to the deflation. 
Since missing samples can be both in the train sample and in the test sample, a new specific algorithm, called \textbf{Koh-Lanta} hereafter, has been developed. Note that this algorithm has been designed for the \textbf{mdd-sPLS} approach  but  can be extended to any supervised multi-block method.

Sections~\ref{sub:mddspls_crit} and \ref{subsub:reg_multi} provide the description of the \textbf{mdd-sPLS} method. 
The corresponding algorithms are given in Section~\ref{subsub:algor}.
The imputation algorithm \textbf{Koh-Lanta}  is described in Section~\ref{subs:koh_lanta}. The computational complexity of \textbf{mdd-sPLS} is studied in Section~\ref{sub:compcomp}. Finally the adaptation to classification (rather than regression) is discussed and illustrated in Section~\ref{sec:adapt_class}.

%=======================================================
\subsection{Description}% of the \textbf{mdd-sPLS} method}
\label{sub:mddspls_crit}
%=======================================================

Let
$R\in \llbracket0, \min( \min_{t\in\llbracket1, T\rrbracket} ( \mbox{rank}(\mathbf{X}_t)), \mbox{rank}(\mathbf{Y}))\rrbracket$.
The \textbf{mdd-sPLS} approach splits into three steps. 

\begin{itemize}
    \item {\sc Step 1: solve independently the $T$  \textbf{ct-sPLS} 
    problems based on ${\bf X}_t$ and $\bf Y$.} 
    
    It provides the weights $(\mathbf{U}_t)_{t\in\llbracket 1,T\rrbracket}$
    $$\mathbf{U}_t := \textbf{ct-sPLS}(\mathbf{X}_t,\mathbf{Y},\lambda,R)=\mathbf{SVD}_R(\mathbf{M}_t)\in \mathbb{R}^{p_t\times R},$$
    where $\mathbf{M}_t=S_\lambda\big(\frac{\mathbf{Y}^T\mathbf{X}_t}{n-1}\big).$
    
\item {\sc Step 2: combine information from the $T$ blocks.}
    
    Let $\mathbf{Z}=[\mathbf{M}_1\mathbf{U}_1,\cdots,\mathbf{M}_T\mathbf{U}_T]$. In order to combine information from all the blocks and to recover a common $R$-dimensional subspace, we introduce the super-weights $(\underline{\boldsymbol{\mathbf{\beta}}}_t)_{t\in\llbracket 1,T\rrbracket}\in\mathbb{R}^{R\times R}$ associated with block $t$ defined as
    $$\underline{\boldsymbol{\beta}}=\big[\underline{\boldsymbol{\mathbf{\beta}}}_1^T, \cdots,\underline{\boldsymbol{\mathbf{\beta}}}_T^T\big]^T=\mathbf{SVD}_R(\mathbf{Z})\in \mathbb{R}^{RT\times R}.$$
 %   where  and $\underline{\boldsymbol{\mathbf{\beta}}}_t   =  \big[{{\boldsymbol\beta}_t^{({1})}}^T, \cdots,{{\boldsymbol\beta}_t^{({R})}}^T\big]^T.
    
\item {\sc Step 3: predict $\bf Y$ using a regression model.}

The weights $(\mathbf{U}_t)_{t\in\llbracket 1,T\rrbracket}$ and the super-weights $(\underline{\boldsymbol{\mathbf{\beta}}}_t)_{t\in\llbracket 1,T\rrbracket}$ are used to build the matrices $(\mathbf{B}_t)_{t\in\llbracket 1,T\rrbracket}$ approximating the regression model 
    \begin{equation}
\mathbf{Y}\approx\sum_{t=1}^T\mathbf{X}_{t}\mathbf{B}_t \in \mathbb{R}^{{n}\times {q}},
\label{equ:multi_reg_Bt}
\end{equation}
This step is detailed in the following section.
\end{itemize}

\begin{remark}
If there exists a covariate matrix ${\bf X}_{t^*}$ such that $S_\lambda\big(\mathbf{Y}^T\mathbf{X}_{t^*}/(n-1)\big)=\mathbf{0}$, the corresponding solution to the first step of \textbf{mdd-sPLS} is the null matrix. The same process is applied if more that one matrix is null. If all the matrices are null, the solution is the empty solution.
\end{remark}

\subsection{Details on Method's Third Step}
\label{subsub:reg_multi}
%=======================================================

The underlying regression model is constructed with the same statistical assumptions as described by~\citet{johnstone2004sparse}, already used in~\eqref{equ:johnstone_2004} in the mono-block ({\bf ct-sPLS}) context and denoted as spike covariance models
\begin{equation*}
\left\{\begin{aligned}
{\bf X}_1&= {\bf L}{\bf \Omega}_1^{1/2}{\bf U}^T_{1,mod}+{\bf E}_1\\
&\vdots\\
{\bf X}_T&= {\bf L}{\bf \Omega}_T^{1/2}{\bf U}^T_{T,mod}+{\bf E}_T\\
{\bf Y}&= {\bf L}{\bf \Omega}_y^{1/2}{\bf V}^T_{mod}+{\bf E}_y\\
\end{aligned}
\right.
\label{equ:johnstone_2004_multi}
\end{equation*}
where $({\bf \Omega}_t)_{t\in\llbracket1,T\rrbracket}$ and ${\bf \Omega}_y$ are $R$-dimensional diagonal matrices with strictly positive diagonal elements. $({\bf U}_{t,mod}\in\mathbb{R}^{p_t\times R})_{t\in\llbracket1,T\rrbracket}$ and ${\bf V}_{mod}\in\mathbb{R}^{q\times R}$ are  matrices with orthonormal columns. ${\bf L}\in\mathbb{R}^{n\times R}$ is a matrix where elements are i.i.d. standard Gaussian random effects, $({\bf E}_t\in\mathbb{R}^{n\times p_t})_{t\in\llbracket1,T\rrbracket}$ (respectively ${\bf E}_y\in\mathbb{R}^{n\times q}$) are matrices such that each row follows the standard multivariate normal distribution $(N_{p_t}(0,\mathbb{I}_{p_t}))_{t\in\llbracket1,T\rrbracket}$ (respectively $N_q(0,\mathbb{I}_q)$) and the $n$ rows are independent and mutually independent noise vectors. Let us mention that the matrix ${\bf L}$ does not depend of $t$ and thus introduces a common structure between the ${\bf X}_t$'s and $\bf Y$ models. Moreover let us recall that the matrix ${\boldsymbol \beta}$ of super-weights  gathers information from the different covariates corresponding to the different blocks.

Under these assumptions, using Theorem 1 from~\citet{JMLR:v17:15-160} and extending the context of Section~\ref{sub:sparsity} to $T$ soft-thresholded matrices  $S_\lambda(\mathbf{Y}^T\mathbf{X}_t/(n-1))$, Theorem 1 insures that  $\mathbf{U}_t\in\mathbb{R}^{p_t\times R}$ (the optimal $R$-dimensional right-singular matrix of $S_\lambda(\mathbf{Y}^T\mathbf{X}_t/(n-1))$) is close to the optimal $R$-dimensional right-singular matrix of $\mathbf{Y}^T\mathbf{X}_t$. 

Let us introduce the following notations
\begin{equation*}
\left\lbrace
\begin{array}{ll}
% \mathfrak{t}^{(r)}&=\quad\sum_{t=1}^T\mathbf{X}_{t}u_t^{(r)}\beta_t^{(r)}\\
\mathbf{U}_{t,super}&=\mathbf{U}_t \underline{\boldsymbol{\mathbf{\beta}}}_t\in\mathbb{R}^{p_t\times R}\\
\mathbf{T}_{super}&=\sum_{t=1}^T\mathbf{X}_{t}\mathbf{U}_{t,super}\in\mathbb{R}^{n\times R}\\
\mathbf{V}_{super}&=norm_2(\mathbf{Z}\underline{\boldsymbol{\mathbf{\beta}}})\in\mathbb{R}^{q\times R}\\
\mathbf{S}_{super}&=\mathbf{YV}_{super}\in\mathbb{R}^{n\times R}\\
\end{array}
\right.,
\label{equ:equMulti_5_ref_last_t_mono}
\end{equation*}
where $norm_2$ is the function that returns the columns normalized to a $\mathcal{L}_2$-norm equal to 1 if the corresponding column is non null and 0 otherwise.

 The aim is to provide weights that favor the most predictive directions ${\bf u}_{t}^{(r)}$ for $\bf Y$. The matrix $\mathbf{U}_{t,super}$ is the super-weight corresponding to the $t^{th}$-block. The matrix $\mathbf{T}_{super}$ is the super-component for covariate part, which is the most predictive of $\bf Y$. The matrix ${\bf V}_{super}$ is the weight enabling to build the component ${\bf S}_{super}$ of the response part. 
 The component $\mathbf{T}_{super}$ and ${\bf S}_{super}$ describes the $n$ individuals from the point of view of the covariates and the response. Let us write the regression model
\begin{equation}
{\bf S}_{super}={\bf T}_{super}\mathbf{B}_{0,mod} + {\bf E}_{ort},
\label{equ:multi_reg_Bt_ort}
\end{equation}
where ${\bf E}_{ort}\in\mathbb{R}^{n\times R}$ a residual matrix and $\mathbf{B}_{0,mod}\in\mathbb{R}^{R\times R}$ the  matrix of parameters. Freedom is given to the user to determine if the model is sufficiently informative. Let us introduce ${\bf V}_{ort}=\mathbf{SVD}_R({\bf T}_{super})$ and ${\bf \Delta}_{ort}$ the diagonal matrix filled with the corresponding square singular values. Using Moore-Penrose pseudo-inverse of ${{\bf T}_{super}}^T{\bf T}_{super}$, the parameters $\mathbf{B}_{0,mod}$ can be estimated by
%Multiplying equation~\eqref{equ:equMulti_5_ref_last_t_mono} on the left by ${\bf T}_{super}^T$, assuming that the \textbf{SVD} of ${\bf T}_{super}^T{\bf T}_{super}$ is a sufficient approximation and using the Moore-Penrose inverse, one can estimate, applying ordinary-least-square minimizing ${\bf E}_{ort}$ 
\begin{equation*}
{\bf B}_{0}={\bf V}_{ort}{\bf \Delta}_{ort}^{-1}{\bf V}_{ort}^T{\bf T}_{super}^T{\bf S}_{super},
%\label{equ:multi_reg_Bt_ort_est_sol}
\end{equation*}
which is the best solution in the regression problem~\eqref{equ:multi_reg_Bt_ort} according to~\citet{penrose1956best}. 
%their cross-product, defines three R-dimensional square matrices $({\bf U}_{ort},{\bf V}_{ort},{\bf \Delta}_{ort})$, where the two first are orthogonal and the last one is diagonal, such as
%\begin{equation}
%    {\bf S}_{super}^T{\bf T}_{super}\approx{\bf U}_{ort}{\bf \Delta}_{ort}{\bf V}_{ort}^T,
%    \label{equ:multi_reg_Bt_ort_svd}
%\end{equation}
%Let be denoted by ${\bf U}_{ort}$ and ${\bf V}_{ort}$ respectively the left and the right R-dimensional singular matrices associated with the cross-product ${\bf S}_{super}^T{\bf T}_{super}$.
%
%where ${\bf S}_{super}$ and ${\bf T}_{super}$ have been designed to share a linear model, one can design the regression model 
%where ${\bf E}_{ort}\in\mathbb{R}^{n\times R}$ a residual matrix and $\mathbf{B}_{0,mod}\in\mathbb{R}^{R\times R}$ a coefficient matrix. Multiplying both parts of~\eqref{equ:multi_reg_Bt_ort} on the left by ${\bf S}_{super}^T$ and using~\eqref{equ:multi_reg_Bt_ort_svd}, $\mathbf{B}_{0,mod}$ can be estimated, using Moore-Penrose inverse of ${\bf S}_{super}^T{\bf T}_{super}$ and applying ordinary-least-square minimizing ${\bf E}_{ort}$
%\begin{equation*}
%{\bf B}_{0}={\bf V}_{ort}{\bf \Delta}_{ort}^{-1}{\bf U}_{ort}^T{\bf S}_{super}^T{\bf %S}_{super}.
%\label{equ:multi_reg_Bt_ort_est_sol}
%\end{equation*}
One can therefore rewrite~\eqref{equ:multi_reg_Bt_ort} as
$${\bf S}_{super}={\bf YV}_{super}  \approx \sum_{t=1}^T{\bf X}_{t}{\bf U}_{t,super}{\bf B}_{0}.$$
And so
   $${\bf Y}\approx{\bf YV}_{super}{\bf V}_{super}^T\approx \sum_{t=1}^T{\bf X}_{t}{\bf U}_{t,super}{\bf B}_{0}{\bf V}_{super}^T.$$
 This approximation is discussed in the next remark. Finally, by identification with~\eqref{equ:multi_reg_Bt}, one obtains
   \begin{equation}
   {\bf B}_{t} ={\bf U}_{t,super}{\bf B}_{0}{\bf V}_{super}^T.
\label{equ:multi_reg_Bt_ort_est_sol}
\end{equation}

\begin{remark}
Since  the application ${\bf Y} \rightarrow {\bf Y}{\bf V}_{super}{\bf V}_{super}^T$ is a projection on the common subspace between the response and the covariate matrices, it is usual~\citep[see for example][]{manne1987analysis} to write, instead of~\eqref{equ:multi_reg_Bt_ort_est_sol}, the approximation 
${\bf Y}\approx\sum_{t=1}^T{\bf X}_{t}{\bf B}_{t}$ already introduced in (\ref{equ:multi_reg_Bt}).
This notation takes into account that no better regression model is accessible in the context of \textbf{PLS} modelling. The proposed method allows to know which variable of the \textbf{Y} part is indeed predicted by the model, since the Theorem~\ref{th:3} insures the {\bf mdd-sPLS} method selects variables both in the $\mathbf{X}_t$'s and in $\textbf{Y}$. 
\end{remark}

\subsection{Algorithms of the Multi-Data-Driven-sPLS}
\label{subsub:algor}
%=======================================================

First, the  algorithm of the \textbf{ct-sPLS} method is provided in Algorithm~\ref{alg:ddspls_algo}.
 Then, the algorithm of \textbf{mdd-sPLS} approach is described in Algorithm~\ref{alg:mddspls_algo}.
 In the following,  $\odot$ is the Hadamard product $\mathbf{A}//\mathbf{B}$ denotes the term-to-term division operator of matrices (or vectors) $\mathbf{A}$ and $\mathbf{B}$.
Let $\boldsymbol{1}_n$ denote the $n$-dimensional vector of 1's and $diag$ which returns, for a given square matrix $\mathbf{X}$, the row vector of the diagonal elements of $\mathbf{X}$.
The notation $\mathcal{M}[.]$ makes it possible to extract any attribute of the considered object obtained from \textbf{ct-sPLS} or \textbf{mdd-sPLS} method, the available attributes correspond to the outputs of the corresponding algorithms. 
%${\boldsymbol \mu}$ represents the vector of the mean values and ${\boldsymbol \sigma}$ the vector of the variance values.
\begin{algorithm}[!ht]
\caption{ct-sPLS}\label{alg:ddspls_algo}
\begin{algorithmic}[1]
\Procedure{$\text{ct-sPLS}$}{$\mathbf{X},\mathbf{Y},\uplambda,R$}
%\Initialize{n the number of individuals}
\State $({\boldsymbol \mu}_x,{\boldsymbol \mu}_y)\gets \frac{1}{n}( \boldsymbol{1}_n^T\mathbf{X},\boldsymbol{1}_n^T\mathbf{Y})$
\State $({\boldsymbol \sigma}_x^2,{\boldsymbol \sigma}_y^2)\gets \frac{1}{n-1}\Big(\boldsymbol{1}_n^T\big((\mathbf{X}-\boldsymbol{1}_n{\boldsymbol \mu}_x)\odot(\mathbf{X}-\boldsymbol{1}_n{\boldsymbol \mu}_x)\big),\boldsymbol{1}_n^T\big((\mathbf{Y}-\boldsymbol{1}_n{\boldsymbol \mu}_y)\odot(\mathbf{Y}-\boldsymbol{1}_n{\boldsymbol \mu}_y)\big)\Big)$
\State $\mathbf{X}\gets (\mathbf{X}-\boldsymbol{1}_n{\boldsymbol \mu}_x)//({\boldsymbol{1}_n\boldsymbol \sigma}_x),\quad \mathbf{Y}_i\gets (\mathbf{Y}_i-{\boldsymbol \mu}_y)//({\boldsymbol{1}_n\boldsymbol \sigma}_y)$
\State $\mathbf{M}=S_\uplambda\big(\frac{\mathbf{Y}^T\mathbf{X}}{n-1}\big)$,\quad $\mathbf{U}\gets \mathbf{SVD}_R(\mathbf{M})$,\quad $\mathbf{V}\gets norm_2(\mathbf{M}\mathbf{U})$
%\State $\mathbf{T}\gets \mathbf{X}\mathbf{U}$,\quad $\mathbf{S}\gets \mathbf{Y}\mathbf{V}$
%\State $(\mathbf{U}_{ort},\mathbf{V}_{ort})\gets \mathbf{SVD}_R(\mathbf{T},\mathbf{S})$,\quad $\mathbf{T}_{ort}\gets \mathbf{T}\mathbf{U}_{ort}$,\quad $\mathbf{S}_{ort}\gets\mathbf{S}\mathbf{V}_{ort}$
%\State $\mathbf{A}\gets diag\big((\dfrac{<\mathbf{s}_{ort}^{(r)},\mathbf{t}_{ort}^{(r)}>}{||\mathbf{t}_{ort}^{(t)}||_2^2})_{r\in\llbracket1,R\rrbracket}\big)$,\quad $\mathbf{B}\gets \mathbf{U}\mathbf{U}_{ort}\mathbf{A}\mathbf{V}_{ort}^T\mathbf{V}^T$
\State \textbf{return} $(\mathbf{U},\mathbf{V},%\mathbf{T},\mathbf{S},\mathbf{T}_{ort},\mathbf{S}_{ort},\mathbf{B},
\mathbf{M}$,${\boldsymbol \mu}_x,{\boldsymbol \sigma}_x,{\boldsymbol \mu}_y,{\boldsymbol \sigma}_y,\mathbf{X},\mathbf{Y})$
\EndProcedure
\end{algorithmic}
  \end{algorithm}

\begin{algorithm}[!ht]
\caption{mdd-sPLS}\label{alg:mddspls_algo}
\begin{algorithmic}[1]
\Procedure{mdd-sPLS}{$\mathbf{X}=\{\mathbf{X_t}\}_{t\in\llbracket1,T\rrbracket},\mathbf{Y},\lambda,R$}
\For{$t\in\llbracket1,T\rrbracket$}
\State $\mathcal{M}_t\gets\texttt{ct-sPLS}(\mathbf{X}_t,\mathbf{Y},\lambda,R)$
\EndFor
\State $({\boldsymbol \mu}_x,{\boldsymbol \mu}_y)\gets\big((\mathcal{M}_t[{\boldsymbol \mu}_x])_{t\in\llbracket 1,T\rrbracket},\mathcal{M}_1[{\boldsymbol \mu}_y]\big)$
\State $({\boldsymbol \sigma}_x,{\boldsymbol \sigma}_y)\gets\big((\mathcal{M}_t[{\boldsymbol \sigma}_x])_{t\in\llbracket 1,T\rrbracket},\mathcal{M}_1[{\boldsymbol \sigma}_y]\big)$
\State $\mathbf{M}\gets\big(\mathcal{M}_t[\mathbf{M}]\big)_{t\in\llbracket 1,T\rrbracket}$
\State $\underline{\boldsymbol{\beta}}\gets\mathbf{SVD}_R(\mathbf{Z})$
\State $\forall t\in\llbracket1,T\rrbracket,\quad\mathbf{U}_{t,super}\gets\mathcal{M}_t[\mathbf{U}] \underline{\boldsymbol{\mathbf{\beta}}}_t$, \quad $\mathbf{T}_{super}\gets\sum_{t=1}^T\mathbf{X}_{t}\mathbf{U}_{t,super}$
\State $\mathbf{V}_{super}\gets norm_2(\mathbf{Z}\underline{\boldsymbol{\beta}})$, \quad $\mathbf{S}_{super}\gets\mathbf{Y}\mathbf{V}_{super}$
\State ${\bf V}_{ort}\gets\mathbf{SVD}_R({\bf T}_{super})$,\quad ${\boldsymbol\Delta}_{ort}\gets({\bf T}_{super}{\bf V}_{ort})^T{\bf T}_{super}{\bf V}_{ort}$
%\State $(\mathbf{U}_{ort},\mathbf{V}_{ort})\gets \mathbf{SVD}_R(\mathbf{T}_{super},\mathbf{S}_{super})$,\quad $\mathbf{T}_{ort}\gets \mathbf{T}\mathbf{U}_{ort}$,\quad $\mathbf{S}_{ort}\gets\mathbf{S}\mathbf{V}_{ort}$
\State ${\bf B}_{0}={\bf V}_{ort}{\bf \Delta}_{ort}^{-1}{\bf V}_{ort}^T{\bf T}_{super}^T{\bf S}_{super}$, \quad  $\mathcal{B}\gets\big({\bf U}_{t,super}{\bf B}_{0}{\bf V}_{super}^T\big)_{t\in\rrbracket 1,T\llbracket}$
\State \textbf{return} $(\mathbf{U}=(\mathcal{M}_t[\mathbf{U}])_{t\in\llbracket1,t\rrbracket},\mathbf{V}=(\mathcal{M}_t[\mathbf{V}])_{t\in\llbracket1,t\rrbracket},\mathbf{T}_{super},\mathbf{S}_{super},\mathcal{B},\mathbf{M},\boldsymbol{\mu}_x,\boldsymbol{\sigma}_x,\boldsymbol{\mu}_y,\boldsymbol{\sigma}_y)$
\EndProcedure
\end{algorithmic}
  \end{algorithm}  

Let $\mathcal{M}$ be the \textbf{mdd-sPLS} model built on train data sets $(\mathbf{X}_1^{(train)},\dots,\mathbf{X}_T^{(train)},\mathbf{Y}^{(train)})$. 
Given test data sets $\mathbf{X}^{(test)}:=(\mathbf{X}_1^{(test)},\dots,\mathbf{X}_T^{(test)})$ of $m$ individuals, the prediction operator, denoted by $\mathcal{P}$, allows to estimate the response $\mathbf{Y}^{(test)}$ by
\begin{equation*}
\mathcal{P}(\mathbf{X}^{(test)},\mathcal{M})=
\Big(\sum_{t=1}^T\big ((\mathbf{X}^{(test)}_{t,i}-\mathcal{M}[{\boldsymbol \mu}_x]_t)\odot(\mathcal{M}[{\boldsymbol \sigma}_y]//\mathcal{M}[{\boldsymbol \sigma}_x]_t)\big )\mathcal{M}[\mathcal{B}]_t+\mathcal{M}[{\boldsymbol \mu}_y]\Big)_{i\in\llbracket 1,m\rrbracket},
\end{equation*}
where $\mathbf{X}^{(test)}_{t,i}$ denotes the row vector containing the information in the data set $\mathbf{X}^{(test)}$ relative to block $t$ and individual $i$.
The treatment of missing values is described in the next section through the {\bf Koh-Lanta} algorithm. In the mono-block case without missing values, the classical regression data set, called \textbf{liver toxicity}, has been used to illustrate the behavior of the method and the corresponding results are provided in Appendix~\ref{app:Reg}.

%=======================================================
\subsection{Koh-Lanta Algorithm: Impute \textit{Train} \& \textit{Test} Data Sets}
\label{subs:koh_lanta}
%=======================================================

For most machine learning procedures, the user splits the available data into a \textit{train} data set used to build the model (in order to recover the underlying  structure) and a \textit{test} data set allowing to evaluate the validity and the precision of the model. Let $m$ 
denote the number of individuals in the \textit{train} data set. Without loss of generality the $m$ first individuals among the $n$ individuals are in the \textit{train} data set. In the following, mathematical symbols powered by a $o$ concern the \textit{train} part and mathematical symbols powered by $b$ concern the \textit{test} part. 
Missing samples can appear in the \textit{train} and/or in the \textit{test} data set.
A visual presentation of the data is given in Figure~\ref{fig:predict_str}.

\begin{figure}[p]
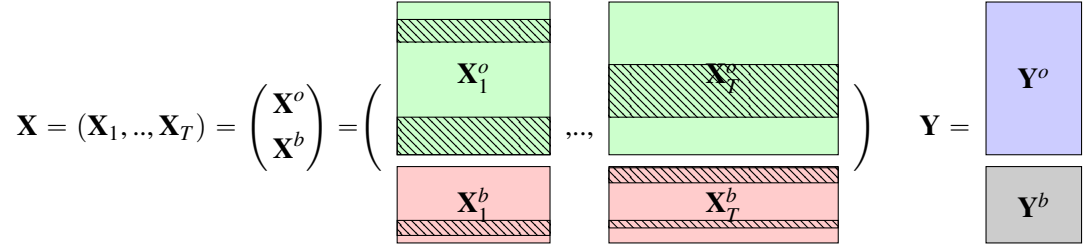

\begin{center}
\begin{blockmatrixtabular}
  \valignbox{
  $\mathbf{X}=(\mathbf{X}_1,..,\mathbf{X}_T)=
  \left(\begin{aligned}
  \mathbf{X}^o\\
  \mathbf{X}^b
  \end{aligned}\right)=
  $\Bigg(
  \begin{blockmatrixtabular}
  \fblockmatrixXaTrain     [0.8,1.0,0.8]{0.8in}{0.8in}{$\mathbf{X}_1^o$}\\
  \fblockmatrixXaTest     [1.0,0.8,0.8]{0.8in}{0.4in}{$\mathbf{X}_1^b$}
  \end{blockmatrixtabular}
  ,..,
    \begin{blockmatrixtabular}
  \fblockmatrixXbTrain     [0.8,1.0,0.8]{1.2in}{0.8in}{$\mathbf{X}_T^o$}\\
  \fblockmatrixXbTest     [1.0,0.8,0.8]{1.2in}{0.4in}{$\mathbf{X}_T^b$}
  \end{blockmatrixtabular}
  \Bigg)
  }&
  \valignbox{
  $\quad \mathbf{Y}=$
  \begin{blockmatrixtabular}
  \fblockmatrixY       [0.8,0.8,1.0]{0.5in}{0.8in}{$\mathbf{Y}^o$}\\
  \fblockmatrixY       [0.8,0.8,0.8]{0.5in}{0.4in}{$\mathbf{Y}^b$}
  \end{blockmatrixtabular}
  }
\end{blockmatrixtabular}\\
\end{center}
\caption{Structure of a $T$-blocks multi-block data set for \textit{train} (resp.\textit{test}) part, denoted by ``$o$'' (resp. `$b$'') symbol. Missing values are symbolized by hatched areas.
\label{fig:predict_str}}
\end{figure}

%\textit{train} and \textit{test} parts differ mainly by the absence of the response in the $\textit{test}$ part, denoted by $\textbf{Y}^b$.
For any data set $\bf X$ the following notation $\mathbf{X}_{\texttt{blocks}|\texttt{indiv}|\texttt{variables}}$ allows to extract the information from $\bf X$ relative to the blocks indexed by \texttt{blocks}, to the individuals indexed $\texttt{indiv}$ and to the variables indexed by $\texttt{variables}$. The use of a dot as index for any of those indices means that all the indexes are taken into account. For instance, $\mathbf{X}_{\texttt{blocks}|\texttt{indiv}|.}$ takes all the variables for blocks in $\texttt{blocks}$ and individuals in $\texttt{indiv}$. To distinguish \textit{train} and \textit{test} parts of the data set, the following notations are introduced
\begin{equation*}
\begin{array}{|lcl|lcl|}
\hline
\multicolumn{3}{|c|}{\text{Train part}} & \multicolumn{3}{c|}{\text{Test part}}\\
\hline
\mathcal{I}^o &= &\{i=1,\dots,m\}&\quad
\mathcal{I}^b &= &\{i=(m+1),\dots,n\}\\
\mathcal{I}_t^o &= &\{i\in \mathcal{I}_{o}|\mathbf{X}_{t|i|.}\text{ is missing}\}&\quad
\mathcal{I}_t^b &= &\{i\in \mathcal{I}|\mathbf{X}_{t|i|.}\text{ is missing}\}\\
\mathcal{J}_t^o &= &\{i\in \mathcal{I}_{o}|\mathbf{X}_{t|i|.}\text{ is present}\}&\quad
\mathcal{J}_t^b &= &\{i\in \mathcal{I}|\mathbf{X}_{t|i|.}\text{ is present}\}\\
\hline
\end{array}
\end{equation*}
Moreover let us also define $\forall i \in \llbracket1,n\rrbracket,\quad \mathcal{K}_i = \{t\in \llbracket1,T\rrbracket|\mathbf{X}_{t|i|.}\text{ is missing}\}$.

The objective of the \textbf{Koh-Lanta} algorithm is to impute predicted values in place of missing samples in the \textit{train} data set and in the \textit{test} data set. Two stages have been designed to solve that problem. The first stage, denoted as ``The Tribe Stage'', imputes in the \textit{train} data set. It uses an EM based algorithm, alternating between estimating the general model $\mathcal{M}^\star$ and using that model for imputation of $\mathbf{X}_{.|\mathcal{I}^o|.}^\star$. The second stage, denoted as ``The Reunification Stage'', allows predicting the potential missing values of the \textit{test} data set. Those two stages are detailed below.

%=======================================================
\subsubsection{Train-Data Imputation and Model Construction: The Tribe Stage}
%=======================================================

The ``Tribe Stage'' can be described thanks to the following algorithm.
\begin{siderules}
\begin{enumerate}[topsep=0pt,itemsep=-1ex]
\setcounter{enumi}{-1}
\item $\forall t\in \llbracket1,T\rrbracket, \mathbf{X}_{t|\mathcal{I}^o_t|.}^\star$ are imputed to the mean variables estimated on $ \mathbf{X}_{t|\mathcal{J}^o_t|.}^\star$,
\item \text{Model construction: }
\begin{itemize}[topsep=0pt,itemsep=-1ex]
\item $\mathcal{M}^\star\gets \text{Mdd-sPLS}(\mathbf{X}_{.|\mathcal{I}^o|.}^\star,\mathbf{Y}_{\mathcal{I}^o},\lambda,R)$,
\item $\mathcal{V}^\star \gets \big\{\mathcal{V}^\star_t=\{j\in 1..p_t|\mathcal{M}^\star[\mathbf{U}]_t\neq 0\}\big\}_{t\in \llbracket1,T\rrbracket}$,
\end{itemize}
\item \text{Imputation process,} $\forall t\in \llbracket1,T\rrbracket|\mathcal{I}_t \neq\varnothing$:
\begin{itemize}[topsep=0pt,itemsep=-1ex]
\item $\mathcal{M}_t\gets \text{Mdd-sPLS}(\mathcal{M}^\star[\mathbf{S}_{super}]_{t|\mathcal{J}_t^o},\mathbf{X}_{t|\mathcal{J}_t^o|\mathcal{V}^\star_t},\lambda,R)$,
\item $\mathbf{X}_{t|\mathcal{I}_t^o|\mathcal{V}^\star_t}^\star\gets \mathcal{P}(\mathcal{M}^\star[\mathbf{S}_{super}]_{t|\mathcal{I}_t^o},\mathcal{M}_t)$
\end{itemize}
\item Back to 1. until convergence of $\mathcal{M}^\star[\mathbf{T}_{super}]$.
\item Return $(\mathcal{M}^\star,\mathbf{X}^\star,\mathcal{V}^\star)$
\end{enumerate}
\end{siderules}
Note that $\mathcal{M}^\star$ is always learned with $\mathbf{X}_{.|\mathcal{I}^o|.}^\star$ as predictors and $\mathbf{Y}_{\mathcal{I}^o}$ as response variables. In the imputation process, only variables in $\mathcal{V}^\star$ are considered because the objective is to impute, for a given block, only the variables on which the learning has been done. So this  imputation stage takes into account only the best features, the best elements for each block, for each tribe. This is why this step is called ``The Tribe Stage''.

%=======================================================
\subsubsection{Test-Data Imputation and Prediction: The Reunification Stage}
%=======================================================

The ``Reunification Stage'' can be described thanks to the following algorithm.
\begin{siderules}
\begin{itemize}[topsep=0pt,itemsep=-1ex]
\item $\forall i \in \mathcal{I}^b$:
\begin{enumerate}[topsep=0pt,itemsep=-1ex]
\item $\mathcal{M}_i\gets\text{Mdd-sPLS}(\mathcal{M}^\star[\mathbf{T}_{super}]_{\bar{K_i}|\mathcal{I}^0},\mathbf{X}_{K_i|\mathcal{I}^0|\mathcal{V}^\star_{K_i}},\lambda,R)$
\item $\mathbf{X}_{K_i|i|\mathcal{V}^\star_{\bar{K_i}} }^\star\gets \mathcal{P}(\mathcal{M}^\star[\mathbf{S}_{super}]_{t|\mathcal{I}_t^o},\mathcal{M}_i)$
\end{enumerate}
\item $\mathbf{Y}_{\mathcal{I}^b} \gets \mathcal{P}(\cup\{\mathbf{X}_{\bar{K_i}|\mathcal{I}^b|. },\mathbf{X}_{K_i|\mathcal{I}^b|.}^\star\},\mathcal{M}^\star)$
\item Return $\mathbf{Y}_{\mathcal{I}^b}$.
\end{itemize}
\end{siderules}
where $\mathcal{M}^\star$, $\mathbf{X}^\star$ and $\mathcal{V}^\star$ were obtained in the ``Tribe Stage''. This step is based on the model $\mathcal{M}^\star$ which is then used to impute missing values of the \textit{test} data set and then to predict its response.

%=======================================================
\subsection{Computational Complexity}
\label{sub:compcomp}
%=======================================================

In that section, for sake of simplicity, let $p$ denote the order of magnitude of the number of covariables in the $\textbf{X}_t$'s blocks. 
Only the highest operations in terms of computational time have been taken into account hereafter, operations with at least quadratic terms in $q$ or $p$. Plus it is supposed that the number $R$ of components computed is largely smaller than $q$, $p$ or even $n$.

Let us focus on the case where there are no missing values. 
The {\bf mdd-SPLS} algorithm needs in its first step the computation of $T$ \textbf{ct-sPLS}, only covariance and \textbf{SVD} computations are greedy, the corresponding complexity is then 
$O(nTqp+Tpq\min(p,q))$=$O(Tpq(n+min(p,q))$. 
In its second step, one \textbf{SVD} is performed and the associated complexity  is $O(Tpq\min(p,q))$.
In the third step,  another \textbf{SVD} is performed to get regression parameters but the corresponding complexity is $O(n)$ and so negligible against other.  Finally, the total computational complexity is
\begin{equation*}
    O\big(Tqp(n+\min(p,q))\big).
    \label{equ:comput}
\end{equation*}
The computation of the $T$ \textbf{SVD} in the first step clearly dominates the computational complexity. Note that different cases may arise:
\begin{itemize}[topsep=0pt,itemsep=-1ex]
    \item $n>>\max(q,p)$: this is the classical context of data analysis. The total complexity is $O(nTpq)$, which implies linearity in each of the parameters. Let us remark that the behavior is the same if $n>>\min(q,p)$.
    \item $p>>n$: this is the high dimensional context. If
    \begin{itemize}[topsep=0pt,itemsep=-1ex]
    \item $n>>q$, the total complexity becomes $O(nTqp)$ which is again linear in each of the parameters. \item $n<<q$, the total complexity is now $O(Tq^2p)$ which is quadratic in $q$ and linear in $T$ and $p$.
    \end{itemize}
\end{itemize}

 Figure~\ref{fig:time_complexity} shows computation times of the {\bf mdd-SPLS} algorithm according to different values of the parameters $n$, $T$, $p$ and $q$.
 For each set of parameters, 20 simulations were performed.
 The observed numerical results is clearly consistent with the previous theoretical total complexity.
 The numerical/theoretical results of the case ``$p>>n$ and $q>p$'' are not provided here but are similar to the case ``$p>>n$ and $q<p$" with linear behavior in $q$ and quadratic behavior in $p$.   
 %The number of components build is equal to $\min(n,q,p)$.

\begin{figure}[p]
	\centering
	\includegraphics[width = \textwidth]{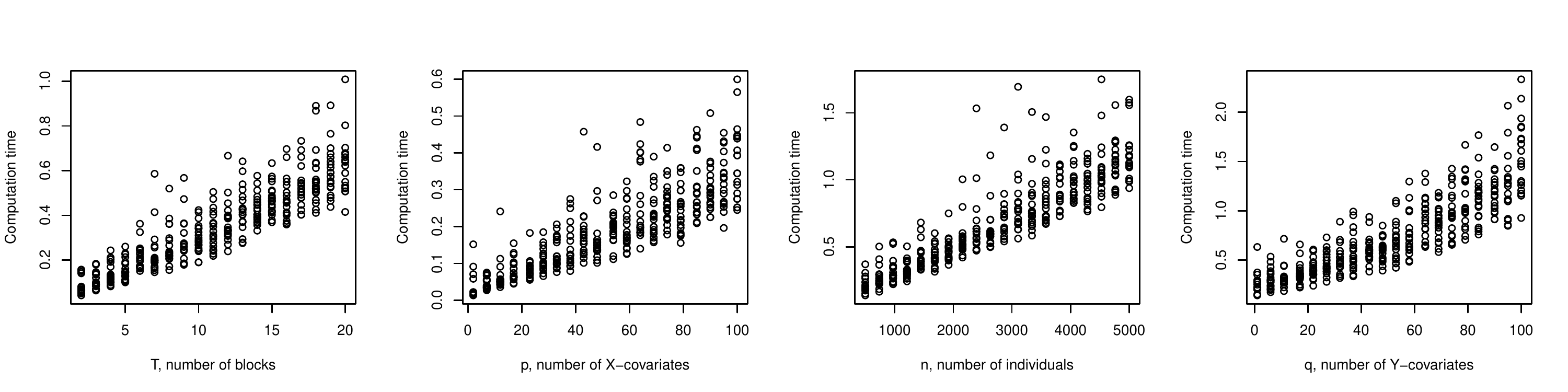}
	\includegraphics[width = \textwidth]{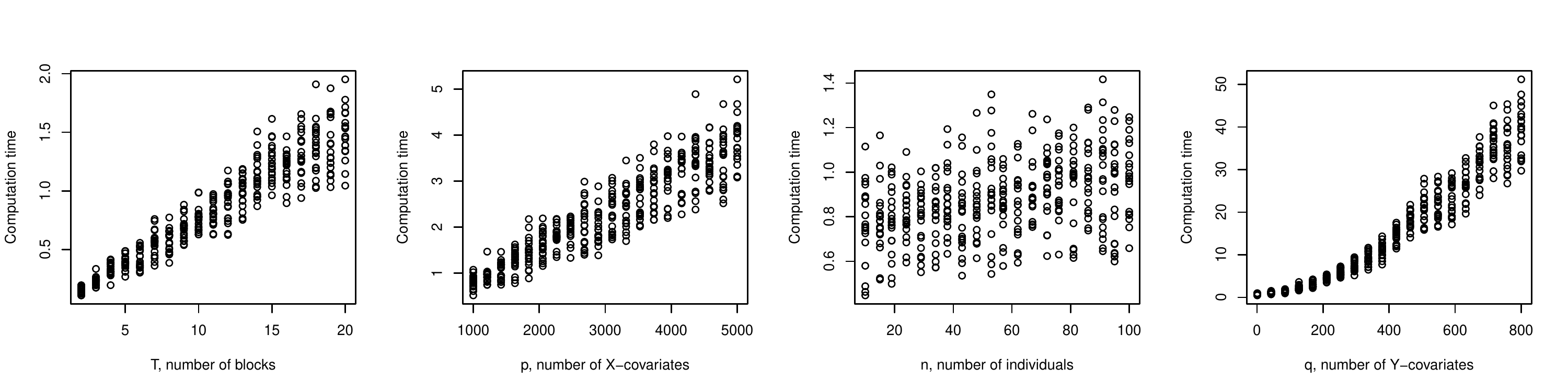}
	\caption{Computation times in both of the discussed regimes against $T$, $p$, $n$ and $q$. First row describes cases when $n>>\max(p,q)$ with $(n,T,q,p)=(1000,10,10,100)$ if not varying. Second row describes cases when $p>>n$ with $(n,T,q,p)=(20,10,10,1000)$ if not varying.\label{fig:time_complexity}}
\end{figure}

%=======================================================
\subsection{Adaptation to the Classification Case}
\label{sec:adapt_class}
%=======================================================

If the response variables are categorical, the regression model so far discussed for numerical response variables must be adapted to the problem of classification. \textbf{PLS} methods were showed to be efficient through a \textbf{Logistic Discrimination} model. The underlying methodology is based on a classic \textbf{PLS} on the dummy-standardized variables of the response matrix to build a convenient number of components. Those components are used to estimate a \textbf{Logistic Regression} model discriminating original classes, see for example~\citet{sjostrom1986pls} for \textbf{PLS-Discriminant Analysis} and~\citet{hosmer1989applied} for the \textbf{Logistic Regression}. \citet{sabatier2003two} showed that this method is biased. \citet{clemmensen2011sparse} recommend using a standard \textbf{Linear Discriminant Analysis} (\textbf{LDA}) in that context. Because slightly better results were found with the \textbf{LDA} method, it has been conserved in the \textbf{mdd-sPLS} implementation. The \textbf{mdd-sPLS} components are thus built on the standardized complete disjunctive coding of the class membership of the individuals for $R$ components, which are the covariates of the \textbf{LDA} model and where class memberships are responses.
Appendix~\ref{app:Classif} provides an illustration of the \textbf{mdd-sPLS} approach in this classification framework using the usual data set \textbf{Penicillium YES data set}.

%=======================================================
%=======================================================
%=======================================================
\section{Simulation Study}\label{S:4}
%=======================================================
%=======================================================
%=======================================================

Previous sections have presented \textbf{mdd-sPLS}, a supervised multi-block method which allows regularization and variable selection. A regression model has also been defined. An algorithm to deal with missing samples, which means that a row of a block might be missing in the \textit{train} or in the \textit{test} data set, has been proposed. Simulations have been performed to explore the performances of the \textbf{mdd-sPLS} method in term of prediction errors according to:
\begin{itemize}
    \item increasing proportion of missing values,
    \item decreasing number of individuals,
    \item decreasing correlations between the blocks ${\bf X}_t, t=1,\dots,T$.
\end{itemize}
Computation times and convergence successes have also been studied.
The data generating process is described in Section~\ref{sub:str_simu}. Competing methods (for imputation and prediction) are presented in Section~\ref{compet}. Notice that the \textbf{mdd-sPLS} (with \textbf{Koh-Lanta} algorithm) method is the only method that handles the missing data and makes the prediction of the response at the same time, denoted as ``all-in-one method''. The other methods are called ``two-steps methods'' because they first handle missing values and then build a model to predict response.
The methodological process used for the ``two-steps methods'' is detailed and takes into account the choices and comments of the corresponding authors. All numerical results of the simulation study are presented in Section~\ref{result-sim}.

%=======================================================
\subsection{Data Generating Process}\label{sub:str_simu}
%=======================================================

Let us consider a multi-block context with $T=10$ blocks of covariates.
The  $\textbf{X}_t$'s are generated with 
 \textbf{inter-block} relationships (i.e. links between the different blocks) and  \textbf{intra-block} relationships (i.e. links between the different variables within a block).
 \begin{itemize}
     \item Each block ${\bf X}_t$ is composed of $D=4$ groups of covariates. The number of covariates in each group is equal to 40.
     \item For $d\in\llbracket1,D-1\rrbracket$, the group $d$ of block $\textbf{X}_t$ is linearly linked to the corresponding group $d$ of the other blocks; the \textbf{inter-block} linear correlation parameter is denoted $\rho_t$.
     \item For a given group $d\in\llbracket1,D-1\rrbracket$ and a given block $t$, the variables are linearly linked through the \textbf{intra-block} linear correlation parameter, denoted $\rho_d$.
     \item For a given block $t$, the $D^{th}$ group  is not linked either to the other groups of the block $t$ nor to the groups of the other blocks.
 \end{itemize}
To resume, the $\mathbf{X}_t$'s data sets can be represented as
\[  
\mathbf{X}=\big(\underbrace{
  \smash[b]{\overbrace{\clubsuit \cdots \clubsuit}^{\text{Group 1}}}
  \boldsymbol{\cdots}
  \smash[b]{\overbrace{\spadesuit \cdots \spadesuit}^{\text{Group $D$}}}
}_{\mathbf{X}_1} 
\underbrace{\smash[b]{\bullet\cdots\bullet}}_{\mathbf{X}_{\text{2 to T-1}}}
\underbrace{
  \smash[b]{\overbrace{\varclubsuit \cdots \varclubsuit}^{\text{Group 1}}}
  \cdots
  \smash[b]{\overbrace{\varspadesuit \cdots \varspadesuit}^{\text{Group $D$}}}
}_{\mathbf{X}_T}\big),
\]
where the card game symbols represent variables with different links.
To generate the $n$ observations of all the covariates of the $T$ blocks, the multivariate normal distribution has been used with a null vector as mean and a covariance matrix of size $T\times D\times 40 = 1600$respecting the conditions mentioned above. 

The response matrix $\mathbf{Y}$ must be designed with links to the covariates of the blocks $\mathbf{X}_1,\dots,\mathbf{X}_T$.
\begin{itemize}
    \item  The block ${\bf Y}$ is composed of $q=1$ variable since the $\mathcal{L}$asso method, the main prediction benchmark method, works for univariate response.
    \item ${\bf Y}$ is linked to $5$ of the  $10$ blocks. In each of those blocks, only a number of variables denoted $\theta$, randomly chosen in $\Theta=\{4,8,12,16,20,24,28,32,36,40\}$, is indeed taken into account. The $40-\theta$ other variables of that group are filled with Gaussian noises.
    
    This process allows to simulate strongly correlated data sets if $\theta$ is high for each of the $5$ blocks. Inversely, if the different $\theta$'s are small, the first left singular vector tends to describe less common information.
    %\item Parameter $\rho_t$ drives the link between the blocks and parameter $\rho_d$ drives the link inside each block.
    \item The response variable is then obtained as the first left singular vector of the \textbf{SVD} applied to the matrix containing only the informative covariates and is then naturally linearly linked to those informative covariates.
    \end{itemize}

 %   the number of  More precisely, this response variable is linked to the first five blocks ${\bf X}_t$ only through a reduced number of covariates of the corresponding first groups. For each of these first groups, the number of the ``relevant/informative'' covariates is randomly chosen in $\Theta=\{4,8,12,16,20,24,28,32,36,40\}$.
  %  The first response variable is then obtained as the first left singular vector of the \textbf{SVD} applied to the matrix containing only the informative covariates and is then naturally linearly linked to those informative covariates.
% Here is a representation of $\bf Y$,
%\[\mathbf{Y}=\big(\smash[b]{\overbrace{\heartsuit \cdots \heartsuit}^{\text{Group 1}}}\boldsymbol{\cdots}  \smash[b]{\overbrace{\vardiamondsuit \cdots \vardiamondsuit}^{\text{Group $m$}}}\big).\]
For generation of the missing values, a random process is used to delete some rows (observations) of some blocks ${\bf X}_t$, corresponding to a proportion of missing values fixed a priori. A constraint has been taken into account: there must be at least one block of non missing values for each individual.

Figure~\ref{fig:corComp} shows the correlation matrix of all the covariates of the  $T=10$ blocks bound together and the right column corresponds to the correlations with the response matrix $\bf Y$, using a data set simulated with  $n=100$,  $\rho_t=0.9$, $\rho_d=0.9$ and $30\%$ of missing values. Note that the calculation of the correlations is based only on the  non missing values.
\begin{figure}[p]
	\centering
	\includegraphics[width = 3in]{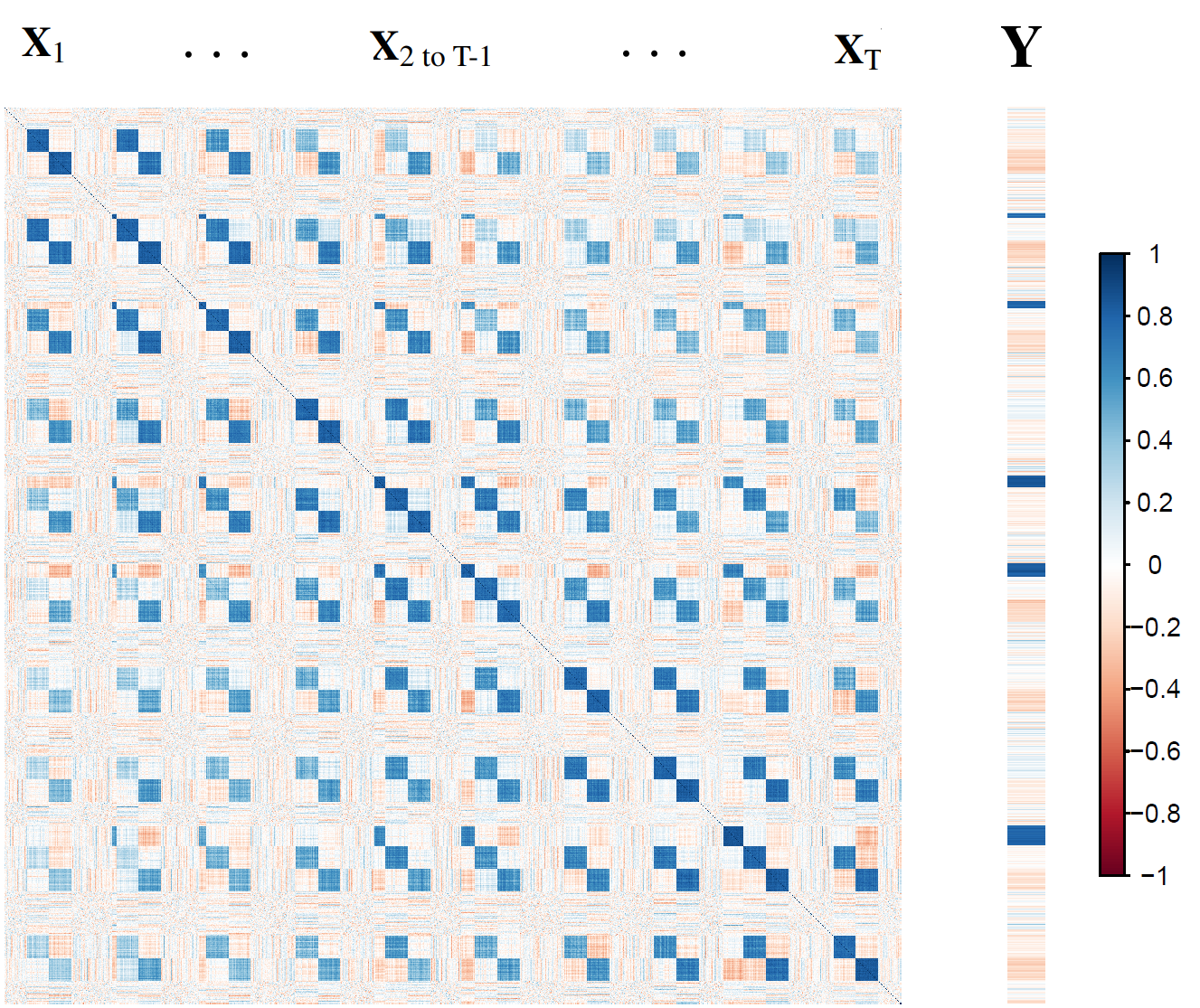}
	\caption{Example of a empirical correlation matrix calculated on a data set simulated with $n=100$, $T=10$, $\rho_t=0.9$, $\rho_d=0.9$ and $30\%$ of missing values. Correlation scale: from red (-1) to blue (+1)
	\label{fig:corComp}}
\end{figure}

In the simulation study of Section~\ref{result-sim}, various values for the parameters $n$, $\rho_t$, $\rho_d$ and the proportion of missing values were considered.

%=======================================================
\subsection{Competing Methods}\label{compet}
%=======================================================

\textbf{mdd-sPLS} (with {\bf Koh-Lanta} algorithm) was compared to several competing methods. 
Since existing approaches are mostly ``two-steps methods'', it is therefore necessary to choose \textbf{imputation} method and \textbf{prediction} method.

The \texttt{imputation} methods selected for this simulation study were: 
\begin{itemize}[topsep=0pt,itemsep=-1ex]
\item \textbf{mean}: this is the simplest way to impute. For a given covariate in a given block, the missing values are estimated with the mean of the non missing values of this covariate.
\item \textbf{softImpute}, see~\citet{hastie2015matrix}: the imputation is based on the use of fast-ALS dedicated to estimate missing values in a single-block context. Hence, the block structure of the data set is ignored.
\item \textbf{imputeMFA}, from the package \texttt{missMDA}, see~\citet{husson2013handling}: the underlying algorithm takes into account the block structure.
\item \textbf{nipals}, suggested by \textbf{mixOmics} authors among other sources: this imputation step is followed by a \textbf{classic-sPLS}. The protocol suggested by the authors is detailed on~\url{http://mixomics.org/methods/missing-values/}.
\end{itemize}
Note that the last three imputation methods look for a \textbf{SVD}-modified component-wise structure of the data,  as in the proposed \textbf{mdd-sPLS} (with {\bf Koh-Lanta} algorithm). However, those \textbf{imputation} baseline methods are not supervised and so the number of axes must be tuned by the user.

For the \textbf{prediction} step, three methods were selected:
\begin{itemize}[topsep=0pt,itemsep=-1ex]
\item \textbf{mdd-sPLS}: the proposed method applied to the imputed data set (i.e. without missing values).
\item $\mathcal{L}$\textbf{asso}: the well-known $\mathcal{L}_1$-penalized prediction method which is easily usable if the response variable is univariate.
\item \textbf{classic-sPLS}: as previously mentioned, this approach is used  once the \textbf{nipals} algorithm has been used to impute missing values.
\end{itemize}
From these different methods of imputation and prediction, we will compare the numerical behavior of the following 8 methodologies:
\begin{enumerate}[topsep=0pt,itemsep=-1ex]
\item \textbf{mdd-sPLS} with {\bf Koh-Lanta} algorithm,
\item \textbf{nipals} + \textbf{classic-sPLS},
\item \textbf{imputeMFA} + \textbf{mdd-sPLS},
\item \textbf{imputeMFA} + $\mathcal{L}$\textbf{asso},
\item \textbf{softImpute} + \textbf{mdd-sPLS},
\item \textbf{softImpute} + $\mathcal{L}$\textbf{asso},
\item \textbf{mean} + \textbf{mdd-sPLS},
\item \textbf{mean} + $\mathcal{L}$\textbf{asso}.
\end{enumerate}

\noindent
In order to properly evaluate the performance of the different methodologies, a learning (train) sample and a test sample should be considered.
The given data set is then splitted into a \textbf{train} data set, ($\mathbf{X^{(train)}}$,$\mathbf{Y^{(train)}}$), and  a \textbf{test} data set, ($\mathbf{X^{(test)}}$,$\mathbf{Y^{(test)}}$). Let us discuss the strategy of imputation of the \textbf{train} data set and the \textbf{test} data set for the considered methodologies.

\begin{itemize}[topsep=0pt,itemsep=-1ex]
\item The \textbf{Koh-Lanta} algorithm allows to deal with missing values in the \textbf{train} and \textbf{test} data sets.

\item \textbf{mean}: the missing values have been estimated as the mean of the $\mathbf{X^{(train)}}$ variables. They are used to impute $\mathbf{X^{(train)}}$ and $\mathbf{X^{(test)}}$ data sets.
    \item \textbf{imputeMFA}: the underlying method is used to impute $\mathbf{X^{(train)}}$, but cannot be applied to $\mathbf{X^{(test)}}$ imputation. Thus the missing values of $\mathbf{X^{(test)}}$ was imputed to the means, estimated from $\mathbf{X^{(train)}}$.
    
\item \textbf{softImpute}: even if the authors consider a mono-block problem, it is possible to build a prediction model of imputation (using the  \texttt{softImpute} function), which is used to estimate $\mathbf{X^{(test)}}$ from the imputed $\mathbf{X^{(train)}}$ (using the  \texttt{complete} function). 
Note that, apart from the proposed ``all-in-one'' method (\textbf{mdd-sPLS} with {\bf Koh-Lant} algorithm), \textbf{softImpute} is the only method reusing the eigen-spaces constructed on $\mathbf{X^{(train)}}$ to impute $\mathbf{X^{(train)}}$ and $\mathbf{X^{(test)}}$.

\item \textbf{nipals}: $\mathbf{X^{(train)}}$ is imputed with the \texttt{nipals} function from \textbf{mixOmics} package. The number of components has been arbitrarily fixed to \texttt{ncomp}=3. As for \textbf{missMDA}, there is no particular reason to reuse the eigen-spaces built to impute the $\mathbf{X^{(train)}}$'s missing values to predict the $\mathbf{X^{(test)}}$'s missing values. Thus the  $\mathbf{X^{(test)}}$'s missing values are imputed to the mean of the $\mathbf{X^{(train)}}$ data set. Note that the estimation of the \textbf{classic-sPLS} model is based on the imputed $\mathbf{X^{(train)}}$ and $\mathbf{Y^{(train)}}$.
\end{itemize}

%=======================================================
\subsection{Simulation Results}\label{result-sim}
%=======================================================

The simulation study splits into five parts in order to evaluate:
\begin{itemize}[topsep=0pt,itemsep=-1ex]
\item the effect of the proportion of missing values,
\item the effect of the number of individuals,
\item the effect of the inter-block correlation structure,
\item the effect of the intra-block correlation structure,
\item the computation time and the convergence (of the underlying algorithm) efficiency.
\end{itemize}
The error considered is the leave-one-out cross-validation \textit{root mean square error}, denoted \textbf{RMSEP}. For the \textbf{mdd-sPLS}, eight different values were tested for $\lambda$ and the one with the lowest \textbf{RMSEP} error is selected. For the $\mathcal{L}$\textbf{asso}, the \texttt{glmnet} package is used to select the \texttt{lambda.1se} regularization coefficient as proposed by the authors when the low sample size is small. 
For \textbf{nipals}, \textbf{softImpute} and \textbf{imputeMFA}, the number of components is fixed to $3$.
%because data sets are simulated on $3$ components and the interesting one is the deepest one. 
Moreover, eight different values were tested for the \textbf{softImpute} parameter and the most accurate was selected. 

The various scenarios considered are inspired by real case problems. For each of the simulation settings, $20$ data sets were generated from the data generating process describes in Section~\ref{sub:str_simu}.
Then the eight methodologies presented in section~\ref{compet} were applied to each of the data sets and the associated \textbf{RMSEP} were calculated.

%=======================================================
\subsubsection{Effect of the Proportion of Missing Values}
%=======================================================

For the eight methodologies considered, Table~\ref{tab:tab_9_9} provides the \textbf{RMSEP} errors for eight different proportions of missing values from $2\%$, to $60\%$ when the data generating process is based on $\rho_{d}=0.9$, $\rho_t=0.9$ (i.e. strong inter/intra-blocks correlations) with $n=100$. Figure~\ref{fig:visu_error_simu_prop} shows the performances of the methods.
\begin{table}[ht]
\resizebox{\textwidth}{!}{
\centering
\begin{tabular}{|c|c|c|c|c|c|c|c|}
  \cline{2-8}
 \multicolumn{1}{c|}{}&\multicolumn{2}{|c|}{\backslashbox{Method}{Prop. of NA}} & \multirow{2}{*}{2\%} & \multirow{2}{*}{5\%} & \multirow{2}{*}{8\%} & \multirow{2}{*}{10\%}& \multirow{2}{*}{15\%} \\ 
\cline{2-3}
 \multicolumn{1}{c|}{}&Imputation&Prediction&&&&& \\
  \hline
  \hline
\multirow{8}{*}{\rotatebox{90}{{\small RMSEP}}}&\multicolumn{2}{|c|}{mdd-sPLS (\textbf{with} Koh-Lanta)} & 0.243 $\pm$  0.0535 & 0.259 $\pm$  0.0560 & \textbf{0.255 $\pm$  0.0588} & \textbf{0.260 $\pm$  0.0582} & \textbf{0.282 $\pm$ 0.0610}\\ 
  \cline{2-8}
&  nipals & classic-sPLS & 0.251 $\pm$  0.0514 & 0.282 $\pm$  0.0545 & 0.296 $\pm$  0.0523 & 0.313 $\pm$  0.0510 & 0.366 $\pm$  0.0540\\ 
  \cline{2-8}
&  imputeMFA & mdd-sPLS & 0.253 $\pm$  0.0542 & 0.283 $\pm$  0.0555 & 0.291 $\pm$  0.0567 & 0.307 $\pm$  0.0549 & 0.347 $\pm$  0.0553\\ 
  \cline{2-8}
&  imputeMFA & Lasso & 0.218 $\pm$  0.0460 & 0.259 $\pm$  0.0495 & 0.269 $\pm$  0.0425 & 0.292 $\pm$  0.0442 & 0.335 $\pm$  0.0484\\ 
  \cline{2-8}
&  softImpute & mdd-sPLS & 0.251 $\pm$  0.0531 & 0.281 $\pm$  0.0547 & 0.290 $\pm$  0.0550 & 0.304 $\pm$  0.0531 & 0.347 $\pm$  0.0554\\ 
  \cline{2-8}
&  softImpute & Lasso & \textbf{0.215 $\pm$  0.0445} & \textbf{0.255 $\pm$  0.0471} & 0.267 $\pm$  0.0403 & 0.289 $\pm$  0.0431 & 0.332 $\pm$  0.0465\\ 
  \cline{2-8}
&  Mean & mdd-sPLS & 0.253 $\pm$  0.0541 & 0.284 $\pm$  0.0553 & 0.292 $\pm$  0.0566 & 0.308 $\pm$  0.0551 & 0.348 $\pm$  0.0557\\ 
  \cline{2-7}
&  Mean & Lasso & 0.219 $\pm$  0.0455 & 0.260 $\pm$  0.0495 & 0.271 $\pm$  0.0430 & 0.293 $\pm$  0.0437 & 0.337 $\pm$  0.0477\\ 
   \hline
  \multicolumn{6}{c}{} \\
    \cline{2-8}
 \multicolumn{1}{c|}{}&\multicolumn{2}{|c|}{\backslashbox{Method}{Prop. of NA}}& \multirow{2}{*}{20\%} & \multirow{2}{*}{30\%} & \multirow{2}{*}{40\%} & \multirow{2}{*}{50\%} & \multirow{2}{*}{60\%}\\
\cline{2-3}
 \multicolumn{1}{c|}{}&Imputation&Prediction&&&&& \\
  \hline
  \hline
\multirow{8}{*}{\rotatebox{90}{{\small RMSEP}}}&\multicolumn{2}{|c|}{mdd-sPLS (\textbf{with} Koh-Lanta)} &\textbf{ 0.300 $\pm$  0.0625 }&\textbf{ 0.315 $\pm$  0.0488 }&\textbf{ 0.362 $\pm$  0.0598} &\textbf{ 0.413 $\pm$  0.0555} & \textbf{0.519 $\pm$  0.0639} \\ 
  \cline{2-8}
&  nipals & classic-sPLS & 0.407 $\pm$  0.0475 & 0.475 $\pm$  0.0439 & 0.561 $\pm$  0.0474 & 0.639 $\pm$  0.0412 & 0.747 $\pm$  0.0418\\ 
  \cline{2-8}
&  imputeMFA & mdd-sPLS & 0.380 $\pm$  0.0516 & 0.426 $\pm$  0.0480 & 0.488 $\pm$  0.0536 & 0.544 $\pm$  0.0473 & 0.634 $\pm$  0.0480\\ 
  \cline{2-8}
 & imputeMFA & Lasso & 0.379 $\pm$  0.0525 & 0.437 $\pm$  0.0578 & 0.516 $\pm$  0.0615 & 0.584 $\pm$  0.0618 & 0.688 $\pm$  0.0688 \\ 
  \cline{2-8}
&  softImpute & mdd-sPLS & 0.379 $\pm$  0.0519 & 0.425 $\pm$  0.0475 & 0.487 $\pm$  0.0539 & 0.541 $\pm$  0.0469 & 0.624 $\pm$  0.0471 \\ 
  \cline{2-8}
 & softImpute & Lasso & 0.378 $\pm$  0.0518 & 0.437 $\pm$  0.0556 & 0.514 $\pm$  0.0588 & 0.582 $\pm$  0.0576 & 0.676 $\pm$  0.0638 \\ 
  \cline{2-8}
 & Mean & mdd-sPLS & 0.381 $\pm$  0.0523 & 0.426 $\pm$  0.0479 & 0.489 $\pm$  0.0539 & 0.544 $\pm$  0.0465 & 0.628 $\pm$  0.0473 \\ 
  \cline{2-8}
 & Mean & Lasso & 0.382 $\pm$  0.0524 & 0.441 $\pm$  0.0548 & 0.517 $\pm$  0.0596 & 0.584 $\pm$  0.0579 & 0.679 $\pm$  0.0643\\  
  \hline
\end{tabular}
}
\caption{Effect of proportion of missing values on the {RMSEP}. $100$ simulation results for ($\rho_{d}=0.9,\rho_t=0.9$) and $n=100$ individuals. The main statistics are given for each method with {(mean $\pm$ std)}. Bolded results correspond to best results for a given proportion of missing values.\label{tab:tab_9_9}}
\end{table}

\begin{figure}[p]
	\centering
	\includegraphics[width = \textwidth]{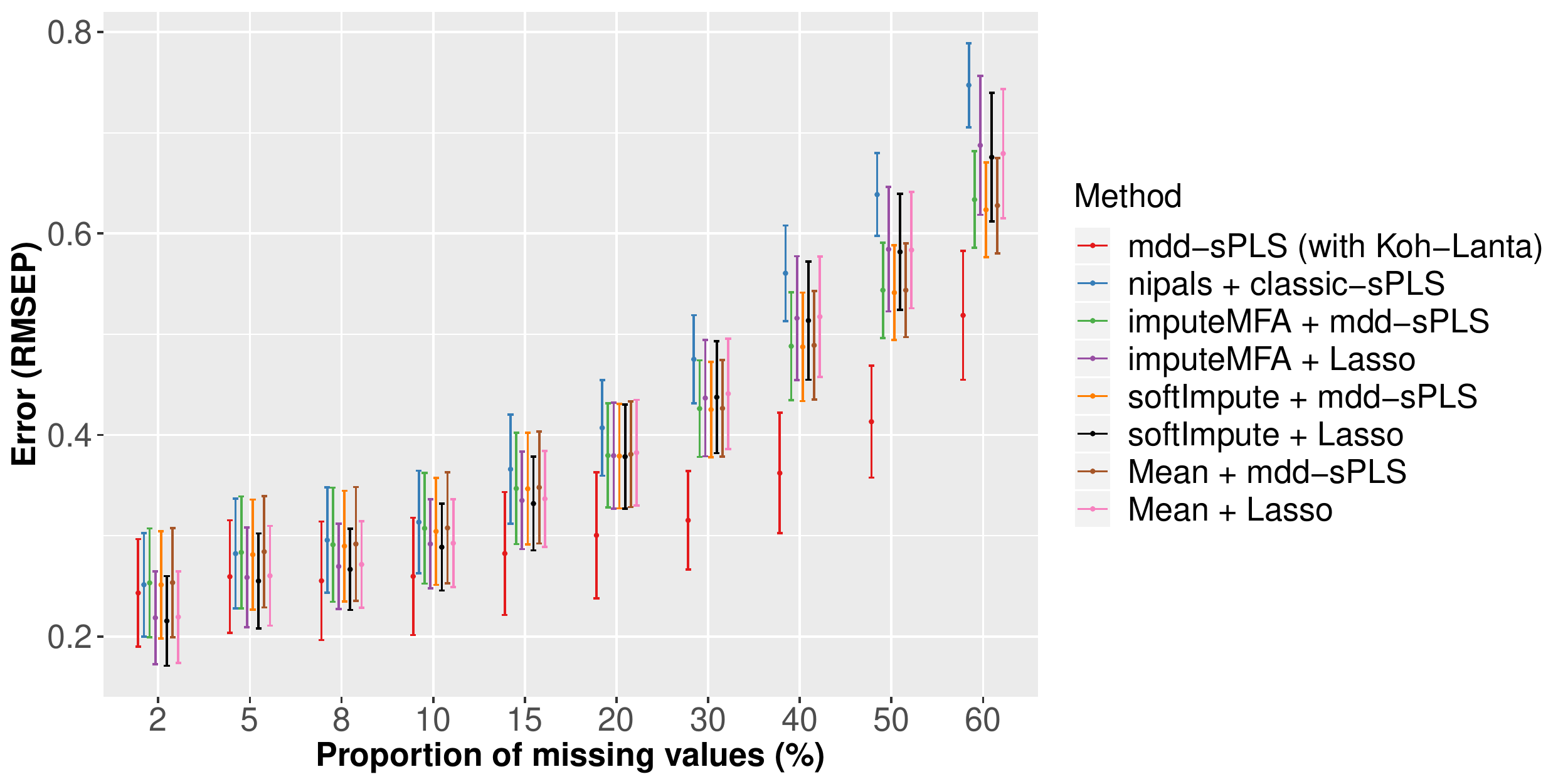}
	\caption{Effect of proportion of missing values on the {RMSEP}, barplot of Table~\ref{tab:tab_9_9} results.
    \label{fig:visu_error_simu_prop}}
\end{figure}

For small proportions of missing values, all the methodologies provide very similar results with a very slight advantage to the \textbf{$\mathcal{L}$asso} based regression methods.
When the proportion of missing values increases, two methods behaved differently compared to the others (Figure~\ref{fig:visu_error_simu_prop}). From 20\% of missing values onward,  The \textbf{mdd-sPLS} with Koh-Lanta gave clearly better results than the others methods. When the proportion of missing values is at 50\% or more, \textbf{nipals + classic-sPLS} methods provides poorer results compared to all the other approaches.
%Because the results were quite stable over the simulated data sets, 20 simulated data sets per scenario have been generated in the following experiments.
% \begin{landscape}
% \input{figures/simu_all.tex}
% \end{landscape}
%=======================================================
\subsubsection{Effect of the Number \textit{n} of Individuals}
%=======================================================

For the eight methodologies considered, Table~\ref{tab:high_dim} provides the \textbf{RMSEP} errors for  different numbers $n$ of individuals ($n=100, 50, 20$) when the data generating process is based on $\rho_{d}=0.9$, $\rho_t=0.9$ (i.e. strong inter/intra-blocks correlations) with a proportion of missing values equal to 30\%.
Table~\ref{tab:high_dim} provides \textbf{RMSEP} error results for these  different numbers $n$ of individuals. 

\begin{table}[ht]
\resizebox{\textwidth}{!}{
\centering
\begin{tabular}{|c|c|c|c|c|c|}
  \cline{2-6}
 \multicolumn{1}{c|}{}&\multicolumn{2}{|c|}{\backslashbox{Method}{$\#$ individuals}} & \multirow{2}{*}{100} & \multirow{2}{*}{50} & \multirow{2}{*}{20} \\ 
\cline{2-3}
\multicolumn{1}{c|}{}&Imputation&Prediction&&&\\
  \hline
  \hline
\multirow{8}{*}{\rotatebox{90}{{\small RMSEP}}}&\multicolumn{2}{|c|}{mdd-sPLS (\textbf{with} Koh-Lanta)} & \textbf{0.308 $\pm$ 0.0528} & \textbf{0.335 $\pm$ 0.0643} & \textbf{0.445 $\pm$ 0.139} \\  \cline{2-6}
 & nipals & classic-sPLS & 0.491 $\pm$ 0.0504 & 0.497 $\pm$ 0.0665 & 0.566 $\pm$ 0.124 \\  \cline{2-6}
 & imputeMFA & mdd-sPLS & 0.420 $\pm$ 0.0530 & 0.440 $\pm$ 0.0648 & 0.521 $\pm$ 0.119 \\  \cline{2-6}
 & imputeMFA & Lasso & 0.438 $\pm$ 0.0599 & 0.506 $\pm$ 0.0852 & 0.714 $\pm$ 0.169 \\  \cline{2-6}
 & softImpute & mdd-sPLS & 0.420 $\pm$ 0.0525 & 0.436 $\pm$ 0.0644 & 0.505 $\pm$ 0.119 \\  \cline{2-6}
 & softImpute & Lasso & 0.436 $\pm$ 0.0581 & 0.496 $\pm$ 0.0824 & 0.656 $\pm$ 0.162 \\  \cline{2-6}
 & Mean & mdd-sPLS & 0.421 $\pm$ 0.0527 & 0.440 $\pm$ 0.0666 & 0.528 $\pm$ 0.125 \\  \cline{2-6}
 & Mean & Lasso & 0.440 $\pm$ 0.0593 & 0.504 $\pm$ 0.0866 & 0.721 $\pm$ 0.188 \\ 
   \hline
\end{tabular}
}
\caption{Effect of sample size  on the {RMSEP}. $100$ simulation results for ($\rho_{d}=0.9,\rho_t=0.9$) and $30\%$ of missing values. The main statistics are given for each method with {(mean $\pm$ std)}. Bolded results correspond to best results for a given number of individuals per sample.\label{tab:high_dim}}
\end{table}

The \textbf{mdd-sPLS} with {\bf Koh-Lanta}  leads to the smallest \textbf{RMSEP} error for the three different sample size $n$. The other \textbf{mdd-sPLS}-based methods have better behavior than the \textbf{$\mathcal{L}$asso}-based methods regardless of the imputation method chosen. Finally, the ``two-steps method'' \textbf{softImpute + mdd-sPLS} has the second best performance.

%=======================================================
\subsubsection{Effect of the Inter-Block Correlations}
%=======================================================

As presented before, the response variable in $\bf Y$ was correlated to some covariates of some $X_t$'s blocks with the intensity $\rho_t$. Moreover the correlation between different $\bf X_t$'s blocks was also equal to $\rho_t$. 
The performances of the methods were evaluated according to the parameter  $\rho_t$ with $\rho_t\in \{0.9, 0.7, 0.5, 0.3\}$. Simulation results are provided in Table~\ref{tab:tab_no_inter} and are plotted in Figure~\ref{fig:vary_rho_t}.
\begin{table}[ht]
\resizebox{\textwidth}{!}{
\centering
\begin{tabular}{|c|c|c|c|c|c|c|}
  \cline{2-7}
\multicolumn{1}{c}{}&\multicolumn{2}{|c|}{\backslashbox{Method}{$\rho_{t}$}} & \multirow{2}{*}{0.9} & \multirow{2}{*}{0.7} & \multirow{2}{*}{0.5} &\multirow{2}{*}{0.3} \\ 
\cline{2-3}
\multicolumn{1}{c|}{}&Imputation&Prediction&&&&\\
  \hline
  \hline
\multirow{8}{*}{\rotatebox{90}{{\small RMSEP}}}&\multicolumn{2}{|c|}{mdd-sPLS {\bf with} Koh-Lanta} & \textbf{0.312 $\pm$ 0.0516} & \textbf{0.528 $\pm$ 0.0801} & \textbf{0.662 $\pm$ 0.0948} & \textbf{0.752 $\pm$ 0.0896} \\    \cline{2-7}
 & nipals & classic-sPLS & 0.470 $\pm$ 0.0451 & 0.602 $\pm$ 0.0662 & 0.699 $\pm$ 0.0779 & 0.766 $\pm$ 0.0751\\   \cline{2-7}
&  imputeMFA & mdd-sPLS & 0.421 $\pm$ 0.0487 & 0.572 $\pm$ 0.0710 & 0.678 $\pm$ 0.0848 & 0.756 $\pm$ 0.0822\\   \cline{2-7}
&  imputeMFA & Lasso & 0.438 $\pm$ 0.0537 & 0.598 $\pm$ 0.0740 & 0.724 $\pm$ 0.103 & 0.816 $\pm$ 0.0975\\   \cline{2-7}
 & softImpute & mdd-sPLS & 0.420 $\pm$ 0.0492 & 0.570 $\pm$ 0.0706 & 0.677 $\pm$ 0.0853 & 0.754 $\pm$ 0.0824\\    \cline{2-7}
 & softImpute & Lasso & 0.433 $\pm$ 0.0533 & 0.591 $\pm$ 0.0723 & 0.718 $\pm$ 0.102 & 0.813 $\pm$ 0.100\\   \cline{2-7}
&  Mean & mdd-sPLS & 0.421 $\pm$ 0.0492 & 0.572 $\pm$ 0.0713 & 0.679 $\pm$ 0.0858 & 0.756 $\pm$ 0.0827\\  \cline{2-7}
 & Mean & Lasso & 0.436 $\pm$ 0.0538 & 0.598 $\pm$ 0.0743 & 0.724 $\pm$ 0.105 & 0.818 $\pm$ 0.0996\\ 
   \hline
\end{tabular}
}
\caption{Effect of inter-block correlation on the {RMSEP}. $100$ simulation results for $\rho_{d}=0.9$, $\rho_t\in \{0.3,0.5,0.7,0.9\}$, $n=100$ individuals  and $30\%$ of missing values. The main statistics are given for each methodology with {(mean $\pm$ std)}. Bolded results correspond to the best ones for a given value of~$\rho_t$.
\label{tab:tab_no_inter}}
\end{table}

For any $\rho_t$, the \textbf{mdd-sPLS with Koh-Lanta} method is the most accurate one according to the RMSEP. For low values of  $\rho_t$, the \textbf{mdd-sPLS}-based methods showed better behaviors than the \textbf{$\mathcal{L}$asso}-based ones.
% \input{figures/simu_100_9_2.tex}

%=======================================================
\subsubsection{Effect of the Intra-Block Correlations}
%=======================================================

The following simulations evaluate the impact of a varying intra-correlation on the overall error {\bf RMSEP}. The other parameters have been fixed to $\rho_t=0.9$ and $30\%$ of missing values. Simulations results are provided in Table~\ref{tab:tab_move_intra} and are plotted in Figure~\ref{fig:vary_rho_d}.
\begin{table}[ht]
\resizebox{\textwidth}{!}{
\centering
\begin{tabular}{|c|c|c|c|c|c|c|}
  \cline{2-7}
\multicolumn{1}{c}{}&\multicolumn{2}{|c|}{\backslashbox{Method}{$\rho_{d}$}} & \multirow{2}{*}{0.3} & \multirow{2}{*}{0.5} & \multirow{2}{*}{0.7} &\multirow{2}{*}{0.9} \\ 
\cline{2-3}
\multicolumn{1}{c|}{}&Imputation&Prediction&&&&\\
  \hline
  \hline
\multirow{8}{*}{\rotatebox{90}{{\small RMSEP}}}&\multicolumn{2}{|c|}{mdd-sPLS {\bf with} Koh-Lanta} &  \textbf{0.399 $\pm$ 0.054} & \textbf{0.346 $\pm$ 0.0563} & \textbf{0.317 $\pm$ 0.0482} & \textbf{0.312 $\pm$ 0.0516 }\\    \cline{2-7}
 & nipals & classic-sPLS &  0.538 $\pm$ 0.0524 & 0.499 $\pm$ 0.0492 & 0.485 $\pm$ 0.0453 & 0.47 $\pm$ 0.0451  \\   \cline{2-7}
&  imputeMFA & mdd-sPLS & 0.477 $\pm$ 0.0534 & 0.443 $\pm$ 0.0533 & 0.433 $\pm$ 0.0468 & 0.421 $\pm$ 0.0487 \\   \cline{2-7}
&  imputeMFA & Lasso & 0.565 $\pm$ 0.0691 & 0.504 $\pm$ 0.0608 & 0.472 $\pm$ 0.0554 & 0.438 $\pm$ 0.0537 \\  \cline{2-7}
 & softImpute & mdd-sPLS & 0.476 $\pm$ 0.0524 & 0.443 $\pm$ 0.0541 & 0.433 $\pm$ 0.0467 & 0.42 $\pm$ 0.0492  \\  \cline{2-7}
 & softImpute & Lasso & 0.565 $\pm$ 0.0661 & 0.503 $\pm$ 0.0591 & 0.472 $\pm$ 0.0531 & 0.433 $\pm$ 0.0533 \\  \cline{2-7}
&  Mean & mdd-sPLS &  0.476 $\pm$ 0.0527 & 0.443 $\pm$ 0.0541 & 0.434 $\pm$ 0.0469 & 0.421 $\pm$ 0.0492 \\  \cline{2-7}
 & Mean & Lasso & 0.572 $\pm$ 0.0697 & 0.508 $\pm$ 0.0604 & 0.476 $\pm$ 0.0533 & 0.436 $\pm$ 0.0538 \\
   \hline
\end{tabular}
}
\caption{Effect of intra-block correlation on the {RMSEP}. 100 simulation results for $\rho_{t}=0.9$, $\rho_d\in \{0.3,0.5,0.7,0.9\}$, $n=100$ individuals  and $30\%$ of missing values. The main statistics are given for each methodology with {(mean $\pm$ std)}. Bolded results correspond to the best ones for a given value of~$\rho_d$.
\label{tab:tab_move_intra}}
\end{table}

Among the four simulated settings, the case $\rho_d=0.9$ corresponds to an already discussed one, see Figure~\ref{fig:visu_error_simu_prop}. It is interesting to see the stability of the results for those new simulations. The data set with $\rho_d=0.7$ shows that all the method are equivalent ($\approx 0.31$) except \textbf{mdd-sPLS (with Koh-Lanta)} for which the error is lower ($\approx 0.27$). For higher intra-block correlations, among the baseline imputation methods, the \textbf{Lasso} based prediction methods are more efficient than \textbf{mdd-sPLS} ones but \textbf{mdd-sPLS (with Koh-Lanta)} show lowest errors. For lower intra-block correlations, $\rho\in\{0.3,0.5\}$,  among the baseline imputation methods, the \textbf{Lasso} based prediction methods are less efficient than \textbf{mdd-sPLS} but \textbf{mdd-sPLS (with Koh-Lanta)} still lead to better results. It is also interesting to notice that the $\textbf{nipals}$ approach has equivalent results than \textbf{mdd-sPLS} prediction based methods with the baseline imputation methods in all features.
% \input{figures/simu_100_9_2.tex}

%=======================================================
\subsubsection{Computation Time and Convergence Quality}
%=======================================================

Regarding the convergence of the various methods, the \textbf{mean} imputation method is not concerned by this numerical aspect since it is based on only one step of imputation. 
For the other methodologies, once imputation stages have no further effects on subspace estimation, we considered that the imputation process has converged. The convergence criterion was defined as the stabilization of estimations in the last estimated subspace with a threshold value set to $10^{-9}$ and the maximum number of iterations to $100$. More precisely, let us specify for each method the concerned matrix:
\begin{itemize}
\item \textbf{mdd-sPLS}: the matrix $\mathbf{T}_{super}\mathbf{V}_{ort}$, which is defined in the algorithm of the method.
\item \textbf{softImpute}: the matrix $U$ of the left-singular vectors~\citep[][Algorithm 2.1]{hastie2015matrix},
\item \textbf{imputeMFA}: the matrix $\mathbf{U}$ of the left-singular vectors~\citep[][Chapter 3.1]{josse2016missmda},
\item \textbf{nipals}: the matrix of the components $\mathbf{t_k}$~\citep[][Algorithm 3c]{wold1983multivariate}. Since the $\mathbf{t}_k$'s are obtained by deflation, the test of convergence is done on each component separately. If one of the components does not converge, we consider that the algorithm did not converge and if all the components converge, then the number of iterations is the mean of the total number of iterations.
\end{itemize}
$100$ simulated data sets have been generated 
with $T=10$ blocks of $p=160$ covariates (with $\rho_{d}=0.9$ and $\rho_t=0.9$), $30\%$ of missing values and $n\in\{100, 50, 20\}$.
For each considered method, results on convergence rate and number of iterations are presented in the first two parts of Table~\ref{tab:time_tab}. The prediction errors have also been computed and are represented in Figure~\ref{fig:vary_n}.
\begin{table}[!ht]
\resizebox{\textwidth}{!}{
\centering
\begin{tabular}{|l|c|c|c|c|c|}
\cline{2-6}
\multicolumn{1}{c}{}&\multicolumn{2}{|c|}{\backslashbox{Method}{$\#$ individuals}}&\multirow{2}{*}{100}&\multirow{2}{*}{50}&\multirow{2}{*}{20}\\
\cline{2-3}
\multicolumn{1}{c|}{}&Imputation&Prediction&&&\\
\hline
\hline
\multirow{5}{*}{\rotatebox{90}{{\small Conv. rate}}}&\multicolumn{2}{|c|}{mdd-sPLS with Koh-Lanta}&\bf100 \% &\bf 100 \% & \bf100   \%  \\
\cline{2-6}
&\multicolumn{2}{|c|}{nipals}&\bf100 \% &\bf 100 \% & \bf100 \%    \\
\cline{2-6}
&\multicolumn{2}{|c|}{imputeMFA}&99.4 \% & 98.8 \% & 99.8   \%  \\
\cline{2-6}
&\multicolumn{2}{|c|}{softImpute}&71.5 \% & 85.9 \% & 92.3   \%  \\
\hline
\hline
\multirow{5}{*}{\rotatebox{90}{{\small$\#$ iterations}}}&\multicolumn{2}{|c|}{mdd-sPLS with Koh-Lanta}&\bf3     $\pm$     0 & \bf3     $\pm$     0 &\bf 3     $\pm$     0     \\
\cline{2-6}
&\multicolumn{2}{|c|}{nipals}&42.6     $\pm$     8.64 & 39.2     $\pm$     8.29 & 48.7     $\pm$     7.71     \\
\cline{2-6}
&\multicolumn{2}{|c|}{imputeMFA}&27.0     $\pm$     10.1& 29.7     $\pm$     11.5 & 31.2     $\pm$     8.91     \\
\cline{2-6}
&\multicolumn{2}{|c|}{softImpute}&71    $\pm$     13.3 & 66.6     $\pm$     12.9 & 72.7     $\pm$      12.5\\
\hline
\hline
\multirow{5}{*}{\rotatebox{90}{{\small Time (s)}}}&\multicolumn{2}{|c|}{mdd-sPLS with Koh-Lanta}&\bf 0.662     $\pm$     0.209 &\bf 0.343    $\pm$     0.0403 &\bf 0.315     $\pm$     0.0550     \\
\cline{2-6}
&nipals & classic-sPLS&33.0     $\pm$     5.27 & 18.  $\pm$      3.77 & 22.1     $\pm$     3.48     \\
\cline{2-6}
&imputeMFA&mdd-sPLS&9.44     $\pm$     3.39 & 3.93     $\pm$     1.45 & 3.12     $\pm$     0.600     \\
\cline{2-6}
&softImpute&mdd-sPLS&2.00 $\pm$ 0.984 & 0.849 $\pm$ 0.155 & 1.14 $\pm$ 0.175       \\
\cline{2-6}
&Mean&mdd-sPLS&\bf 0.0124     $\pm$     0.00215 &\bf 0.00683     $\pm$     0.00069 &\bf 0.00410     $\pm$     0.000469\\
\hline
\end{tabular}
}
\caption{Effect of number of individuals. $100$ simulation results for $T=10$ blocks of $p=160$ covariates (with $\rho_{d}=0.9$ and $\rho_t=0.9$) and for $30\%$ of missing values. The main statistics (over the $100$ simulations) are given for each method with {(mean $\pm$ std)}. That table is divided in three parts. 
The first part (lines 1 to 4) corresponds to the convergence rate for each imputation method. 
The second part (lines 5 to 8) corresponds to the number of iterations for each imputation method, only in case of convergence. The {Mean} imputation method (that works in 1 iteration and thus always converges) is not taken into account in those two parts. 
The third part (lines 9 to 13) corresponds to the computation time for each method. $mean$'s are calculated not only over the $100$ simulated data sets but also over the number of individuals (indicated in the column and having an impact on  leave-one-out computations) in order to ``standardize'' the results to the time to that of creating a single model one model. Bolded results correspond to best results for a given number $n$ of individuals.\label{tab:time_tab}}
\end{table}

\textbf{nipals} and \textbf{mdd-sPLS} with {\bf Koh-Lanta} get $100\%$ convergence. \textbf{imputeMFA} almost always converged while for \textbf{softImpute} almost $30\%$ of imputation processes did not converge when $n=100$.
Concerning the number of iterations, denoted \textit{$\#$ iterations} in Table~\ref{tab:time_tab}, the \textbf{mdd-sPLS}  with {\bf Koh-Lanta} only needs $3$ iterations before converging. \textbf{softImpute} shows a high number of iterations. \textbf{nipals} needed less iterations to converge.  \textbf{imputeMFA} used an average of $\sim 30$ iterations with a large standard deviation relatively to the other methods.

Computations have been performed on \texttt{Intel\textregistered\ Xeon\textregistered\ CPU E5-2690 v2, 3.00GHz} processors. Concerning  the computation time, one notice that the \textbf{Mean} process naturally is the fastest. This intuitive result is followed by the \textbf{mdd-sPLS} with {\bf Koh-Lanta} approach for which the computation time lasts $\sim 0.5$ seconds. On the contrary the \textbf{nipals} method lasts within tens of seconds, \textbf{imputeMFA} is faster but still lasts within seconds (from 3.1 to 9.4 seconds). The \textbf{softImpute} method is faster, less than $2$ seconds.
Not surprisingly, the computation time of almost all methods decreased as the number $n$ of individuals decreased, with the exception of the \textbf{nipals} and {\bf softImpute} algorithms.

%=======================================================
\subsubsection{Conclusion from the Simulations}
%=======================================================

In comparison with the other competing methods, \textbf{mddsPLS} with {\bf Koh-Lanta} clearly exhibits very good performances in terms of predictive capacities in the context of a large proportion of missing values and small number $n$ of individuals. This is shown in the context of strongly correlated blocks as well as in the context of low inter-block information correlation (small $\rho_t$). Another set of simulations show the robustness of the results for low $\rho_d$ and low $\rho_t$ and is presented in Appendix~\ref{app:sup_simu}.

%=======================================================
%=======================================================
%=======================================================
\section{Real Data Application: the Ebola rVSV-ZEBOV Data Set}
\label{S:5}
%=======================================================
%=======================================================
%=======================================================

The current work was inspired by this real data application.

%=======================================================
\subsection{The Data Set}
%=======================================================

The application is an early phase vaccine trial evaluating the rVSV-ZEBOV Ebola vaccine already studied by~\citet{rechtienrichertlorenzo}. As many modern early vaccine trials it includes small number of participants (here, $n=18$) with heterogeneous and high dimensional data sets carrying a lot of information through numerous covariates allowing a deep evaluation of the response to the vaccine.

More specifically, for each participant, the gene expression in whole blood by RNA-seq and the cellular functionality by cytometry have been measured at four different days $\in\{0,1,3,7\}$ after vaccination. Genes of interest were pre-selected by removing those with a variance less than $0.2$ leading to $18~301$ genes included in the following analysis. The cellular functionality consisted in the characterization of \textit{Natural killers}, \textit{Dendritic cells} and \textit{Cytokines}, covering a total of 129 variables. 
So, $T=8$ blocks ${\bf X}_t$ of covariates were available, see Table~\ref{tab:tab_recap_vars} for the number of covariates in each ${\bf X}_t$'s blocks. 

Moreover, the antibody responses against the Gueckedou strain by ELISA have been measured at days $\in\{28,56,84,180\}$ after vaccination, so ${\bf Y}\in\mathbb{R}^{18\times 4}$. The aim of the analysis was to find the best predictors of the antibody responses among the gene expression  (transcriptome) and the cellular functionality.

Recall that the \textbf{mdd-sPLS} method works with standardized variables. This standardization step implies that the information contained in the variance is not taken into account for each of the variables.

\begin{table}[ht]
\centering
\begin{tabular}{|l|c|c|c|c|c|c|c|c|}
  \hline
 Type &\multicolumn{4}{c|}{RNA-SEQ} & \multicolumn{4}{c|}{Cellular functionality}\\
    \hline
  Block & ${\bf X}_1$ & ${\bf X}_2$ & ${\bf X}_3$ & ${\bf X}_4$ & ${\bf X}_5$ & ${\bf X}_6$ & ${\bf X}_7$ & ${\bf X}_8$ \\
   \hline
  Day & 0 & 1 & 3 & 7 & 0 & 1 & 3 & 7 \\
  \hline
$\#(variables)$ & 10279 & 10134 & 9082 & 9670 & 129 & 129 & 129 & 129 \\ 
   \hline
\end{tabular}
  \caption{Number of covariates per block ${\bf X}_t$.
    \label{tab:tab_recap_vars}}
\end{table}

Because of the sample quality constraints, gene expression was not available in about 30\% of cases leading to missing values.
For example, Table~\ref{tab:missing_path} shows the absence (in blue) of all the RNA-Seq values for a particular individual (in columns) for a particular day (in rows) depicting around $30\%$ of missing values. The data set used in that table is available on the NCBI repository.

%~\href{https://www.ncbi.nlm.nih.gov/geo/query/acc.cgi?acc=GSE97590}{here}, on a NCBI repository\footnote{This is a \textbf{R} object corresponding to the normalized  data but this is also possible to get the raw data on the same page.}.

\begin{table}[ht]
\resizebox{\textwidth}{!}{
  \centering
\begin{tabular}{|c!{\vrule width1pt}c|c|c|c|c|c|c|c|c|c|c|c|c|c|c|c|c|c|c!{\vrule width1pt}|}
\hline
Individuals & 7 & 5 & 9 & 1 & 15 & 10 & 14& 4&2&12&17&16&8&18&13&11&3&6\\
\hline
  Day 0: $\mbox{Vec}({\bf X}_1^T)^T$ &   &  \cellcolor{blue} & \cellcolor{blue} & \cellcolor{blue}  &   & &\cellcolor{blue} & & & & & & &\cellcolor{blue} & & & &\\
  \hline
  Day 1: $\mbox{Vec}({\bf X}_2^T)^T$ & \cellcolor{blue}  &  \cellcolor{blue} & \cellcolor{blue} &   &   & & & & & & &\cellcolor{blue} & &\cellcolor{blue} & & & &\\
  \hline
  Day 3: $\mbox{Vec}({\bf X}_3^T)^T$ &   &   & \cellcolor{blue} &   &   & & & & \cellcolor{blue}& & & & \cellcolor{blue}& & & & &\\
  \hline
  Day 7: $\mbox{Vec}({\bf X}_4^T)^T$ &   &   &  &   &   & & & & \cellcolor{blue}& & & & \cellcolor{blue}&\cellcolor{blue} & & & &\\
  \hline
\end{tabular}
}
  \caption{Missing values (in blue) in the Ebola rVSV-ZEBOV RNA-Seq data sets ${\bf X}_t,~t=1,\dots,4$, where the notation Vec stands for the Vec operator. For a given individual and a given day, all the corresponding values are missing.
  \label{tab:missing_path}}
\end{table}

%=======================================================
\subsection{Statistical Analysis}
%=======================================================

Four {\bf mdd-sPLS}-based methodologies were compared through \textbf{MSEP} (means square error of prediction) calculated by leave-one-out cross-validation:
\begin{itemize}
    \item {\bf mdd-sPLS} with {\bf Koh-Lanta},
    \item two-step approach: imputation to the {\bf mean} + {\bf mdd-sPLS},
    \item two-step approach: imputation with {\bf softImpute} + {\bf mdd-sPLS},
    \item two-step approach: imputation with {\bf imputeMFA} + {\bf mdd-sPLS}.
\end{itemize}
Figure~\ref{fig:error_mean_mixOmics_28_56_84_180_GOOD_compare} focuses on {\bf mdd-sPLS} with {\bf Koh-Lanta} and {\bf mean} + {\bf mdd-sPLS} and shows the number of times each response variable has been selected for every optimal $\lambda$ value. All comparisons are provided in Table~\ref{tab:yab_res}.  
Since the \textbf{softImpute} method uses random initialization and does not converge systematically, a variability appears in the prediction errors, here depicted by the ($mean\pm std$) notation.
All the methods led to the selection of the day 56 response variable in the model, the only variable that was always selected by the four methods. {\bf mdd-sPLS} with {\bf Koh-Lanta} clearly retains two response variables in the model: day 56 and day 84.

\begin{table}[!ht]
\centering
\resizebox{\textwidth}{!}{
\begin{tabular}{|c|c||c|c|c|c|c|c|c|c||c|}
\cline{3-11}
\multicolumn{2}{c|}{}&\multicolumn{9}{c|}{Leave-One-Out prediction error}\\
\hline
\multicolumn{2}{|c||}{Method}& 
\multicolumn{2}{c|}{Day 28} &
\multicolumn{2}{c|}{Day 56} &
\multicolumn{2}{c|}{Day 84} &
\multicolumn{2}{c||}{Day 180} &
Mean
\\
\cline{1-10}
Imputation&Prediction& \textbf{RMSEP} & $\#$ 
& \textbf{RMSEP} & $\#$ &  \textbf{RMSEP} & $\#$ & \textbf{RMSEP} & $\#$ & \textbf{RMSEP}\\
\hline
\hline
\multicolumn{2}{|c||}{mdd-sPLS with Koh-lanta}&\multirow{2}{*}{\bf 1.027}&\multirow{2}{*}{4/18}&\multirow{2}{*}{}&\multirow{2}{*}{18/18}&&\multirow{2}{*}{}&\multirow{2}{*}{\bf 1.029}&\multirow{2}{*}{1/18}&\multirow{2}{*}{\bf 0.9035}\\
\multicolumn{2}{|c||}{$\lambda= 0.8654$}&&&\multirow{-2}{*}{\bf 0.6143}&&\multirow{-2}{*}{\bf 0.9426}&\multirow{-2}{*}{17/18}&&&\\
\hline
\multirow{2}{*}{Mean} &mdd-sPLS&
\multirow{2}{*}{1.028}&\multirow{2}{*}{2/18}&&\multirow{2}{*}{18/18}&&\multirow{2}{*}{}&\multirow{2}{*}{\bf 1.029}&\multirow{2}{*}{0/18}&\multirow{2}{*}{0.9326}\\
&$\lambda= 0.863$&&&\multirow{-2}{*}{0.6312}&&\multirow{-2}{*}{1.041}&\multirow{-2}{*}{6/18}&&&\\
\hline
\multirow{2}{*}{softImpute} &mdd-sPLS& \multirow{2}{*}{1.029$\pm$0}&\multirow{2}{*}{($0\pm0$)/18}&&\multirow{2}{*}{($18\pm0$)/18}&&\multirow{2}{*}{}&\multirow{2}{*}{\bf 1.029$\bf \pm0.0001374$}&\multirow{2}{*}{($0.3\pm0.5$)/18}&\multirow{2}{*}{0.9294}\\
&$\lambda=0.8566667$&&&\multirow{-2}{*}{$0.6326\pm0.03795$}&&\multirow{-2}{*}{$1.027\pm0.002191$}&\multirow{-2}{*}{($4.4\pm0.8$)/18}&&&\\
\hline
\multirow{2}{*}{imputeMFA} &mdd-sPLS& \multirow{2}{*}{ 1.028}&\multirow{2}{*}{3/18}&&\multirow{2}{*}{18/18}&&\multirow{2}{*}{}&\multirow{2}{*}{\bf 1.029}&\multirow{2}{*}{0/18}&\multirow{2}{*}{0.9433}\\
&$\lambda= 0.852222$&&&\multirow{-2}{*}{0.6899}&&\multirow{-2}{*}{ 1.026}&\multirow{-2}{*}{7/18}&&&\\
\hline
\end{tabular}
}
  \caption{Results of the leave-one-out cross-validation prediction errors applied to the rVSV data set. The last column gives the mean error. The \# symbol represents the number of times each variable is selected in the cross-validation process, among $18$ different models built in the cross-validation process (since $n=18$).
    \label{tab:tab_results_loo}}
\end{table}

% classic sPLS method has been tried on the data set with mean imputation, since the \textbf{nipals} function does not work on our data sets since there are too much variables. Figure~\ref{fig:error_mean_mixOmics_28_56_84_180_GOOD}) represents the leave-one-out errors for the 4 to be predicted variables. we have decided to compare this to the \textbf{mdd-sPLS} method with mean imputation, which means the transparent lines on figure~\ref{fig:error_mean_mixOmics_28_56_84_180_GOOD_compare}. Figure~\ref{fig:error_mean_mixOmics_28_56_84_180_GOOD}) insures that the classic sPLS is well suited to minimize the mean error while on figure~\ref{fig:error_mean_mixOmics_28_56_84_180_GOOD_compare}, $day \ {28}$ and $day \ {180}$ are surely badly predicted, not considered by the model. This is actually a good point for our model, it implies that we correctly select the most valuable information. Indeed, a MSEP around $0.7$ as this is the case for figure~\ref{fig:error_mean_mixOmics_28_56_84_180_GOOD} shows bad results predicting $day \ {28}$.\\

Table~\ref{tab:yab_res} shows the final model, selected as minimum for day 56, with $\lambda\simeq 0.866$. 
The {\bf mdd-sPLS} with {\bf Koh-Lanta} approach was highly selective as it kept three covariates while the other methods kept $15$ variables ~\cite[see][Figure S5 from supplementary materials]{rechtienrichertlorenzo}. As mentionned before, the selection over the $\bf Y$ part was also efficient and kept 2 response variables in the model: the antibody levels at day 56 and day 84. 

Finally, the selected covariates were biologically meaningful. The three genes (TIFA$_{day \ 1}$, SLC6A9$_{day \ 3}$, FAM129B$_{day \ 3}$) selected through the proposed approach were also selected by~\citet{rechtienrichertlorenzo} as the three top genes in the sensibility analysis realized with bootstrap analysis.
Note that the other three methodologies selected the same three covariates except for the \textbf{softImpute+mdd-sPLS} methods which did not select SLC6A9$_{day \ 3}$, the corresponding results are not provided here.
For each selected covariate in Table~\ref{tab:yab_res}, the absolute value of the product between the corresponding weight and super-weight gives a measure of its impact in the model. For TIFA$_{day \ 1}$, this value was equal to $0.961$, for SLC6A9$_{day \ 3}$ equal to $0.107$ and for FAM129B$_{day \ 3}$ equal to $0.255$. The interpretation was that TIFA$_{day \ 1}$ was the most important covariate while FAM129B$_{day \ 3}$ was the second most important one and SLC6A9$_{day \ 3}$ was the third most important one. Correlations, considering only present samples, between TIFA$_{day \ 1}$ and Gueckedou strain on days $\in \{28,56,84,180\}$ are respectively equal to $0.83$, $0.96$, $0.90$ and $0.81$. 

It is also interesting to interpret the parsimony of the antibody response measurements by selecting day 56 and day 84 and not response measurements at day 28 and day 180. This reflects probably the fact that once established, the antibody response is quite stable over every individual and therefore does not need many repeated measurements to be characterized. 

\begin{figure}[p]
	\centering
%     \begin{subfigure}[b]{.49\textwidth}
\includegraphics[width = 0.8\textwidth] {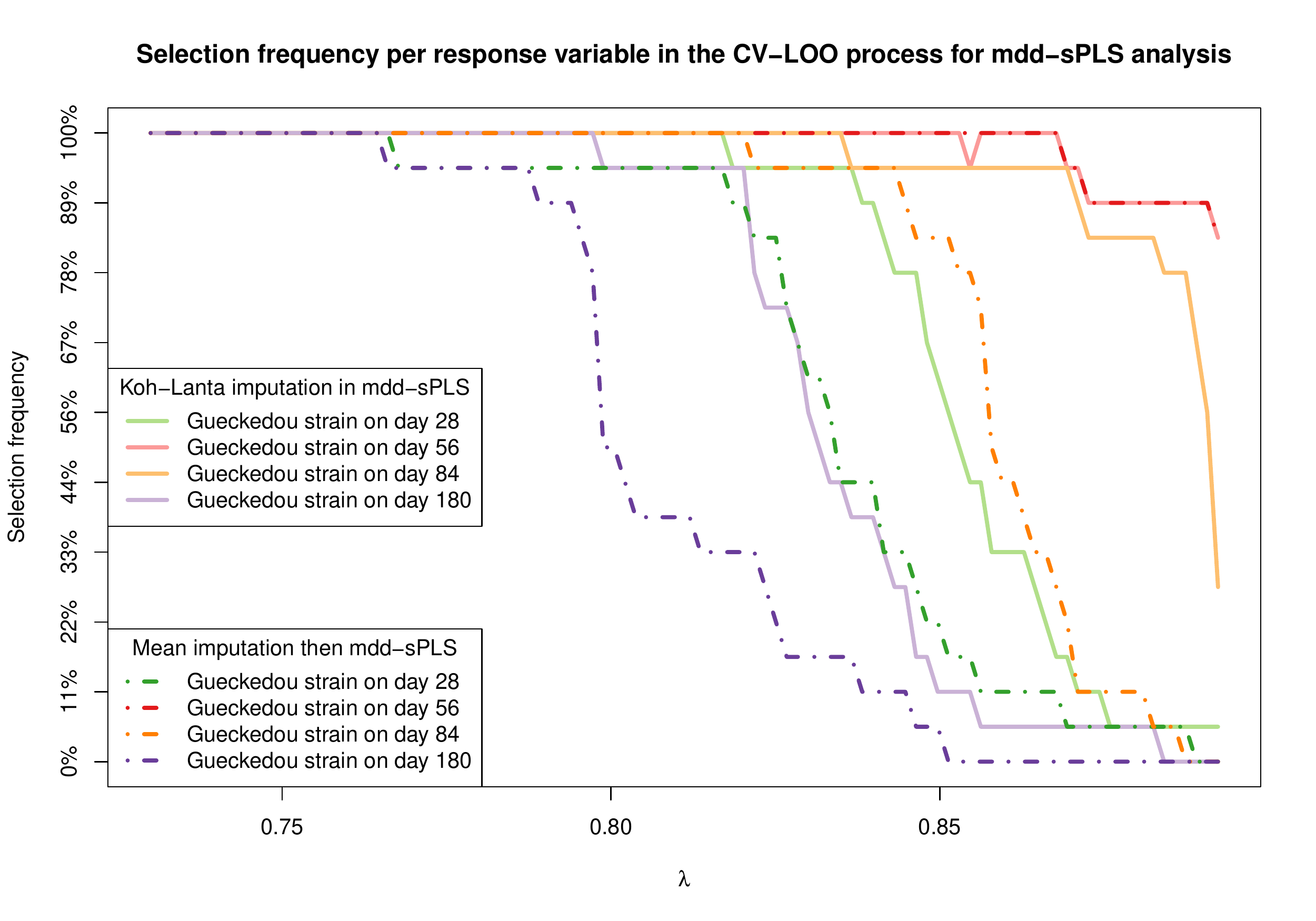}
\caption{Proportion of selection of each $\mathbf{Y}$ variable for both considered methods on the rVSV data set through Leave-One-Out Cross-Validation. Dotted lines show mean-imputation results and bolded lines full {mdd-sPLS} with the {Koh-Lanta} algorithm.}
\label{fig:error_mean_mixOmics_28_56_84_180_GOOD_compare}
\end{figure}

\begin{table}[!ht]
\centering
\begin{tabular}{|c||c|c|c|c|}
\hline
& 
Block&
Variable&
Weights&
Super-weights\\
\hline
\hline
\multirow{3}{*}{X}&Genes on day 1&TIFA&1&-0.961\\
\cline{2-5}
&\multirow{2}{*}{Genes on day 3}&SLC6A9&0.388&\multirow{2}{*}{0.277}\\
\cline{3-4}
&&FAM129B&0.922&\\
\hline
\hline
\multirow{2}{*}{Y}&\multirow{2}{*}{Response}&Gueckedou on day 56&-0.924&\multirow{2}{*}{$\times$}\\
\cline{3-4}
&&Gueckedou on day 84&-0.382&\\
\hline
\end{tabular}
  \caption{Ebola rVSV phase I/II constructed {mdd-sPLS} model for $\lambda=0.8654$. $4$ blocks of gene expression and $4$ blocks of cellular functionality, one for each of the days $\{0,1,3,7\}$, have been introduced but only days $1$ and $3$ blocks of gene expression have been selected. Also, in the response block, only days $56$ and $84$ have been selected. In columns are represented the weights which denote $u_t^{(1)}$ for the $\mathbf{X}$ blocks and $v^{(1)}$ for the $\mathbf{Y}$ block and also the super-weights for the $\mathbf{X}$ blocks and denoted by $\beta_t^{(1)}$. Only one dimension was found interesting here.
    \label{tab:yab_res}}
\end{table}

%=========================================
%=========================================
%=========================================
\section{Conclusion}
\label{S:6}
%=========================================
%=========================================
%=========================================

The \textbf{mdd-sPLS} method is a \textbf{SVD}-based method (without iteration process) dedicated to multi-block supervised analysis. The \textbf{Koh-Lanta} algorithm deals with missing values in the \textit{train} sample but also in the \textit{test} sample and is implemented in the \textbf{mdd-sPLS} method. The considered method shows very good performance on simulated data sets and gave relevant results in the real data application. This approach allows to make variable selection and missing values imputation. 
The missing data context is limited to entire rows of missing values for certain blocks and can be generalized to any position of missing values by adjusting missing values thanks to known values through a linear model for example. Most of the results, described in this paper, relate to regression problem but the method can also be applied to classification problem. 

The \textbf{mdd-sPLS} including {\bf Koh-Lanta} algorithm method  has been implemented in:
 \begin{itemize}[topsep=0pt,itemsep=-1ex]
 \item a \textbf{R}-package accessible on the \textbf{CRAN}, \url{https://cran.r-project.org/package=ddsPLS},
 \item a \textbf{python}-package accessible on \textbf{PyPi}, \url{https://pypi.org/project/py_ddspls/}.
 \end{itemize}

%=========================================
%=========================================
%=========================================
\section*{Acknowledgments}
\label{S:7}
%=========================================
%=========================================
%=========================================

The authors would like to thank François Husson, Arthur Tenenhaus and Julie Josse for helppful discussions. Hadrien Lorenzo is supported by a 2016 Inria-Inserm thesis grant \textit{M\'{e}decine Num\'{e}rique} (for \textit{Digital Medicine}).

\appendix
\addcontentsline{toc}{section}{Appendix~\ref{app:scripts}: Training Scripts}

\section{Monotonicity of the Weight Cardinality: a Counter Example}
\label{app:1}
%=======================================================================

The remaining question is about the potential decreasing of the number of variables selected per component. In other words, is the number of null coefficients of a given component a decreasing function of $\lambda$? 
The answer is no as we will see through the following counterexample.

A sample of $n=100$ individuals is generated  with the following correlation structure between a 9-dimensional covariate  and a two-dimensional response:
\begin{table}[!ht]
$\frac{\mathbf{Y}^T\mathbf{X}}{n-1}=$
\centering
\begin{tabular}{|c|c|c|c|c|c|c|c|c|c|}
  \hline
& $X_{1}$ & $X_{2}$ & $X_{3}$ & $X_{4}$ & $X_{5}$ & $X_{6}$ & $X_{7}$ & $X_{8}$ & $X_{9}$ \\ 
  \hline
 $Y_{1}$ & 1.00 & -0.06 & -0.10 & 0.07 & 0.09 & 0.15 & 0.16 & 0.14 & 0.22 \\ 
 \hline
  $Y_{2}$ & -0.08 & 0.98 & 0.29 & -0.18 & 0.25 & 0.02 & 0.04 & -0.01 & -0.03 \\ 
   \hline
\end{tabular}.
\end{table}

The $1^{st}$ and the $2^{nd}$ $\mathbf{X}$-variables are clearly well correlated respectively with the $1^{st}$ and the $2^{nd}$ $\mathbf{Y}$-variables. Let us denote by $\bf u$, respectively $\bf v$, the first right, respectively left, eigen vector of the soft-thresholded covariance matrix $S_\lambda(\frac{\mathbf{Y}^T\mathbf{X}}{n-1})$ for any positive $\lambda$. Figure~\ref{fig:path_ce} shows the real cardinalities (\textbf{black lines}) and the upper bound cardinalities (\textcolor{red}{red lines}) of $\bf u$ and $\bf v$ weights which correspond the application of Corollary~\ref{coro:1}. Vertical lines (\textcolor{blue}{discontinuous blue lines}) symbolize the vanishing of a coefficient of the current matrix $S_\lambda(\frac{\mathbf{Y}^T\mathbf{X}}{n-1})$, depending on $\lambda$. When $\lambda\in[0.1,0.14]$, $\mathbf{Card}({\bf u})$ increases, this corresponds to an area in which $\mathbf{Card}({\bf u})$ might take the value $9$, in that part all the columns are different from $0$, except for $\lambda=0.14$, where the  variable $X_8$ ``disappears''.
\begin{figure}[p]
	\centering
	\includegraphics[width = 0.8\textwidth] {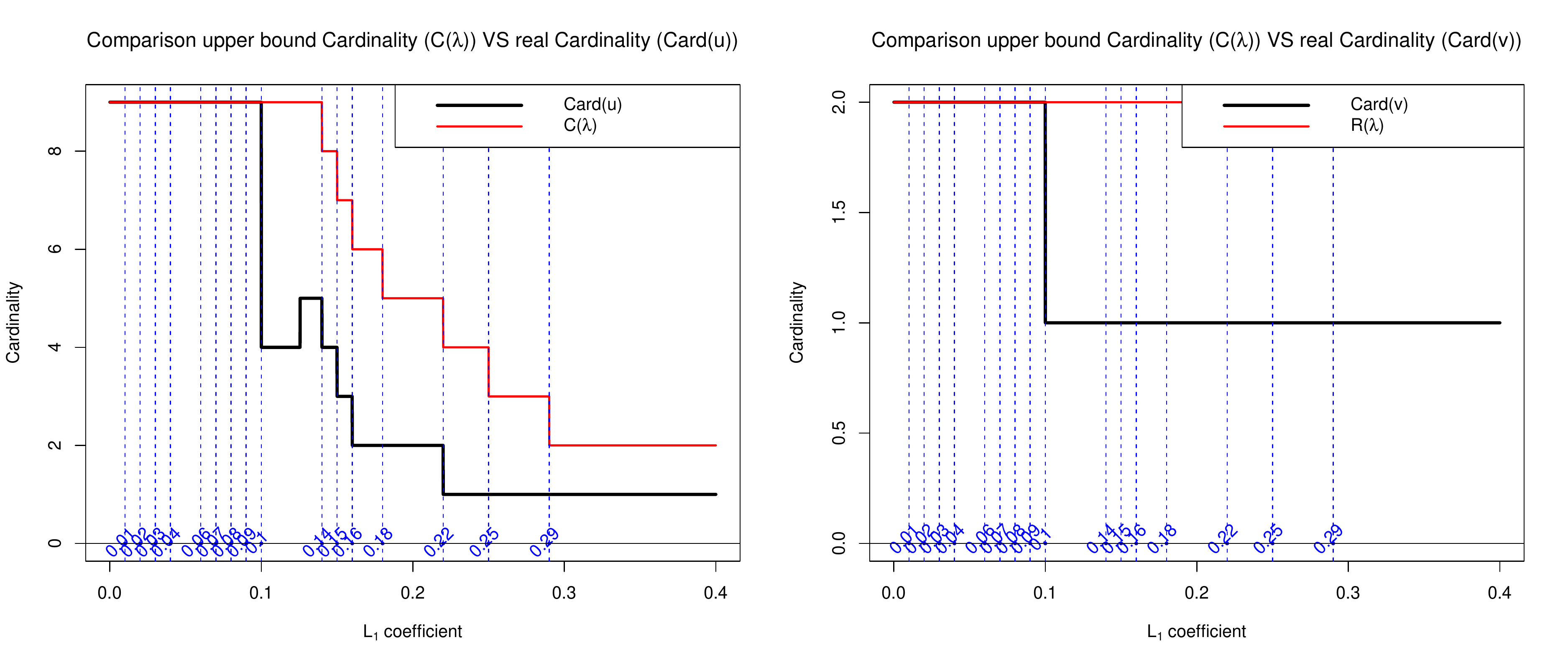}
	\caption{Behaviors of the cardinalities and their upper bounds, defined with corollary~\ref{coro:1}, for a simulation case, first component.
    \label{fig:path_ce}}
\end{figure}
Let us zoom on those particular points:
\[
\begin{split}
u(\lambda=0.12)\approx
\begin{matrix} 
(0 ,0.97 ,0.19,-0.07,0.15,0,0,0, 0)^T
\end{matrix} &\rightarrow \mathbf{Card}({\bf u}(\lambda=0.12))=4\\
u(\lambda=0.13)\approx
\begin{matrix} 
(0.99 ,0 ,0,0,0,0.023,0.034,0.011, 0.10)^T
\end{matrix}& \rightarrow \mathbf{Card}({\bf u}(\lambda=0.13))=5,
\end{split}
\]
This is due to the fact that the order of the components associated with the first two-dimensional eigenspace is defined through the $\mathcal{L}_2$-norm of the components. However, the $\mathcal{L}_1$-shrinkage of the coefficients based on $\lambda$ can change this order since the power of both the first two components are very close to each other, , and only in that case. A way of avoiding this kind of reversal would be to change the soft-thresholding operation with a more $\mathcal{L}_2$-shrinkage flavored operation such as the {\bf SCAD operation}, see for example~\citet{fan2001variable}. But in real cases, the first components are not often sufficiently close in the $\mathcal{L}_2$-norm sense to observe this kind of reversal. Thus, it was decided to keep the soft-thresholding operator in the \textbf{mdds-PLS} method.

%=======================================================================
\section{Regression Example: the Liver Toxicity Data Set}
\label{app:Reg}
%=======================================================================

 In the liver toxicity data set~\citep[see][]{heinloth2004gene} $n=64$ male rats of the inbred strain Fisher 334 were exposed to non toxic (50 or 150 mg/kg), moderately toxic (1500 mg/kg) or severely toxic (2000 mg/kg) doses of acetaminophen (paracetamol) in a controlled experiment. The values of $\mathbf{X}\in\mathbb{R}^{64\times 3116}$ are RNA measures and  the values of $\mathbf{Y}\in\mathbb{R}^{64\times 10}$ are  clinical measures of markers for liver injury. 
 There are no missing values. A comparison of \textbf{classic-sPLS} and the proposed \textbf{mdd-sPLS} is given in Figure~\ref{fig:data_liver}.
\begin{figure}[p]
	\centering
	\includegraphics[width = 6in] {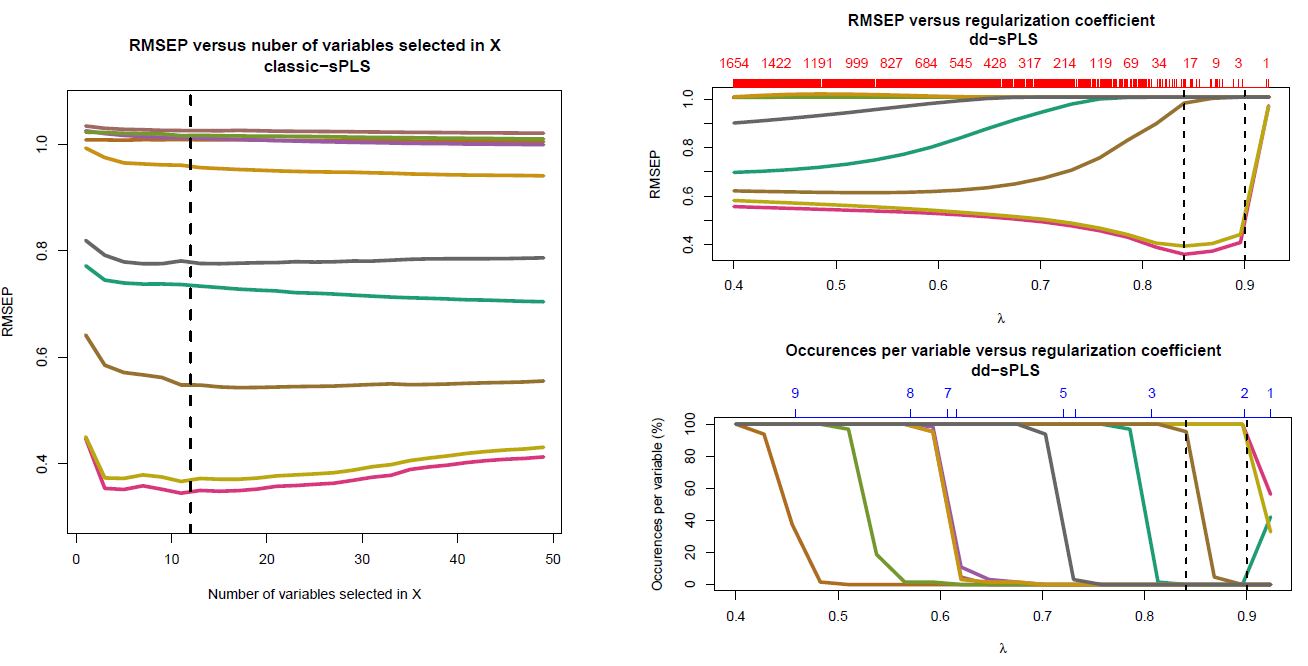}
	\caption{Results of Cross-Validation Leave One Out for both the {\bf classic-sPLS} and {\bf mdd-sPLS} methods for the Liver Toxicity data set.
    \label{fig:data_liver}}
\end{figure}
The first two graphics provide the results of leave-one-out cross-validation for the associated tuning parameters: the number of variables for \textbf{classic-sPLS} and the parameter $\lambda$ for \textbf{mdd-sPLS}. The third graphic gives the number of occurrences of each of the ${\bf Y}$'s response variables versus $\lambda$. These 10 response variables are plotted independently, so one curve represents the behavior of one response variable depending on the model and the regularization parameter chosen. Let us comment the results of  these two approaches.  
\begin{itemize}
\item {\bf classic-sPLS}. According to the suggestion of~\citet{le2008sparse}, the $keep_Y$ parameter has arbitrarily been fixed to 2. With this choice, 8 response variables are removed systematically  from the built models in the cross-validation process. One hope that the same two response variables are almost always selected, but this is clearly not the case. Indeed, if it were, 8 response variables should have a RMSEP value around 1 (the mean-prediction error, since the variables were standardized), which means that the model does not take into account those variables. However, this observation is only valid for only 5 response. This is problematic because the obtained sparse model strongly depends on this arbitrary choice of the $keep_Y$ parameter and thus may provide a wrong model for both selection and prediction.

\item \textbf{mdd-sPLS}.
The graphic of the RMSEP errors versus $\lambda$ clearly shows that 
5 response variables are already estimated to the mean when $\lambda=0.4$ (with corresponding RMSEP errors closed to 1) while the five others are still estimated by the model at this stage. Then, as the regularization parameter $\lambda$ increases, the variability of the RMSEP errors increases since the model has less and less information in the underlying soft-thresholded matrix. By carefully studying this graph, only 2 response variables (the 2 bottom ones) show real learning interest, observable through their decreasing curves while the other curves are increasing. The decreasing part reaches a minimum for $\lambda\approx 0.85$. One can see that this value coincides with the moment when the $3^{rd}$ response variable, in terms of increasing RMSEP ranking, reaches to 1, the symbolic limit of the error. This is equivalent to say that the model doesn't select that variable.

The graphic of the occurrences (per response variable) in the estimated models built in the cross-validation process also reinforces the user's choice to only select two response variables. 

\end{itemize}
\begin{table}[!ht]
\centering
\resizebox{\textwidth}{!}{
\begin{tabular}{|c|c||c|c|c|c|c|c|c|c|c|c|c|c|c|c||c|c|}
\hline
\multicolumn{2}{|c||}{Variable} & {\rotatebox[origin=c]{90}{A\_43\_P14131}} & {\rotatebox[origin=c]{90}{A\_42\_P620915}} & \rotatebox[origin=c]{90}{A\_43\_P11724} & \rotatebox[origin=c]{90}{A\_42\_P802628} & \rotatebox[origin=c]{90}{A\_43\_P10606} & \rotatebox[origin=c]{90}{A\_42\_P675890} & \rotatebox[origin=c]{90}{A\_43\_P23376} & \rotatebox[origin=c]{90}{A\_42\_P758454} & \rotatebox[origin=c]{90}{A\_42\_P578246} & \rotatebox[origin=c]{90}{A\_43\_P17415} & \rotatebox[origin=c]{90}{A\_42\_P610788} & \rotatebox[origin=c]{90}{A\_42\_P840776} & \rotatebox[origin=c]{90}{A\_42\_P705413} & \rotatebox[origin=c]{90}{A\_43\_P22616}&\cellcolor{red!20!white}{\rotatebox[origin=c]{90}{Mean RMSEP(LOO)}} &\cellcolor{red!20!white}{\rotatebox[origin=c]{90}{Min RMSEP(LOO)}} \\
\hline
\rotatebox[origin=c]{90}{classic-sPLS} & \rotatebox[origin=c]{45}{$k_X=12$} & {-0.6} & {-0.52} & 0.17&-0.12&-0.14&-0.18&-0.21&-0.18&-0.14&-0.33&-0.07&-0.26& \multicolumn{2}{c|}{}& \cellcolor{red!20!white}{0.78}&\cellcolor{red!20!white}{0.34}\\
\hline
\multirow{2}{*}{\rotatebox[origin=c]{90}{mdd-sPLS}} & \rotatebox[origin=c]{45}{$\lambda=0.845$} & {-0.6} & {-0.52} &0.17&-0.12&-0.14&-0.18&-0.21&-0.18&-0.14&-0.33&-0.07&-0.26&-0.03&-0.01&\cellcolor{red!20!white}{0.88}&\cellcolor{red!20!white}{0.36}\\\cline{2-18}
& \rotatebox[origin=c]{45}{$\lambda=0.9$} &{-0.86}&{-0.51}&\multicolumn{12}{c|}{}&\cellcolor{red!20!white}{0.89}&\cellcolor{red!20!white}{0.41}\\
\hline
\end{tabular}
}
\caption{Results for  {classic-sPLS} and {mdd-sPLS} methods: selected genes with their corresponding estimated weights.  The two last columns provide the mean and the minimum of the RMSEP errors calculated during the cross-validation process.\label{tab:tabEx}}
\end{table}

Table~\ref{tab:tabEx} provides results of $\bf X$'s weights obtained for \textbf{classic-sPLS} and \textbf{mdd-sPLS} with two choices of $\lambda$. Those models have been retained according to Figure~\ref{fig:data_liver}. 
\textbf{classic-sPLS} based on \textbf{mixOmics} R package selects 12 genes (with an optimal parameter $keep_X$ obtained by cross-validation) and 2 response variables (with the parameter $keep_y$ arbitrarily set to 2  by the user). This approach provides the lower cross-validation leave-one-out errors. 
For the \textbf{mdd-sPLS} method, one can clearly see that the best model, in terms of minimum RMSEP error   is not the sparsest one (with 14 genes selected including the 12 genes selected by \textbf{classic-sPLS}). But, looking at the degree of sparsity in $\mathbf{Y}$ also permits to select $\lambda=0.9$ as a good candidate. For that value of $\lambda$ (very close to the optimal one), the number of genes selected goes from 14 to 2 which is a very good model in terms of sparsity.
For $\lambda=0.9$, \textbf{mdd-sPLS} (with the first component only) provides an excellent selection simultaneously in $\bf X$ and $\bf Y$, with two variables selected in each matrix and a parameter $\lambda$ close to its optimal value in terms of cross-validation leave-one-out errors.

%=======================================================================
\section{Classification Example: the Penicillium YES Data Set}
\label{app:Classif}
%=======================================================================

The \textbf{Penicillium YES} data set (available in the \texttt{sparseLDA} package) is a classification data set describing three Penicillium species: melanoconodium, polonicum, and venetum. In this data set of size $n=36$ (with the three balanced groups),  $p=3542$ covariates were extracted from multi-spectral images with 18 spectral bands: $\mathbf{X}\in\mathbb{R}^{36\times 3542}$ and $\mathbf{Y}\in\mathbb{R}^{36\times 3}$ where the three columns of $\bf Y$ are the indicator variables of the groups). More details are available by~\citet{clemmensen2011sparse} where the interest of the \textbf{Sparse Discriminant Analysis} method is highlighted. The \textbf{Sparse Discriminant Analysis} method needed only 2 covariates to perfectly predict the assignment to one of the three groups.
%The authors used the \textbf{sPLS} method developed by~\citet{chun2010sparse} and recalled in the introduction through problem~\eqref{equ:equPLS4_intro}. 
%That method is known to be greedy in terms of computation parameters and time consuming and the authors could have advantageously used the lighter method developed by~\citet{le2009integromics}. As a remark the authors argue that complexity to not give numerical results of that method on the application to another data set in the discussed publication.\\
 A leave-one-out cross-validation has been performed for each fold $i=1,\dots,12$, the $i^{th}$ triplet of melanoconodium, polonicum, and venetum, to optimize the  parameter $\lambda$. 
The {\bf mdd-sPLS} method (with $\lambda=0.956$ for example) permits to select 4 different covariates, 2 on each of the two components, with a perfect assignment rate. These two components are plotted in Figure~\ref{fig:Penicillium_res} and the separation of the three groups is clearly visible.
\begin{figure}[p]
	\centering
	\includegraphics[width = 4in] {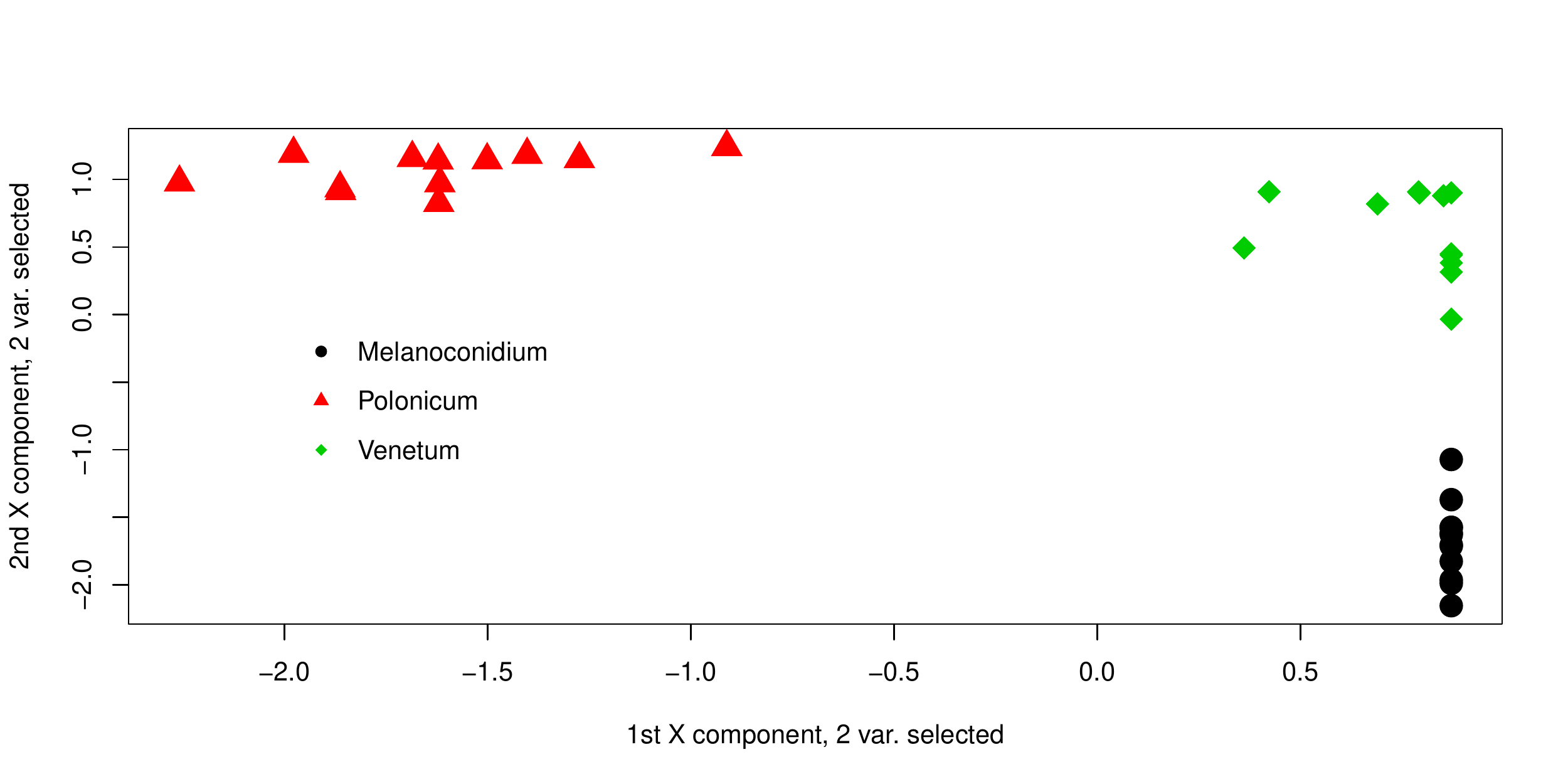}
	\caption{Results of the application of the {\bf mdd-sPLS} to the Penicillium YES data set.\label{fig:Penicillium_res}}
\end{figure}

%=======================================================================
\section{Effect of Varying Inter-Block Correlation}
\label{app:sup_inter_simu}
%=======================================================================

The case of varying inter-block correlation has been analysed in that part. The Figure~\ref{fig:vary_rho_t} summarizes the corresponding results. Other parameters have been respectively fixed to $\rho_d=0.9$, $30\%$ of missing values, 100 individuals per simulated data set and 100 simulations per $\rho_t$.

\begin{figure}[p]
	\centering
	\includegraphics[width = \textwidth]{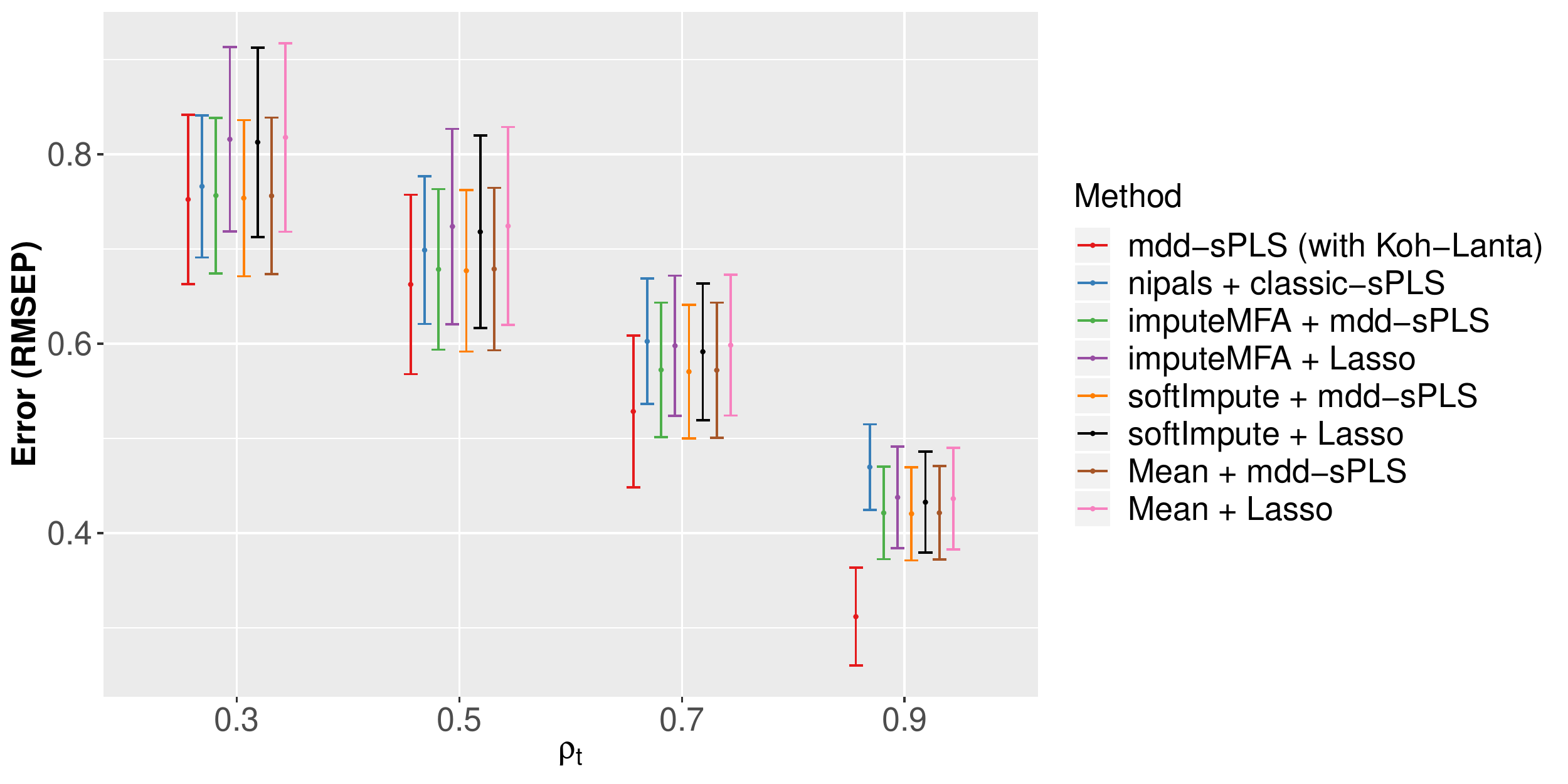}
	\caption{Effect of varying $\rho_t$ on the {RMSEP}, barplot of Table~\ref{tab:tab_no_inter} results. For fixed values of $\rho_d=0.9$, $30\%$ of missing values, 100 individuals per simulated data set and a total of $100$ simulated data sets
    \label{fig:vary_rho_t}}
\end{figure}

In the context of strongly correlated blocks, \textbf{mdd-sPLS (with Koh-Lanta)} seems to have a better behavior than the other methods where \textbf{nipals + classic-sPLS} is slightly less efficient than other methods where \textbf{mdd-sPLS} prediction based methods are lighlty better than \textbf{Lasso} prediction based methods. As the correlation between blocks shrinks, the \textbf{mdd-sPLS} prediction based methods and \textbf{mdd-sPLS (with Koh-Lanta)} show equivalent results, plus, \textbf{nipals + classic-sPLS} is almost as good as those methods. Finally in that low correlation context, the \textbf{Lasso} prediction based methods, are less efficient than other methods.

%=======================================================================
\section{Effect of Varying Intra-Block Correlation}
\label{app:sup_intra_simu}
%=======================================================================

The case of varying intra-block correlation has been analysed in that part. The Figure~\ref{fig:vary_rho_d} summarizes the corresponding results. Other parameters have been respectively fixed to $\rho_t=0.9$, $30\%$ of missing values, 100 individuals per simulated data set and 100 simulations per $\rho_d$.

\begin{figure}[p]
	\centering
	\includegraphics[width = \textwidth]{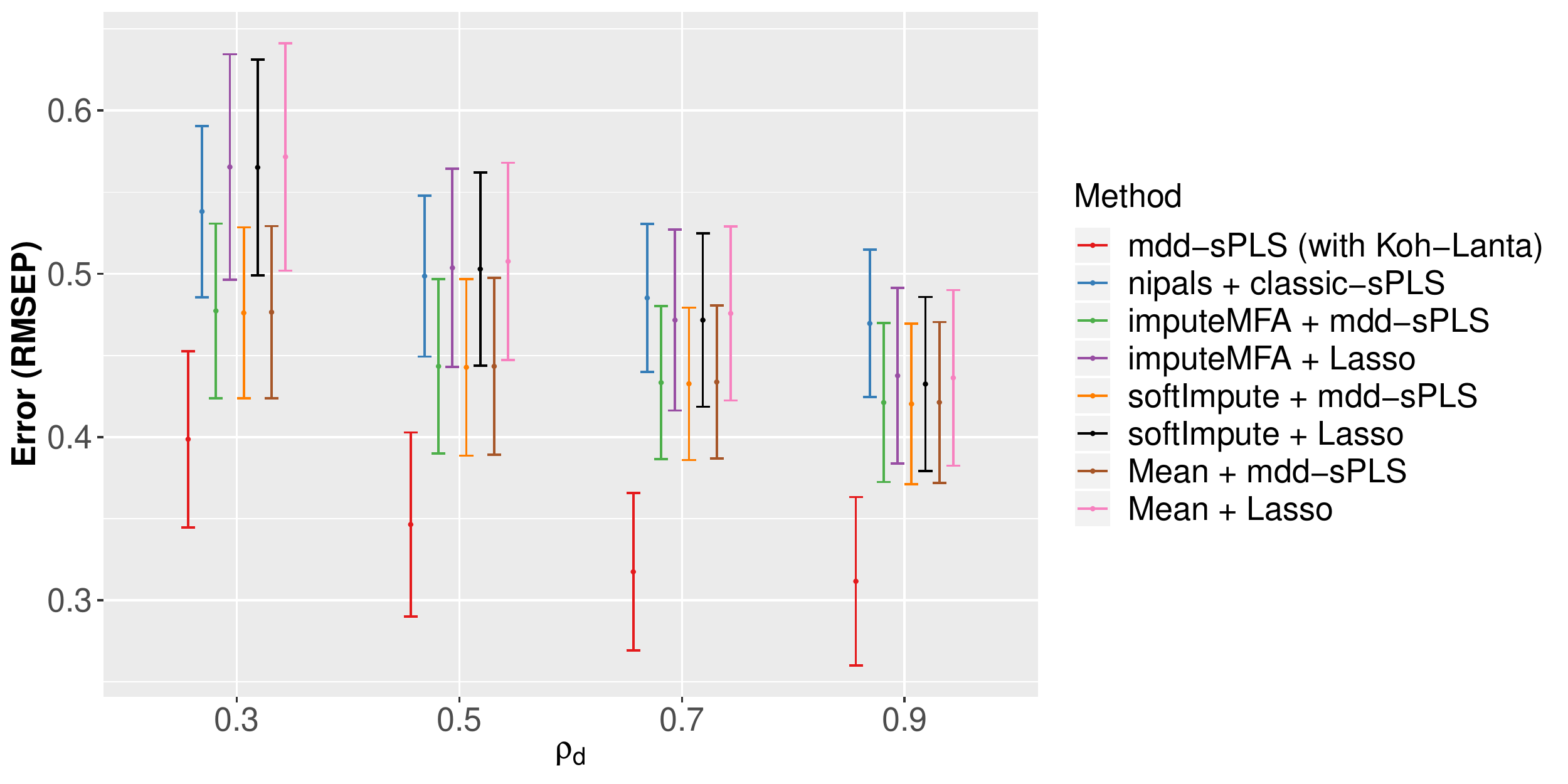}
	\caption{Effect of of varying $\rho_d$ on the {RMSEP}, barplot of Table~\ref{tab:tab_move_intra} results. For fixed values of $\rho_t=0.9$, $30\%$ of missing values, 100 individuals per simulated data set and a total of $100$ simulated data sets
    \label{fig:vary_rho_d}}
\end{figure}

Whatever is the intensity of the intra-block correlation, \textbf{mdd-sPLS (with Koh-Lanta)} shows better results in terms of prediction error. Furthermore the \textbf{mdd-sPLS} prediction based methods are the second ranked methods. The third position is given to the  \textbf{Lasso} prediction based methods for strong intra-block correlations and to \textbf{nipals + classic-sPLS} for low intra-block correlations.

%=======================================================================
\section{Effect of Varying Number of Individuals}
\label{app:sup_n}
%=======================================================================

The case of varying number of individuals has been analysed in that part. Figure~\ref{fig:vary_n} summarizes the corresponding results. Other parameters have been respectively fixed to $\rho_t=0.9$, $\rho_d=0.9$, $30\%$ of missing values, 100 individuals per simulated data set and 100 simulations per $n$.

\begin{figure}[p]
	\centering
	\includegraphics[width = \textwidth]{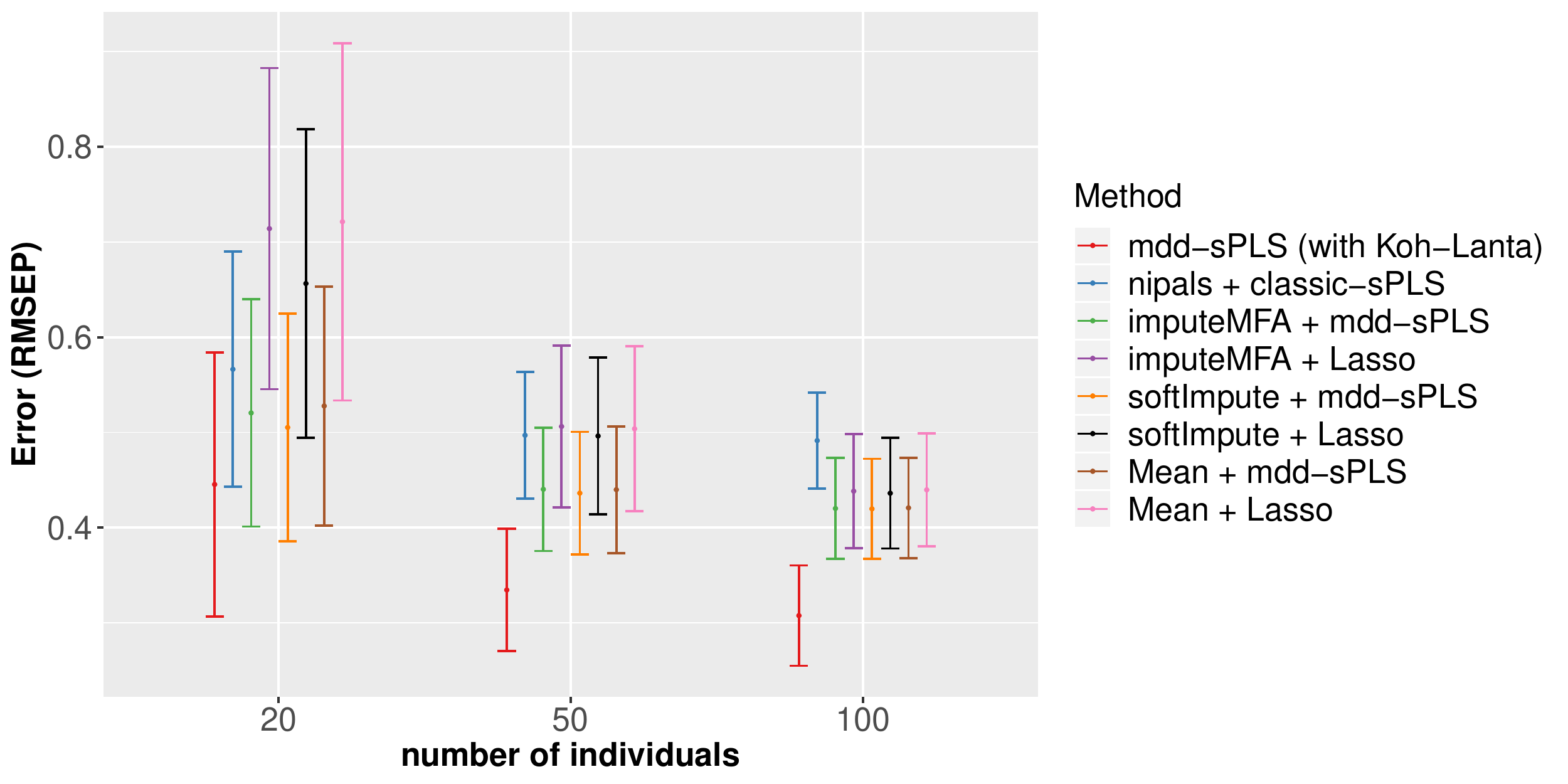}
	\caption{Effect of of varying number of participants on the {RMSEP}, barplot of Table~\ref{tab:time_tab} results. For fixed values of $\rho_t=0.9$, $\rho_d=0.9$, $30\%$ of missing values and a total of $100$ simulated data sets.
    \label{fig:vary_n}}
\end{figure}

Three cases have been considered. For a low number of individuals,  \textbf{mdd-sPLS (with Koh-Lanta)} has the smallest error, the \textbf{mdd-sPLS} prediction based methods and the \textbf{nipals + classic-sPLS} method are slightly less precise and the \textbf{Lasso} prediction based methods are fewer precise. As the number of individuals increase, the \textbf{nipals + classic-sPLS} method seems to be less precise than other methods while \textbf{mdd-sPLS (with Koh-Lanta)} finally gets very better results.

%=======================================================================
\section{Effect of Low Intra-Block and Inter-Block Simulations}
\label{app:sup_simu}
%=======================================================================

The following simulations permit to appreciate the validity of the proposed method if there is little information to catch between blocks (low inter-block correlation) and inside blocks (low intra-block correlation).

Here, the proportion of missing values has been fixed to $30\%$ and the number of individuals to $n=100$. The intra-block correlation has been fixed to $\rho_d=0.3$ and the inter-block correlation to $\rho_t=0.5$. Figure~\ref{fig:error_low_info_extreme} shows the results for $100$ simulated data sets.

\begin{figure}[p]
	\centering
	\includegraphics[width = \textwidth] {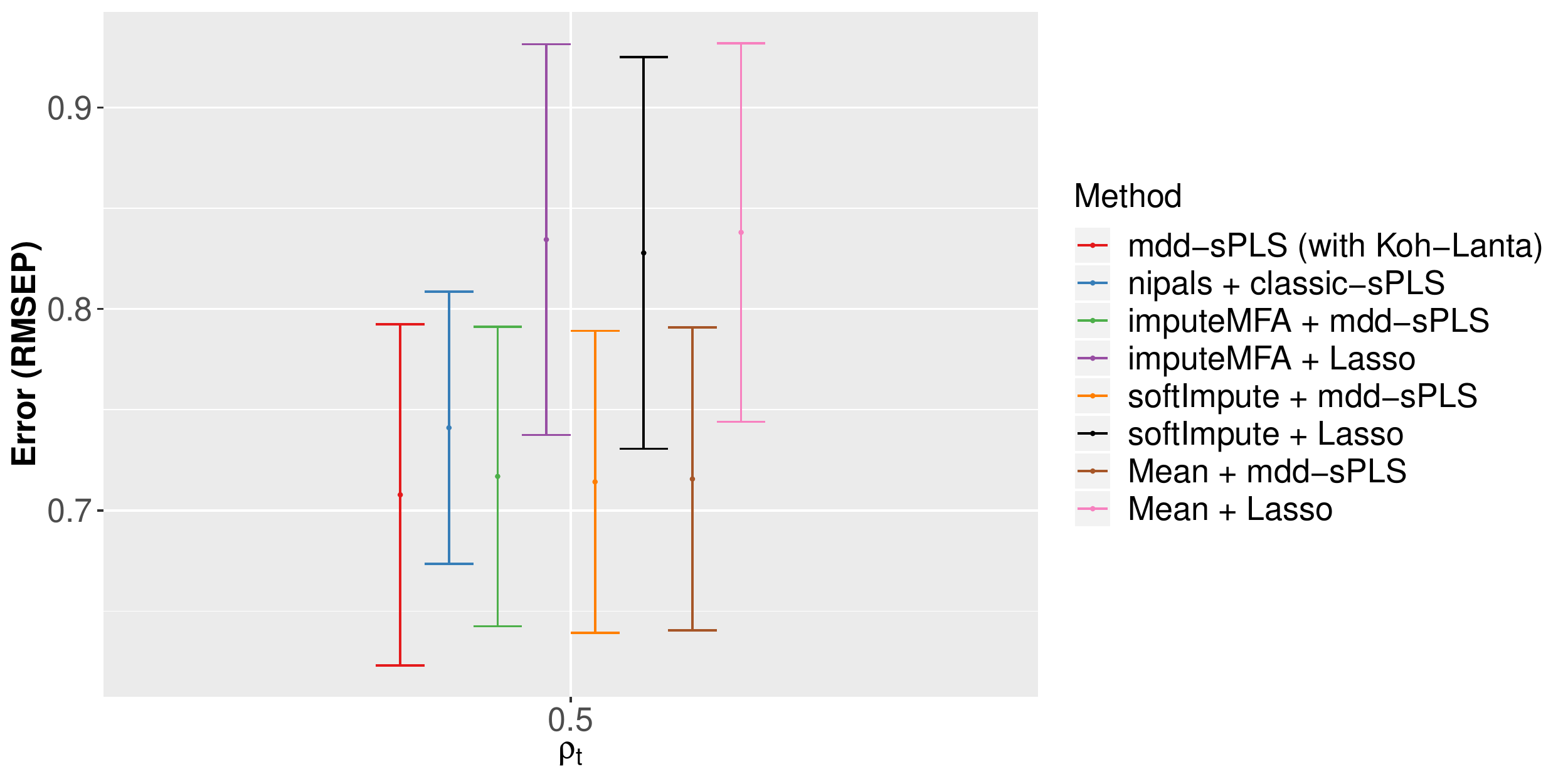}
	\caption{Effect of low intra-block correlation, $\rho_d=0.3$, and low inter-block correlation, $\rho_t=0.5$, on the overall errors for fixed values of of missing values, $30\%$, for $100$ individuals per simulated data set and a total of $100$ simulated data sets.
	\label{fig:error_low_info_extreme}}
\end{figure}

For that level of information, the all methods show equivalent performances, with an error close to $\approx 0.73$, except for the \textbf{Lasso} based prediction methods, for which the error is higher $\approx 0.83$.

\clearpage

%\bibliographystyle{plainnat}%model1-num-names}
%\bibliography{bib.bib}

\begin{thebibliography}{42}
\providecommand{\natexlab}[1]{#1}
\providecommand{\url}[1]{\texttt{#1}}
\expandafter\ifx\csname urlstyle\endcsname\relax
  \providecommand{\doi}[1]{doi: #1}\else
  \providecommand{\doi}{doi: \begingroup \urlstyle{rm}\Url}\fi

\bibitem[Amini and Wainwright(2008)]{amini2008high}
Arash~A Amini and Martin~J Wainwright.
\newblock High-dimensional analysis of semidefinite relaxations for sparse
  principal components.
\newblock In \emph{Information Theory, 2008. ISIT 2008. IEEE International
  Symposium on}, pages 2454--2458. IEEE, 2008.

\bibitem[Bertsimas et~al.(2018)Bertsimas, Pawlowski, and Zhuo]{JMLR:v18:17-073}
Dimitris Bertsimas, Colin Pawlowski, and Ying~Daisy Zhuo.
\newblock From predictive methods to missing data imputation: An optimization
  approach.
\newblock \emph{Journal of Machine Learning Research}, 18\penalty0
  (196):\penalty0 1--39, 2018.
\newblock URL \url{http://jmlr.org/papers/v18/17-073.html}.

\bibitem[Bougeard et~al.(2011)Bougeard, Qannari, Lupo, and
  Hanafi]{bougeard2011multiblock}
St{\'e}phanie Bougeard, El~Mostafa Qannari, Coralie Lupo, and Mohamed Hanafi.
\newblock From multiblock partial least squares to multiblock redundancy
  analysis. a continuum approach.
\newblock \emph{Informatica}, 22\penalty0 (1):\penalty0 11--26, 2011.

\bibitem[Buuren and Groothuis-Oudshoorn(2010)]{buuren2010mice}
S~van Buuren and Karin Groothuis-Oudshoorn.
\newblock mice: Multivariate imputation by chained equations in r.
\newblock \emph{Journal of statistical software}, pages 1--68, 2010.

\bibitem[Cai and Liu(2011)]{cai2011adaptive}
Tony Cai and Weidong Liu.
\newblock Adaptive thresholding for sparse covariance matrix estimation.
\newblock \emph{Journal of the American Statistical Association}, 106\penalty0
  (494):\penalty0 672--684, 2011.

\bibitem[Chun and Kele{\c{s}}(2010)]{chun2010sparse}
Hyonho Chun and S{\"u}nd{\"u}z Kele{\c{s}}.
\newblock Sparse partial least squares regression for simultaneous dimension
  reduction and variable selection.
\newblock \emph{Journal of the Royal Statistical Society: Series B (Statistical
  Methodology)}, 72\penalty0 (1):\penalty0 3--25, 2010.

\bibitem[Clemmensen et~al.(2011)Clemmensen, Hastie, Witten, and
  Ersb{\o}ll]{clemmensen2011sparse}
Line Clemmensen, Trevor Hastie, Daniela Witten, and Bjarne Ersb{\o}ll.
\newblock Sparse discriminant analysis.
\newblock \emph{Technometrics}, 53\penalty0 (4):\penalty0 406--413, 2011.

\bibitem[d'Aspremont et~al.(2005)d'Aspremont, Ghaoui, Jordan, and
  Lanckriet]{d2005direct}
Alexandre d'Aspremont, Laurent~E Ghaoui, Michael~I Jordan, and Gert~R
  Lanckriet.
\newblock A direct formulation for sparse pca using semidefinite programming.
\newblock In \emph{Advances in neural information processing systems}, pages
  41--48, 2005.

\bibitem[Deshpande and Montanari(2016)]{JMLR:v17:15-160}
Yash Deshpande and Andrea Montanari.
\newblock Sparse pca via covariance thresholding.
\newblock \emph{Journal of Machine Learning Research}, 17\penalty0
  (141):\penalty0 1--41, 2016.
\newblock URL \url{http://jmlr.org/papers/v17/15-160.html}.

\bibitem[Fan and Li(2001)]{fan2001variable}
Jianqing Fan and Runze Li.
\newblock Variable selection via nonconcave penalized likelihood and its oracle
  properties.
\newblock \emph{Journal of the American statistical Association}, 96\penalty0
  (456):\penalty0 1348--1360, 2001.

\bibitem[Hastie et~al.(2015)Hastie, Mazumder, Lee, and Zadeh]{hastie2015matrix}
Trevor Hastie, Rahul Mazumder, Jason~D Lee, and Reza Zadeh.
\newblock Matrix completion and low-rank svd via fast alternating least
  squares.
\newblock \emph{Journal of Machine Learning Research}, 16\penalty0
  (1):\penalty0 3367--3402, 2015.
\newblock URL \url{http://jmlr.org/papers/volume16/hastie15a/hastie15a.pdf}.

\bibitem[Heinloth et~al.(2004)Heinloth, Irwin, Boorman, Nettesheim, Fannin,
  Sieber, Snell, Tucker, Li, Travlos, et~al.]{heinloth2004gene}
Alexandra~N Heinloth, Richard~D Irwin, Gary~A Boorman, Paul Nettesheim,
  Rickie~D Fannin, Stella~O Sieber, Michael~L Snell, Charles~J Tucker, Leping
  Li, Gregory~S Travlos, et~al.
\newblock Gene expression profiling of rat livers reveals indicators of
  potential adverse effects.
\newblock \emph{Toxicological Sciences}, 80\penalty0 (1):\penalty0 193--202,
  2004.

\bibitem[H{\"o}skuldsson(1988)]{hoskuldsson1988pls}
Agnar H{\"o}skuldsson.
\newblock Pls regression methods.
\newblock \emph{Journal of chemometrics}, 2\penalty0 (3):\penalty0 211--228,
  1988.

\bibitem[Hosmer and Lemeshow(1989)]{hosmer1989applied}
DW~Hosmer and Stanley Lemeshow.
\newblock Applied logistic regression. 1989.
\newblock \emph{New York: Johns Wiley \& Sons}, 1989.

\bibitem[Husson and Josse(2013)]{husson2013handling}
Fran{\c{c}}ois Husson and Julie Josse.
\newblock Handling missing values in multiple factor analysis.
\newblock \emph{Food quality and preference}, 30\penalty0 (2):\penalty0 77--85,
  2013.

\bibitem[Johnstone and Lu(2004)]{johnstone2004sparse}
Iain~M Johnstone and Arthur~Yu Lu.
\newblock Sparse principal components analysis.
\newblock \emph{Unpublished manuscript}, 7, 2004.

\bibitem[Johnstone and Lu(2009)]{johnstone2009consistency}
Iain~M Johnstone and Arthur~Yu Lu.
\newblock On consistency and sparsity for principal components analysis in high
  dimensions.
\newblock \emph{Journal of the American Statistical Association}, 104\penalty0
  (486):\penalty0 682--693, 2009.

\bibitem[Jolliffe et~al.(2003)Jolliffe, Trendafilov, and
  Uddin]{jolliffe2003modified}
Ian~T Jolliffe, Nickolay~T Trendafilov, and Mudassir Uddin.
\newblock A modified principal component technique based on the lasso.
\newblock \emph{Journal of computational and Graphical Statistics}, 12\penalty0
  (3):\penalty0 531--547, 2003.

\bibitem[Josse and Husson(2016)]{josse2016missmda}
Julie Josse and Fran{\c{c}}ois Husson.
\newblock missmda: a package for handling missing values in multivariate data
  analysis.
\newblock \emph{Journal of Statistical Software}, 70\penalty0 (1):\penalty0
  1--31, 2016.

\bibitem[Krauthgamer et~al.(2015)Krauthgamer, Nadler, Vilenchik,
  et~al.]{krauthgamer2015semidefinite}
Robert Krauthgamer, Boaz Nadler, Dan Vilenchik, et~al.
\newblock Do semidefinite relaxations solve sparse pca up to the information
  limit?
\newblock \emph{The Annals of Statistics}, 43\penalty0 (3):\penalty0
  1300--1322, 2015.

\bibitem[L{\^e}~Cao et~al.(2008)L{\^e}~Cao, Rossouw, Robert-Grani{\'e}, and
  Besse]{le2008sparse}
Kim-Anh L{\^e}~Cao, Debra Rossouw, Christele Robert-Grani{\'e}, and Philippe
  Besse.
\newblock A sparse pls for variable selection when integrating omics data.
\newblock \emph{Statistical applications in genetics and molecular biology},
  7\penalty0 (1), 2008.

\bibitem[L{\^e}~Cao et~al.(2009)L{\^e}~Cao, Gonz{\'a}lez, and
  D{\'e}jean]{le2009integromics}
Kim-Anh L{\^e}~Cao, Ignacio Gonz{\'a}lez, and S{\'e}bastien D{\'e}jean.
\newblock integromics: an r package to unravel relationships between two omics
  data sets.
\newblock \emph{Bioinformatics}, 25\penalty0 (21):\penalty0 2855--2856, 2009.

\bibitem[Manne(1987)]{manne1987analysis}
Rolf Manne.
\newblock Analysis of two partial-least-squares algorithms for multivariate
  calibration.
\newblock \emph{Chemometrics and Intelligent Laboratory Systems}, 2\penalty0
  (1-3):\penalty0 187--197, 1987.

\bibitem[Nelson et~al.(1996)Nelson, Taylor, and MacGregor]{nelson1996missing}
Philip~RC Nelson, Paul~A Taylor, and John~F MacGregor.
\newblock Missing data methods in pca and pls: Score calculations with
  incomplete observations.
\newblock \emph{Chemometrics and intelligent laboratory systems}, 35\penalty0
  (1):\penalty0 45--65, 1996.

\bibitem[Penrose(1956)]{penrose1956best}
Roger Penrose.
\newblock On best approximate solutions of linear matrix equations.
\newblock In \emph{Mathematical Proceedings of the Cambridge Philosophical
  Society}, volume~52, pages 17--19. Cambridge University Press, 1956.

\bibitem[Qin et~al.(2001)Qin, Valle, and Piovoso]{qin2001unifying}
S~Joe Qin, Sergio Valle, and Michael~J Piovoso.
\newblock On unifying multiblock analysis with application to decentralized
  process monitoring.
\newblock \emph{Journal of chemometrics}, 15\penalty0 (9):\penalty0 715--742,
  2001.

\bibitem[Rechtien et~al.(2017)Rechtien, Richert, Lorenzo, Martrus, Hejblum,
  Dahlke, Kasonta, Zinser, Stubbe, Matschl, Lohse, Kr{\"a}hling, Eickmann,
  Becker, Agnandji, Krishna, Kremsner, Brosnahan, Bejon, Njuguna, Addo,
  Siegrist, Huttner, Kieny, Moorthy, Fast, Savarese, Lapujade, Thi{\'e}baut,
  Altfeld, and Addo]{rechtienrichertlorenzo}
Anne Rechtien, Laura Richert, Hadrien Lorenzo, Gloria Martrus, Boris Hejblum,
  Christine Dahlke, Rahel Kasonta, Madeleine Zinser, Hans Stubbe, Urte Matschl,
  Ansgar Lohse, Verena Kr{\"a}hling, Markus Eickmann, Stephan Becker,
  Selidji~Todagbe Agnandji, Sanjeev Krishna, Peter~G. Kremsner, Jessica~S.
  Brosnahan, Philip Bejon, Patricia Njuguna, Marylyn~M. Addo, Claire-Anne
  Siegrist, Angela Huttner, Marie-Paule Kieny, Vasee Moorthy, Patricia Fast,
  Barbara Savarese, Olivier Lapujade, Rodolphe Thi{\'e}baut, Marcus Altfeld,
  and Marylyn Addo.
\newblock Systems vaccinology identifies an early innate immune signature as a
  correlate of antibody responses to the ebola vaccine rvsv-zebov.
\newblock \emph{Cell Reports}, 20\penalty0 (9):\penalty0 2251--2261, 09 2017.
\newblock ISSN 2211-1247.
\newblock \doi{10.1016/j.celrep.2017.08.023}.
\newblock URL \url{http://dx.doi.org/10.1016/j.celrep.2017.08.023}.

\bibitem[Rothman et~al.(2009)Rothman, Levina, and Zhu]{rothman2009generalized}
Adam~J Rothman, Elizaveta Levina, and Ji~Zhu.
\newblock Generalized thresholding of large covariance matrices.
\newblock \emph{Journal of the American Statistical Association}, 104\penalty0
  (485):\penalty0 177--186, 2009.

\bibitem[Sabatier et~al.(2003)Sabatier, Vivien, and Amenta]{sabatier2003two}
Robert Sabatier, Myrtille Vivien, and Pietro Amenta.
\newblock Two approaches for discriminant partial least squares.
\newblock In \emph{Between data science and applied data analysis}, pages
  100--108. Springer, 2003.

\bibitem[Sj{\"o}str{\"o}m et~al.(1986)Sj{\"o}str{\"o}m, Wold, and
  S{\"o}derstr{\"o}m]{sjostrom1986pls}
Michael Sj{\"o}str{\"o}m, Svante Wold, and Bengt S{\"o}derstr{\"o}m.
\newblock Pls discriminant plots.
\newblock In \emph{Pattern Recognition in Practice, Volume II}, pages 461--470.
  Elsevier, 1986.

\bibitem[Stekhoven and B{\"u}hlmann(2011)]{stekhoven2011missforest}
Daniel~J Stekhoven and Peter B{\"u}hlmann.
\newblock Missforest—non-parametric missing value imputation for mixed-type
  data.
\newblock \emph{Bioinformatics}, 28\penalty0 (1):\penalty0 112--118, 2011.

\bibitem[Tenenhaus and Tenenhaus(2011)]{tenenhaus2011regularized}
Arthur Tenenhaus and Michel Tenenhaus.
\newblock Regularized generalized canonical correlation analysis.
\newblock \emph{Psychometrika}, 76\penalty0 (2):\penalty0 257, 2011.

\bibitem[Tenenhaus et~al.(2014)Tenenhaus, Philippe, Guillemot, Le~Cao, Grill,
  and Frouin]{tenenhaus2014variable}
Arthur Tenenhaus, Cathy Philippe, Vincent Guillemot, Kim-Anh Le~Cao, Jacques
  Grill, and Vincent Frouin.
\newblock Variable selection for generalized canonical correlation analysis.
\newblock \emph{Biostatistics}, 15\penalty0 (3):\penalty0 569--583, 2014.

\bibitem[Tibshirani(1996)]{tibshirani1996regression}
Robert Tibshirani.
\newblock Regression shrinkage and selection via the lasso.
\newblock \emph{Journal of the Royal Statistical Society. Series B
  (Methodological)}, pages 267--288, 1996.

\bibitem[Troyanskaya et~al.(2001)Troyanskaya, Cantor, Sherlock, Brown, Hastie,
  Tibshirani, Botstein, and Altman]{troyanskaya2001missing}
Olga Troyanskaya, Michael Cantor, Gavin Sherlock, Pat Brown, Trevor Hastie,
  Robert Tibshirani, David Botstein, and Russ~B Altman.
\newblock Missing value estimation methods for dna microarrays.
\newblock \emph{Bioinformatics}, 17\penalty0 (6):\penalty0 520--525, 2001.

\bibitem[Wangen and Kowalski(1989)]{wangen1989multiblock}
LE~Wangen and BR~Kowalski.
\newblock A multiblock partial least squares algorithm for investigating
  complex chemical systems.
\newblock \emph{Journal of chemometrics}, 3\penalty0 (1):\penalty0 3--20, 1989.

\bibitem[Westerhuis and Smilde(2001)]{westerhuis2001deflation}
Johan~A Westerhuis and Age~K Smilde.
\newblock Deflation in multiblock pls.
\newblock \emph{Journal of chemometrics}, 15\penalty0 (5):\penalty0 485--493,
  2001.

\bibitem[Westerhuis et~al.(1997)Westerhuis, Coenegracht, and
  Lerk]{westerhuis1997multivariate}
Johan~A Westerhuis, Pierre~MJ Coenegracht, and Coenraad~F Lerk.
\newblock Multivariate modelling of the tablet manufacturing process with wet
  granulation for tablet optimization and in-process control.
\newblock \emph{International journal of Pharmaceutics}, 156\penalty0
  (1):\penalty0 109--117, 1997.

\bibitem[Wold(1966)]{wold1966estimation}
Herman Wold.
\newblock Estimation of principal components and related models by iterative
  least squares.
\newblock \emph{Multivariate analysis}, pages 391--420, 1966.

\bibitem[Wold(1984)]{wold1984three}
S~Wold.
\newblock Three pls algorithms according to sw.
\newblock In \emph{Proc.: Symposium MULDAST (multivariate analysis in science
  and technology)}, pages 26--30, 1984.

\bibitem[Wold et~al.(1983)Wold, Martens, and Wold]{wold1983multivariate}
Svante Wold, Harold Martens, and H~Wold.
\newblock The multivariate calibration problem in chemistry solved by the pls
  method.
\newblock In \emph{Matrix pencils}, pages 286--293. Springer, 1983.

\bibitem[Zou et~al.(2006)Zou, Hastie, and Tibshirani]{zou2006sparse}
Hui Zou, Trevor Hastie, and Robert Tibshirani.
\newblock Sparse principal component analysis.
\newblock \emph{Journal of computational and graphical statistics}, 15\penalty0
  (2):\penalty0 265--286, 2006.

\end{thebibliography}

\end{document}